%% file: main.tex
\documentclass[letterpaper]{article}

\newif\ifanonymousauthor
\anonymousauthorfalse

\usepackage{xcolor}
\newif\ifeditorial
\editorialfalse

\usepackage{kotex}
\ifanonymousauthor
  \usepackage[submission]{aaai2027}
\else
  \usepackage[preprint]{aaai2027}
\fi
\usepackage[hyphens]{url}
\usepackage{graphicx}
\usepackage{natbib}
\usepackage{caption}
\usepackage{algorithm}
\usepackage{algorithmic}
\usepackage{booktabs}
\usepackage{amsmath,amssymb,amsthm,mathtools}
\usepackage{tabularx}
\usepackage{comment}

\newtheorem{theorem}{Theorem}[section]
\newtheorem{lemma}[theorem]{Lemma}
\newtheorem{proposition}[theorem]{Proposition}
\newtheorem{corollary}[theorem]{Corollary}
\newtheorem{remark}[theorem]{Remark}
\newcommand{\R}{\mathbb{R}}
\newcommand{\softmax}{\operatorname{softmax}}
\newcolumntype{Y}{>{\raggedright\arraybackslash}X}
\newcommand{\suppsec}[1]{Appendix~\ref{#1} of the supplementary material}

\title{Provably Learning Multi-Head Attention with Queries}
\ifanonymousauthor
  \author{Anonymous Submission}
\else
  \author{
    Sunyeop Kim\textsuperscript{\rm 1,2},
    Insung Kim\textsuperscript{\rm 1}
  }
\fi

\affiliations{
  \textsuperscript{\rm 1}Korea University, Republic of Korea\\
  \textsuperscript{\rm 2}Nanyang Technological University, Singapore\\
  \texttt{kin3548@korea.ac.kr},
  \texttt{cmcom35@korea.ac.kr},
}

\begin{document}
\maketitle

\input{local_commands}
\input{sections/abstract}
\input{sections/introduction}
\input{sections/related_work}

\input{sections/model_and_query_access}
\input{sections/repeated_token_local_decoding}
\input{sections/global_additive_head_matching}

\input{sections/main_result}

\input{sections/finite_precision_boundary}
\input{sections/ffn_extension}
\input{sections/experiments}

\input{sections/limitations_and_conclusion}

\bibliography{references}

\onecolumn
\setcounter{page}{1}
\appendix
\setcounter{section}{0}
\renewcommand{\thesection}{S\arabic{section}}
\renewcommand{\thesubsection}{\thesection.\arabic{subsection}}

\appendix
\setcounter{secnumdepth}{2}
\section*{Supplementary Material}
\input{supplementary/supplementary_material}

\end{document}

%% file: local_commands.tex
\ifeditorial
\newcommand{\jian}[1]{\textcolor{cyan}{\textbf{[JG: #1]}}}
\newcommand{\insung}[1]{\textcolor{red!50!black}{\textbf{[IK: #1]}}}
\newcommand{\sunyeop}[1]{\textcolor{green!50!black}{\textbf{[SK: #1]}}}
\newcommand{\important}[1]{\textcolor{red}{\textbf{[IMPORTANT: #1]}}}
\else
 \newcommand{\jian}[1]{}
 \newcommand{\insung}[1]{}
  \newcommand{\sunyeop}[1]{}
  \newcommand{\important}[1]{}
\fi

%% file: sections/abstract.tex
\begin{abstract} 
We study the problem of learning multi-head softmax attention from black-box input-output access. The learner may query arbitrary real-valued token sequences and observe only the scalar output at the final token. Recent work gives an algorithm using $O(d^2)$ value
queries to recover the single-head parameters $(W,v)$. For multiple heads, the same work establishes identifiability under the
assumption that the heads occupy pairwise orthogonal subspaces. Applying
the single-head recovery algorithm separately to the heads additionally
requires bases for these subspaces to be known. We recover a canonical representation by merging heads with the same $W_h$, summing their corresponding $v_h$, and discarding a merged head when this sum is zero, without these subspace assumptions. By varying the number of copies of a token, our algorithm obtains samples of a rational function whose interpolation separates the canonical heads. Additional queries formed by adding selected token vectors then match the same head across different queries. When the oracle outputs and all subsequent computations are exact, the learner chooses its query vectors at random and recovers the canonical pairs $\{(W_h,v_h):h\in[H]\}$ up to permutation with probability one. When $H$ is known, it uses exactly $4Hd^2-2H+1$ value queries of maximum length $2H+1$. If only a known upper bound $H_0$ is available, the algorithm uses
$4H_0d^2-2H_0+1$ value queries of maximum length $2H_0+1$. For approximate oracle outputs, we give conditions under which the parameter error is at most a model- and query-dependent constant multiple of the output error. Finally, we extend our result to a one-layer Transformer with multi-head attention followed by a bias-free ReLU feed-forward network. Under additional conditions, we recover a functionally equivalent Transformer without relying on a separate algorithm for learning the feed-forward network. 
\end{abstract}

%% file: sections/introduction.tex
\section{Introduction}
Multi-head self-attention is a central component of Transformer architectures \citep{vaswani2017attention}. Recent work gives an $O(d^2)$-query algorithm for exactly recovering a scalar-output single-head softmax attention model \citep{bhattamishra2026attention}. For multi-head attention, they reduce the problem to independent
single-head recovery, assuming that the heads occupy mutually orthogonal
subspaces and that a basis for each subspace is known to the learner. We ask whether exact recovery is possible without such a subspace decomposition.


We represent head $h$ by a pair $(W_h,v_h)$, where a token $x_i$ contributes attention score $x_i^\top W_h q$ and value $x_i^\top v_h$ against the final query token $q$. If two heads $h$ and $g$ satisfy $W_h=W_g$, every query assigns them identical attention weights. The oracle therefore depends on $v_h$ and $v_g$ only through their sum $v_h+v_g$, so the two vectors cannot be recovered separately. We therefore merge such heads, sum their vectors $v_h$, and discard groups whose sum is zero. We call the remaining set of head pairs the canonical representation and aim to recover it. \citet{tran2025equivariant} characterize when two multi-head attention models produce the same output at every token. We prove that equality of the scalar outputs at the final token determines the canonical representation. In fact, equality on length-three inputs already suffices.

Recovering this representation requires more than running a single-head learner repeatedly. Each oracle response is the sum of $H$ nonlinear head outputs and the individual heads are therefore not observed separately. To overcome
this obstacle, our algorithm proceeds in two stages: it first separates the
head contributions for each chosen pair of vectors and then matches these
contributions across different pairs. For each chosen pair of vectors \(u,q\), the response to a sequence containing one perturbed token \(q+tu\) and \(m\) copies of \(q\) is a rational function of \(m\). Interpolation of this rational function separates the $H$ head contributions and returns an unordered set containing, for each head, the two scalars $u^\top W_h q$ and $u^\top v_h$. We then make additional queries using sums of the chosen vectors. Only components from the same head combine consistently under these queries, allowing us to match the heads across different queries. Once the heads are matched across queries, the recovered scalars form linear systems whose solutions recover every pair $(W_h,v_h)$.

\paragraph{Contributions.} 
We study the recovery of the canonical representation of multi-head softmax attention from value queries and establish the following results:

\begin{enumerate} 
\item \textbf{Recovery of canonical multi-head attention.} 
Repeated-token queries make the observed scalar response a rational function of the number of repeated tokens. Rational interpolation reconstructs this function and separates the canonical heads, initially only as an unordered set. Additional queries formed by adding selected input directions align these unordered sets under one common head labeling. For a model with $H$ canonical heads, under the computational assumptions stated in Section~\ref{sec:problem-formulation}, our algorithm recovers all canonical pairs $(W_h,v_h)$ up to permutation with probability one over its random query directions. The algorithm is nonadaptive and, when $H$ is known, uses exactly $4Hd^2-2H+1$ value queries, each of length at most $2H+1$.

\item \textbf{Stability under approximate outputs.} When each scalar oracle response is within $\tau$ of the exact output, we show that, under quantitative separation and conditioning assumptions and for sufficiently small $\tau$, the parameter recovery error is at most $C_{\rm stab}\tau$ up to permutation, where $C_{\rm stab}$ depends on the target model and the realized query directions. A two-head construction shows that no uniform bound on $C_{\rm stab}$ follows from $W_h\ne W_g$ and $v_h\ne0$ alone. We also evaluate our algorithm using high-precision oracle responses and responses rounded to IEEE 754 binary64~\citep{ieee754}.

\item \textbf{Extension to a one-layer Transformer with ReLU.} 
We extend our multi-head recovery algorithm to a bias-free one-layer Transformer in which multi-head attention is followed by a ReLU feed-forward network.

Following \citet{bhattamishra2026attention}, subtracting the output on
$-X$ from that on $X$ simulates one value query to a scalar-output
multi-head attention model using two Transformer queries. Applying our
recovery algorithm to these simulated queries recovers
$(W_h,v_h)$ for every head, up to permutation. We then give a second stage that, under additional conditions on the feed-forward network, constructs a functionally equivalent Transformer without assuming a separate algorithm for learning the feed-forward network and without requiring its width $m$ to satisfy $m\le d$.
\end{enumerate}

Complete proofs, recovery when only an upper bound on $H$ is known, the low-rank extension, conditional recovery from binary membership queries, and full details of the one-layer Transformer extension are provided in Appendices~\ref{supp:canonical}--\ref{supp:transformer} of the supplementary material.

%% file: sections/related_work.tex
\section{Related Work}

\paragraph{Learning Neural Networks with Queries.} Learning with membership and value queries is a classical topic in computational learning theory \citep{angluin1988queries}. For neural networks, prior work ranges from identifiability using the full
input-output map \citep{fefferman1993recovering} to approximate learning with value queries and reconstruction of shallow ReLU networks from membership queries \citep{chen2021reluqueries,daniely2023exact}. Related access models also appear in black-box model extraction \citep{tramer2016stealing,jagielski2020high}. For production language models, \citet{carlini2024stealing} recover the hidden dimension and embedding projection matrix. Other cryptanalytic work studies broader parameter extraction under real-valued, hard-label, nonlinear, and recurrent settings \citep{carlini2020cryptanalytic, chen2024hard,carlini2025hardlabel,asselineau2026nonlinear, wei2026rnn}.

\paragraph{Learning Attention with Queries.} \citet{brunner2019identifiability} show that attention weights need not be identifiable from contextualized representations under a different access model. \citet{bhattamishra2026attention} recover the parameters $(W,v)$ of a single-head softmax attention model with $v\ne0$ using $O(d^2)$ value queries. They establish identifiability of multi-head attention when the head parameters lie in pairwise orthogonal subspaces. Their reduction to single-head recovery additionally requires the learner to know a basis for each subspace. Their one-layer Transformer extension uses the odd-component reduction and assumes access to an algorithm for learning the ReLU feed-forward network. For example, \citet{milli2019model} provide such a learner under a linear-independence assumption, which requires $m\le d$. We use the same reduction based on the difference between the outputs on $X$ and $-X$ to recover the pairs $(W_h,v_h)$ and, under additional conditions, directly construct a functionally equivalent Transformer without requiring $m\le d$. \citet{tran2025equivariant} characterize functional equivalence for full-output multi-head attention, with later work covering positional information \citep{tran2026functional}. For scalar outputs at the final token, we prove the corresponding canonical characterization by merging heads with the same $W_h$, summing their $v_h$, and removing zero aggregates. We recover the canonical pairs for pairwise distinct $W_h$ and nonzero $v_h$ without a subspace decomposition.


\paragraph{Learning Attention from Examples.} 

Unlike our query setting, where the learner chooses each input, these works
assume that the learner receives randomly sampled input--output examples. \citet{chenli2025mha} study softmax multi-head attention under
uniformly sampled sign inputs and observation of the full output matrix.
\citet{yau2025linear} give a polynomial-time learning algorithm for
multi-head linear attention, which uses a different attention mechanism.
Thus, these results concern prediction from randomly sampled examples,
rather than exact recovery of a given softmax attention model using
learner-chosen queries.

Table~\ref{tab:comparison} compares the access models, assumptions, and guarantees of the most directly related results.

\begin{table*}[t]
\centering
\small
\setlength{\tabcolsep}{4pt}
\begin{tabularx}{\textwidth}{
  @{}>{\raggedright\arraybackslash}p{0.19\textwidth}
     >{\raggedright\arraybackslash}p{0.21\textwidth}
     Y
     >{\raggedright\arraybackslash}p{0.20\textwidth}@{}}
\toprule
Result
& Access and target
& Assumptions and guarantee
& Complexity \\
\midrule

Single head \citep{bhattamishra2026attention}
& Exact value queries; recover $(W,v)$
& $v\ne0$; exact recovery
& $O(d^2)$ queries; length at most two \\
\midrule

Known-subspace MHA \citep{bhattamishra2026attention}
& Exact value queries; recover multiple heads
& Pairwise orthogonal subspaces with known bases; exact up to permutation
& Reduction to single-head learners \\
\midrule

Passive softmax MHA \citep{chenli2025mha}
& Random examples; observe the full output matrix
& Uniform sign inputs; approximate prediction
& $(dN)^{O(H^3)}$ running time \\
\midrule

Ours
& Exact value queries; recover canonical pairs
& Pairwise distinct $W_h$ and nonzero $v_h$; exact up to permutation with
probability one
& $4Hd^2-2H+1$ queries; length at most $2H+1$ \\

\bottomrule
\end{tabularx}
\caption{Comparison of prior work and our result on learning softmax
attention. MHA denotes multi-head attention. Query and sample complexities are not directly comparable because the
access models and observed outputs differ.}
\label{tab:comparison}
\end{table*}




%% file: sections/model_and_query_access.tex
\section{Definitions}\label{sec:learning_setting}
\insung{We need to formally rewrite the definition of each symbol and the problem statement.}

\subsection{Notation}
For an integer $n\ge1$, let $[n]\coloneqq\{1,\ldots,n\}$. We write $G^\top$ for the transpose of a matrix $G$. All vectors are column vectors, and for $y_1,\ldots,y_N\in\R^d$, we write $[y_1^\top;\ldots;y_N^\top]\in\R^{N\times d}$ for the matrix whose $i$-th row is $y_i^\top$. We use $\|\cdot\|_2$ for the Euclidean norm on vectors and the spectral norm on matrices, and $\|\cdot\|_F$ for the Frobenius norm.

\subsection{Multi-Head Attention Model}

Fix a token dimension $d\ge1$. An input sequence of length $N\ge1$ is
represented by a matrix $X\in\R^{N\times d}$ with rows
$x_1^\top,\ldots,x_N^\top$, where each $x_i\in\R^d$. The final token
$q\coloneqq x_N$ is used as the query token.

The model has $H_0\ge1$ heads, each with head dimension $d_h\ge1$. For each $h\in[H_0]$, let $Q_h,K_h,V_h\in\R^{d_h\times d}$ be the query, key, and value projection matrices of head $h$, and let $o_h\in\R^{d_h}$ be a read-out vector mapping the head output to the observed scalar. The query and key projections combine into a single matrix, and the value projection and read-out vector into a single vector,
\[
 W_h\coloneqq\frac{K_h^\top Q_h}{\sqrt{d_h}}\in\R^{d\times d},
 \qquad
 v_h\coloneqq V_h^\top o_h\in\R^d,
\]
so that for every token $x_i$ and query $q$,
\[
 \frac{(K_hx_i)^\top(Q_hq)}{\sqrt{d_h}}=x_i^\top W_hq,
 \qquad
 o_h^\top(V_hx_i)=x_i^\top v_h.
\]
Any two choices of $(Q_h,K_h,V_h,o_h)$ that yield the same $(W_h,v_h)$ produce the same attention scores and scalar token values, hence the same scalar output. We therefore parameterize the model by $ M\coloneqq\bigl((W_1,v_1),\ldots,(W_{H_0},v_{H_0})\bigr),$
where $W_h\in\R^{d\times d}$ and $v_h\in\R^d$. The learner targets only the pairs $(W_h,v_h)$. The factors $(Q_h,K_h,V_h,o_h)$ are not identifiable from scalar outputs.

For each $h\in[H_0]$, the attention weight assigned to token $x_i$ is
\[
 a_{h,i}(X)
 \coloneqq
 \frac{\exp(x_i^\top W_hq)}{\sum_{j=1}^N\exp(x_j^\top W_hq)},
 \qquad i\in[N],
\]
the scalar output of head $h$ is
\[
 f_{W_h,v_h}(X)\coloneqq\sum_{i=1}^N a_{h,i}(X)\,x_i^\top v_h,
\]
and the multi-head model returns
\[
 F_M(X)\coloneqq\sum_{h=1}^{H_0}f_{W_h,v_h}(X).
\]

\subsection{Canonical Representation}
Heads sharing the same $W_h$ assign identical attention weights to every input. Since the scalar head output is linear in $v_h$ for fixed $W_h$, such heads contribute through the sum of their corresponding vectors. For any $A\in\R^{d\times d}$, define
\[
C_M(A)
\coloneqq
\sum_{\substack{h\in[H_0]\\W_h=A}}v_h.
\]
The canonical representation of $M$ is
\[
\operatorname{Can}(M)
\coloneqq
\left\{
(A,C_M(A))\,\middle|\,
\substack{
A\in\{W_h:h\in[H_0]\}\\
C_M(A)\ne0
}
\right\}.
\]
Thus, heads with the same $W_h$ are merged, their corresponding $v_h$
are summed, and groups whose sum is zero are removed. We write $H\coloneqq|\operatorname{Can}(M)|$
for the number of canonical heads. From this point on, we reuse the
symbols $(W_h,v_h)$ for the canonical heads and enumerate the canonical
representation as $\operatorname{Can}(M)=\{(W_h,v_h):h\in[H]\}$.
By construction, these canonical heads satisfy 
\begin{equation} 
W_h\ne W_g\quad(h\ne g),\qquad v_h\ne0. \label{eq:id} 
\end{equation}

\begin{proposition}[Canonical identifiability]\label{prop:canonical-identifiability} 
Let $M$ and $M'$ be two scalar-output multi-head attention models of
token dimension $d$. Then $\operatorname{Can}(M)=\operatorname{Can}(M')$
if and only if $F_M(X)=F_{M'}(X)
~\text{for every }X\in\R^{3\times d}.$
\end{proposition}

\citet{tran2025equivariant} characterize when two multi-head attention models produce identical full outputs for every input sequence, regardless of its length. Our oracle instead observes only the scalar $F_M(X)$ obtained from the output at the final token. Equality of these scalar observations is a weaker requirement than equality of the full outputs, so their characterization does not apply directly to our setting. The proof of Proposition~\ref{prop:canonical-identifiability} is given in \suppsec{supp:canonical}.

This result concerns identifiability rather than an explicit recovery algorithm. It shows that observations on all inputs of length three uniquely \emph{determine} $\operatorname{Can}(M)$, but it does not give a procedure that recovers the heads using only such inputs. Our recovery algorithm uses repeated-token queries to interpolate a rational function and therefore uses sequences of length up to $2H+1$. Thus, the two sequence lengths describe different results: length three is sufficient for identifiability, whereas length $2H+1$ is used by our constructive recovery algorithm.

\subsection{Problem Statement}\label{sec:problem-formulation}

\paragraph{Value-query oracle.} The learner has value-query access to $F_M$. A value query consists of choosing a sequence length $N\ge1$ and an arbitrary input $X\in\R^{N\times d}$ and receiving the exact scalar $F_M(X)$. The learner observes neither the attention weights nor the individual head outputs. We count each evaluation of $F_M$ as one value query and define the maximum query length as the largest sequence length used by the learner.

\paragraph{Recovery objective.} The learner knows the token dimension $d$ and the raw head count $H_0$, but does not know the model parameters or the canonical head count $H=|\operatorname{Can}(M)|$. The goal is to recover the unordered canonical set $\operatorname{Can}(M)$. Our sharp query bound is first stated under the additional assumption that $H$ is known; a subsequent result uses only $H_0$ and determines $H$ from the oracle responses.


\paragraph{Computational assumptions.}
For the theoretical analysis, we assume that the oracle returns exact real values and that the learner can perform exact real arithmetic, evaluate exponential and logarithmic functions, solve finite-dimensional linear systems, and factor polynomials. We count oracle calls and the number of these operations, rather than finite-bit operations.

These assumptions are separate from the numerical algorithms used to realize the operations in practice. In our implementation, linear systems are solved in multiprecision arithmetic and polynomial roots are computed from the eigenvalues of a companion matrix. Further implementation details are provided in
\suppsec{supp:numerical}.

\sunyeop{justifying this part will be important. also, we should say how we could approximate this for real computation}
\insung{I agree. Along with robustness, I think we'll need to spend a good amount of time justifying our model. Let's first get the overall draft down, and then work through this part together and fill in the details.}

%% file: sections/repeated_token_local_decoding.tex
\section{Recovery Algorithm}
\subsection{Repeated-Token Local Decoding}\label{sec:local}

Fix $u,q\in\R^d$ and a nonzero scale $t$. For $m\ge1$, let
\begin{equation}
 X_m^{(t)}(u,q)=[(q+tu)^\top;q^\top;\ldots;q^\top],
 \label{eq:repeat}
\end{equation}
a sequence of length $m+1$ whose first token is $q+tu$ and whose
remaining $m$ tokens are all equal to $q$. Define
\[
 s_h\coloneqq t u^\top W_hq,\qquad
 c_h\coloneqq t u^\top v_h,\qquad
 r_h\coloneqq e^{s_h}.
\]
These are functions of $(u,q)$, which we make explicit as $s_h(u,q)$ and
$c_h(u,q)$ when the dependence matters.
\begin{lemma}[Repeated-token identity]
For every integer $m\ge1$, define $\mathcal{R}(m)$ as follows.
\begin{equation}
 \mathcal{R}(m)\coloneqq F_M(X_m^{(t)}(u,q))-F_M([q^\top]).
 \label{eq:response}
\end{equation}
Then
\begin{equation}
 \mathcal{R}(m)=\sum_{h=1}^H\frac{c_hr_h}{m+r_h}.
 \label{eq:rational}
\end{equation}
\end{lemma}

\begin{proof}

Fix a head $h$. The query is the final token $q$, so the first token has attention score $(q+tu)^\top W_hq$ and each of the $m$ copies of $q$ has attention score
$q^\top W_hq$. Its attention weight is therefore
\[
 \frac{e^{(q+tu)^\top W_hq}}{e^{(q+tu)^\top W_hq}+m\,e^{q^\top W_hq}}
 =\frac{r_h}{m+r_h},
\]
and the value of the first token exceeds that of $q$ by
$(q+tu)^\top v_h-q^\top v_h=c_h$. Hence
\begin{align*}
 f_{W_h,v_h}(X_m^{(t)}(u,q))
 &=\frac{r_h}{m+r_h}(q+tu)^\top v_h+\frac{m}{m+r_h}q^\top v_h\\
 &=q^\top v_h+\frac{c_hr_h}{m+r_h},
\end{align*}
whereas $f_{W_h,v_h}([q^\top])=q^\top v_h$. Subtracting and summing over
the heads gives \eqref{eq:rational}.

\end{proof}
\insung{I've written out the definition of $X_m$ and the proof of Lemma 4.1 in more detail.}

$\mathcal{R}$ in \eqref{eq:rational} can be naturally extended to the rational function of a real variable $z$:
\begin{equation}
 \mathcal{R}(z)=\sum_{h=1}^H\frac{c_hr_h}{z+r_h}
     =\frac{\mathcal{P}(z)}{\mathcal{Q}(z)},\qquad
 \mathcal{Q}(z)=\prod_{h=1}^H(z+r_h). 
 \label{eq:pq}
\end{equation}
For every fixed model satisfying \eqref{eq:id}, the random choice of $(u,q)$ yields distinct $r_1,\ldots,r_H$ and nonzero $c_1,\ldots,c_H$ with probability one. Consequently, $\mathcal{P}$ and $\mathcal{Q}$ have no common factor and $\deg \mathcal{P}<\deg \mathcal{Q}=H$. Each $-r_h$ is therefore a pole of
$\mathcal{R}$. The residue at $-r_h$, namely the coefficient of
$1/(z+r_h)$ in \eqref{eq:pq}, is $c_hr_h$. Since $r_h>0$, we also have
$\mathcal{Q}(k)\ne0$ for $k=1,\ldots,2H$. The following lemma therefore
applies.

\begin{lemma}[Rational interpolation]
Let $\mathcal{R}=\mathcal{P}/ \mathcal{Q}$, where $\mathcal{P}$ and $\mathcal{Q}$ have no common factor, $\mathcal{Q}$ is monic,
and $\deg \mathcal{P}<\deg \mathcal{Q}=H$. Suppose that $\mathcal{Q}(k)\ne0$ for
$k=1,\ldots,2H$. Then $\mathcal{R}(1),\ldots,\mathcal{R}(2H)$ uniquely determine
$\mathcal{P}$ and $\mathcal{Q}$.
\end{lemma}

\begin{proof}
Suppose that $\mathcal{P}_1/\mathcal{Q}_1$ and $\mathcal{P}_2/\mathcal{Q}_2$ both satisfy the assumptions and
agree at $1,\ldots,2H$. Since neither denominator vanishes at these
points, the polynomial $\mathcal{D}=\mathcal{P}_1\mathcal{Q}_2-\mathcal{P}_2\mathcal{Q}_1$
vanishes at all $2H$ points. Moreover,
$\deg \mathcal{D}\le (H-1)+H=2H-1$. Hence $\mathcal{D}$ is identically zero, so $\mathcal{P}_1\mathcal{Q}_2=\mathcal{P}_2\mathcal{Q}_1$. Since $\mathcal{P}_1$ and
$\mathcal{Q}_1$ have no common factor, $\mathcal{Q}_1$ divides $\mathcal{Q}_2$. Both denominators
are monic and have degree $H$, so $\mathcal{Q}_1=\mathcal{Q}_2$, and consequently
$\mathcal{P}_1=\mathcal{P}_2$.
\end{proof}




Thus, when $H$ is known, $\mathcal{R}(1),\ldots,\mathcal{R}(2H)$ determine $\mathcal{P}$ and $\mathcal{Q}$
uniquely by the method of undetermined coefficients. Factoring $\mathcal{Q}$
recovers the poles, hence the unordered set $\{r_1,\ldots,r_H\}$, and
taking logarithms gives each $s_h$. Dividing each residue by the
corresponding $r_h$ recovers $c_h$. We write
\[
 D_t(u,q)\coloneqq\{(s_h,c_h):h\in[H]\}
\]
for the resulting unordered set. We call the procedure that recovers
$D_t(u,q)$ from the repeated-token responses a local decoder. It recovers
one pair $(s_h,c_h)$ for each head, while leaving the ordering of the heads
unknown. Given the one-token output $F_M([q^\top])$, each local decoder uses
$2H$ repeated-token queries.

\sunyeop{we should determine how we treat $H$ and $H_0$ differently}
\insung{Handled it this way for now. The main flow assumes $H$ is known, matching the sharp bound in Thm.~\ref{thm:main}, and the paragraph above covers the case where only $H_0$ is available via the least-degree search, which is what Cor.~\ref{cor:h0} relies on. Let me know if you'd rather lead with the $H_0$ version instead.}

\sunyeop{I think we might explain how it is actually implemented this in the appendix and work only on the exact model}

Under our computation assumptions, the denominator polynomial is factored
exactly. Our numerical implementation instead solves the interpolation
system in multiprecision and computes the denominator roots as the
eigenvalues of a companion matrix.

\paragraph{Companion-matrix realization.}
In the main setting, $\deg\mathcal Q=H$ and $\deg\mathcal P<H$. We write
$\mathcal P(z)=\sum_{\ell=0}^{H-1}p_\ell z^\ell$ and
$\mathcal Q(z)=z^H+\sum_{\ell=0}^{H-1}q_\ell z^\ell$.
The observations satisfy
$\mathcal P(m)=\mathcal R(m)\mathcal Q(m)$ for $m=1,\ldots,2H$,
giving a linear system for the $2H$ coefficients. The standard
$H\times H$ companion matrix $C_{\mathcal Q}$ has last column
$(-q_0,\ldots,-q_{H-1})^\top$, subdiagonal entries equal to one, and
satisfies $\det(zI-C_{\mathcal Q})=\mathcal Q(z)$. Its eigenvalues are
therefore $-r_1,\ldots,-r_H$, from which
$s_h=\log r_h$ and
$c_h=\mathcal P(-r_h)/(r_h\mathcal Q'(-r_h))$ are recovered.
Our numerical implementation uses multiprecision arithmetic. Using
conventional linear-system and eigenvalue algorithms, both steps require
$O(H^3)$ arithmetic operations per local decoder.

\insung{Cleaned up this paragraph: "the model structure certifies" overstated things, since the roots being negative reals follows immediately from $r_h=e^{s_h}>0$, so I stated it that way. Also switched the remaining $^T$ to $^\top$ for consistency.}

%% file: sections/global_additive_head_matching.tex
\subsection{Head Matching Across Queries}
\label{sec:global}

\sunyeop{this section was previously named additive matching. was it better?}
\insung{I think the current name is more intuitive and works well as it is.}
\insung{I think it would be good to have at least one figure illustrating the workflow of our method. In addition, it would be nice to include some simple examples and figures in the supplementary material.}

Because each $D_t(u,q)$ is unordered, the learner cannot directly
determine which outputs from different local decoders were produced by the
same head.  To align them, let $O_{\mathbf U}$ and $O_{\mathbf Q}$ be sampled
independently and uniformly from the set of orthogonal matrices. Independently, let $\Lambda_{\mathbf U}$ and $\Lambda_{\mathbf Q}$ be diagonal matrices
with independent entries uniform on $[1,2]$. Set\footnote{Although the reconstruction formulas permit any invertible $\mathbf U$ and $\mathbf Q$ satisfying the required nondegeneracy conditions, we use this construction so that the matrix inversions remain well conditioned, reducing error amplification in finite-precision computation.}
\begin{equation}
 \mathbf U=\Lambda_{\mathbf U}O_{\mathbf U},
 \qquad
 \mathbf Q=O_{\mathbf Q}\Lambda_{\mathbf Q}.
 \label{eq:conditioned-bases}
\end{equation}
and write
$\mathbf U=[u_1,\ldots,u_d]^\top$ and
$\mathbf Q=[q_1,\ldots,q_d]$.
All singular values of $\mathbf U$ and $\mathbf Q$ lie in $[1,2]$.
The learner applies the local decoder to the following pairs:
\begin{align}
 &D_t(u_i,q_j)
 &&i,j\in[d],\label{eq:grid}\\
 &D_t(u_1+u_i,q_1)
 &&i=2,\ldots,d,\label{eq:ubridge}\\
 &D_t(u_i,q_1+q_j)
 &&i\in[d],\quad j=2,\ldots,d.
 \label{eq:qbridge}
\end{align}
Since $s_h(u,q)=tu^\top W_hq$ is bilinear in $(u,q)$,
\begin{align}
 s_h(u_1+u_i,q_1)
 &=s_h(u_1,q_1)+s_h(u_i,q_1),\label{eq:add-u}\\
 s_h(u_i,q_1+q_j)
 &=s_h(u_i,q_1)+s_h(u_i,q_j).\label{eq:add-q}
\end{align}
The following lemma uses these identities to match entries across the sets and recover every $(W_h,v_h)$.

\begin{lemma}[Consistent head labeling]\label{lem:labeling}
Suppose every set in \eqref{eq:grid}--\eqref{eq:qbridge} is recovered
exactly, and that for all $h_1,h_2,h_3\in[H]$ and all $i,j$
in the ranges specified in \eqref{eq:ubridge} and \eqref{eq:qbridge},
\begin{align*}
 &s_{h_1}(u_1+u_i,q_1)=s_{h_2}(u_1,q_1)+s_{h_3}(u_i,q_1)\\
 &\qquad\Longrightarrow\ h_1=h_2=h_3,\\
 &s_{h_1}(u_i,q_1+q_j)=s_{h_2}(u_i,q_1)+s_{h_3}(u_i,q_j)\\
 &\qquad\Longrightarrow\ h_1=h_2=h_3.
\end{align*}
Then every pair $(W_h,v_h)$ is recovered up to permutation.
\end{lemma}


\begin{proof}
The learner first queries the one-token output $F_M([q_1^\top])$, which suffices to decode $D_t(u_i,q_1)$ for $i\in[d]$ and $D_t(u_1+u_i,q_1)$ for $i=2,\ldots,d$. Then the learner assigns arbitrary labels \(1,\ldots,H\) to the elements of \(D_t(u_1,q_1)\). For each \(i=2,\ldots,d\), \eqref{eq:add-u} propagates these labels to \(D_t(u_i,q_1)\).
Entries of the same head satisfy
\[
 s_h(u_1+u_i,q_1)=s_h(u_1,q_1)+s_h(u_i,q_1),
\]
and by the hypothesis of the lemma no other combination does.
For each label $h$, collect the vector
\[
 \mathbf c_h\coloneqq
 \bigl[c_h(u_1,q_1),\ldots,c_h(u_d,q_1)\bigr]^\top
 =t\mathbf Uv_h.
\]
Since $u_1,\ldots,u_d$ are linearly independent, $\mathbf U$ is
invertible, and hence
$v_h=t^{-1}\mathbf U^{-1}\mathbf c_h$.

Let $v_{\mathrm{sum}}\coloneqq \sum_{h=1}^H v_h$. Since every head places weight one on the single token of a one-token sequence,
\[
 F_M([q^\top])=\sum_{h=1}^H q^\top v_h=q^\top v_{\mathrm{sum}}
 \qquad\text{for every }q\in\R^d.
\]
The learner therefore computes the one-token outputs at $q_j$ and $q_1+q_j$ without further value queries, and decodes the remaining sets in \eqref{eq:grid} and \eqref{eq:qbridge}.

For each $i\in[d]$ and $j=2,\ldots,d$, the set $D_t(u_i,q_1)$ is already labeled, and \eqref{eq:add-q} determines the labels in
$D_t(u_i,q_j)$ through the equality $ s_h(u_i,q_1+q_j)=s_h(u_i,q_1)+s_h(u_i,q_j),$
and every set $D_t(u_i,q_j)$ thus receives the same labels. Collecting $(S_h)_{ij}\coloneqq s_h(u_i,q_j)$ gives $S_h=t\mathbf{U}W_h\mathbf{Q}$. Since $\mathbf{U}$ and $\mathbf{Q}$ are invertible,
\begin{equation}
 W_h=t^{-1}\mathbf{U}^{-1}S_h\mathbf{Q}^{-1},
 \qquad
 v_h=t^{-1}\mathbf{U}^{-1}\mathbf c_h.
 \label{eq:reconstruct}
\end{equation}
Thus every $(W_h,v_h)$ is recovered, up to the arbitrary labeling of $D_t(u_1,q_1)$.
\end{proof}

\sunyeop{The choice of $U$,$Q$ and $t$ bellow is just for numerical precision. We might mention this.}

%% file: sections/main_result.tex
\section{Recovery Guarantees and Extensions}
\subsection{Exact Recovery and Query Complexity}
\label{sec:main-result}

Throughout, we use the fixed public value $t=1$, independently of the target parameters. Exact recovery holds for every fixed $t\ne0$. The effect of other choices on numerical conditioning is discussed in \suppsec{supp:numerical}.

\begin{theorem}[Exact multi-head recovery] \label{thm:main}
Let $d,H\ge1$, and let $M$ be a model with $H$ canonical heads,
enumerated as $\operatorname{Can}(M)=\{(W_h,v_h):h\in[H]\}$, satisfying
\eqref{eq:id}. Assuming that every computation is exact, for any fixed $t\ne0$, Algorithm~\ref{alg:main} recovers every pair $(W_h,v_h)$ up to permutation with probability one. It uses exactly \begin{equation} 4Hd^2-2H+1 \label{eq:count} \end{equation} value queries, each of length at most $2H+1$. 
\end{theorem}


\begin{proof}
For every fixed model satisfying \eqref{eq:id}, the choice in \eqref{eq:conditioned-bases} satisfies the local-decoding and head-matching conditions of Lemma~\ref{lem:labeling} with probability one (see \suppsec{supp:exact}). Applying Lemma~\ref{lem:labeling} then yields recovery of every $(W_h,v_h)$ up to permutation. The algorithm evaluates $d^2+(d-1)+d(d-1)=2d^2-1$ sets $D_t(u,q)$, each using $2H$ repeated-token queries, together with the single one-token query $F_M([q_1^\top])$. Hence the total query count is $2H(2d^2-1)+1=4Hd^2-2H+1$. The maximum query length is $2H+1$.
\end{proof}

\begin{algorithm}[t]
\caption{Exact recovery from value queries}
\label{alg:main}
\begin{algorithmic}[1]
\REQUIRE Dimensions $d,H\ge1$, a fixed $t\ne0$ and value-query access to $F_M$
\STATE Choose $\mathbf U,\mathbf Q$ according to
\eqref{eq:conditioned-bases}.
\STATE Query $F_M([q_1^\top])$ and, for every pair in
       \eqref{eq:grid}--\eqref{eq:qbridge} and every
       $m=1,\ldots,2H$, query the corresponding repeated-token sequence
\STATE Using $F_M([q_1^\top])$, recover $D_t(u_i,q_1)$ for all
       $i\in[d]$ and $D_t(u_1+u_i,q_1)$ for $i=2,\ldots,d$
\STATE Assign arbitrary labels to the elements of $D_t(u_1,q_1)$
\STATE Use \eqref{eq:add-u} to assign the same labels to every
       $D_t(u_i,q_1)$
\STATE Recover every $v_h$ from the labeled $c_h(u_i,q_1)$-values
\STATE Compute $F_M([q_j^\top])$ and $F_M([(q_1+q_j)^\top])$ from
       $\sum_{h=1}^H v_h$
\STATE Recover all remaining sets in \eqref{eq:grid} and
       \eqref{eq:qbridge}
\STATE Use \eqref{eq:add-q} to assign the same labels to every
       $D_t(u_i,q_j)$
\STATE Recover every $W_h$ and return the pairs $(W_h,v_h)$ using
       \eqref{eq:reconstruct}
\end{algorithmic}
\end{algorithm}

Note that all value queries in Algorithm~\ref{alg:main} are fixed before any oracle response is observed. The recovered pairs contain $H(d^2+d)$ real entries, while the algorithm uses $4Hd^2-2H+1$ value queries. Both quantities are of order $Hd^2$.

\paragraph{Further variants.} Theorem~\ref{thm:main} assumes that the canonical head count $H$ is known. If only the raw head count $H_0$ is known, then $H\le H_0$. The learner can use $2H_0$ repeated-token queries for each $(u,q)$ and determine $H$ as the smallest denominator degree consistent with the observed values. This uses $4H_0d^2-2H_0+1$ queries with maximum query length $2H_0+1$. Complete constructions and proofs, including a low-rank variant that
uses $O(Hrd)$ value queries when $\operatorname{rank}(W_h)\le r$ under
additional matrix-recovery assumptions, are given in
\suppsec{supp:variants}.

\paragraph{Offline computational complexity.}
The exact reconstruction performs $2d^2-1$ local decodings. Each local decoding solves a rational-interpolation system of dimension $2H$ and factors one degree-$H$ denominator polynomial. Hence the exact algorithm makes $2d^2-1$ calls to degree-$H$ univariate polynomial factorization. Under our computational assumptions, each such factorization is treated as an exact primitive operation.

Apart from these factorization calls, rational interpolation, coefficient recovery, and exhaustive head matching require $O(d^2H^3)$ exact arithmetic operations. The inverses $\mathbf U^{-1}$ and $\mathbf Q^{-1}$ are computed once in $O(d^3)$ operations, and forming $\mathbf U^{-1}S_h\mathbf Q^{-1}$ for all $H$ heads requires $O(d^3H)$ operations. The algorithm additionally uses $O(d^2H)$ exact logarithm evaluations and polynomially many comparison and equality tests.

In the numerical implementation, each degree-$H$ polynomial factorization is replaced by an eigenvalue computation for an $H\times H$ companion matrix. Using a conventional dense eigensolver, this requires $O(H^3)$ arithmetic operations per local decoding and therefore $O(d^2H^3)$ operations in total. Thus the numerical realization has overall arithmetic complexity $O(d^2H^3+d^3H)$.


%% file: sections/finite_precision_boundary.tex
\subsection{Recovery from Approximate Outputs}
\label{sec:finite}

We next consider an oracle whose returned value satisfies
\begin{equation}
 \bigl|\widetilde F(X)-F_M(X)\bigr|\le\tau
 \qquad\text{for every queried }X.
 \label{eq:approx-oracle}
\end{equation}
All subsequent computations on the returned values are assumed exact. Thus, the result concerns errors in the scalar oracle outputs and excludes offline rounding errors.

In the approximate-output variant, we directly query all $2d-1$ one-token outputs. Subtracting the corresponding one-token output from a repeated-token response produces error at most $2\tau$. After recovering the approximate poles, we estimate the coefficients $c_h$ by least squares and replace exact additive equalities with closest-candidate matching.

\begin{theorem}[Conditional recovery from approximate outputs]
\label{thm:robust-recovery}
Fix a model with $H\ge2$ canonical heads and fix $\mathbf U$ and
$\mathbf Q$ as in \eqref{eq:conditioned-bases}. Assume that, for every
pair $(u,q)$ appearing in \eqref{eq:grid}--\eqref{eq:qbridge}, the values
$r_1,\ldots,r_H$ are pairwise distinct, $c_h\ne0$ for every $h$, and the
equalities in \eqref{eq:add-u} and \eqref{eq:add-q} hold only when all
terms belong to the same head. There exist explicit constants
\[
 \tau_0>0
 \qquad\text{and}\qquad
 C_{\rm stab}<\infty,
\]
depending on the target model and the chosen $\mathbf U$ and $\mathbf Q$, such that whenever $0\le\tau<\tau_0$, the recovery algorithm returns
estimates $\{(\widehat W_h,\widehat v_h):h\in[H]\}$ and a permutation
$\pi$ satisfying
\[
 \max_{h\in[H]}
 \left(
  \|\widehat W_h-W_{\pi(h)}\|_F+
  \|\widehat v_h-v_{\pi(h)}\|_2
 \right)
 \le C_{\rm stab}\tau.
\]
The algorithm uses $4Hd^2-2H+2d-1$ value queries, each of length at
most $2H+1$.
\end{theorem}

\begin{proof}[Proof sketch]
For the finitely many pairs in
\eqref{eq:grid}--\eqref{eq:qbridge}, the assumed inequalities have
positive minimum gaps. Since the recovered roots are simple and the
relevant linear systems are nonsingular, rational interpolation and
coefficient estimation are locally Lipschitz near the exact responses.
Closest-candidate matching therefore remains unchanged for sufficiently
small $\tau$. Applying \eqref{eq:reconstruct} gives the stated bound. The complete proof, precise definitions, and explicit constants are given in \suppsec{supp:stability}.
\end{proof}

The same appendix gives a two-head construction showing that no uniform bound on $C_{\rm stab}$ follows from \eqref{eq:id} alone.

%% file: sections/ffn_extension.tex
\subsection{Extension to a One-Layer ReLU Transformer}
\label{sec:relu-extension}

Consider
\[
\begin{aligned}
 \mathrm{TF}(X)
 &=w_o^\top\operatorname{ReLU}\!\left(
   \sum_{h=1}^H A_h^\top y_h(X)\right),\\
 y_h(X)
 &=X^\top\softmax(XW_hq),
\end{aligned}
\]
where $A_h\in\mathbb R^{d\times m}$, $w_o\in\mathbb R^m$, and $q$ is
the final token. We assume that the learner may query both $X$ and
$-X$. The odd-component reduction of
\citet[Section~4.1]{bhattamishra2026attention} gives
\[
 \mathrm{TF}(X)-\mathrm{TF}(-X)
 =\sum_{h=1}^H v_h^\top y_h(X),
 \qquad v_h=A_hw_o.
\]
Thus, two Transformer queries simulate one value query to the scalar-output
multi-head attention model. Theorem~\ref{thm:main} therefore recovers the
pairs $(W_h,v_h)$ up to permutation using
$2(4Hd^2-2H+1)$ Transformer queries.

To recover the remaining ReLU part, first observe that length-one
queries give
\[
 \mathrm{TF}([x^\top])
 =\sum_{j=1}^m w_j\operatorname{ReLU}(\bar b_j^\top x),
 \qquad
 \bar b_j=\sum_{h=1}^H b_{hj},
\]
where $A_h=[b_{h1},\ldots,b_{hm}]$ and
$w_o=(w_1,\ldots,w_m)^\top$. For each nonzero $\bar b_j$, let $n_j$ be the
unique element of $\{\pm\bar b_j/\|\bar b_j\|_2\}$ whose first nonzero
coordinate is positive, and set $\gamma_j=w_j\|\bar b_j\|_2$. For
$\sigma\in\{-1,+1\}^m$, define
$Z(\sigma)=\sum_{j=1}^m\gamma_j\sigma_jn_j$.

\begin{theorem}[Functional recovery]
\label{thm:full-transformer}
Assume that (i) $H,d,m$ are known, (ii) the $W_h$ are pairwise distinct
and $v_h\ne0$, (iii) every $w_j$ and $\bar b_j$ is nonzero and no two
$\bar b_j$ are scalar multiples of one another, and (iv) $Z$ is
one-to-one. Then there is an algorithm that, with probability one over
its sampled query directions and scales, halts after finitely many exact
Transformer queries, each of length at most $2H+1$, and returns a
bias-free one-layer ReLU Transformer of width at most $2m+2$ that has
the same output as $\mathrm{TF}$ for every input $X$.
\end{theorem}

\begin{proof}[Proof sketch]
For length-one inputs,
$\mathrm{TF}([x^\top])=\sum_jw_j\operatorname{ReLU}(\bar b_j^\top x)$.
Slope changes along the queried lines recover all $(n_j,\gamma_j)$.
Crossing each hyperplane away from the others yields sign-vector
differences $\pm2e_j$, so the recovered sign vectors span $\mathbb R^m$
even when $m>d$. Queries from $m$ linearly independent sign vectors at
increasing scales then separate the ReLU units and attention heads. Each
preactivation changes sign at most once, and the one-to-one property of
$Z$ distinguishes different sign configurations, yielding eventual
recovery. Finally, $\operatorname{ReLU}(z)=(z+|z|)/2$ gives an equivalent
representation of width at most $2m+2$. Complete details are given in
\suppsec{supp:transformer}.
\end{proof}

%% file: sections/experiments.tex
\section{Experiments}
\label{sec:experiments}

\paragraph{Setup.}
Unless otherwise stated, each tested $(d,H)$ configuration contains
100 independently sampled target models. For each target and each
$h\in[H]$, the entries of $W_h\in\mathbb R^{d\times d}$ and
$v_h\in\mathbb R^d$ are drawn independently from
$\mathcal N(0,1/d)$. This scaling keeps their typical norms comparable
as $d$ varies. We use the same fixed scale $t=1$ for every target,
independently of its parameters. Let
$\{(\widehat W_h,\widehat v_h):h\in[H]\}$ denote the estimates returned
by the recovery algorithm. We compute 
\[ 
E_{\rm param}= \min_{\pi\in \Pi_H}\max_{h\in[H]} \left( \|\widehat W_h-W_{\pi(h)}\|_F+ \|\widehat v_h-v_{\pi(h)}\|_2 \right), 
\] 
where $\Pi_H$ is the set of permutations of $[H]$.

\paragraph{High-precision recovery.} We approximate the exact setting of Theorem~\ref{thm:main} using 180-digit oracle responses and offline computation. Across all tested configurations, every run recovers all heads with $E_{\rm param}<10^{-100}$. Table~\ref{tab:exact-recovery} shows representative settings. Complete results are given in \suppsec{supp:high-precision}.

\begin{table}[t]
\centering
{\small
\setlength{\tabcolsep}{2.5pt}
\begin{tabular}{@{}c r r c c c@{}}
\toprule
& & & \multicolumn{3}{c}{$E_{\rm param}$}\\
\cmidrule(lr){4-6}
$(d,H)$ & Params. & Queries & Min. & Median & Max.\\
\midrule
$(3,8)$
& 96 & 273
& $3.67\mathrm{e}{-142}$
& $1.46\mathrm{e}{-133}$
& $3.60\mathrm{e}{-124}$\\
$(8,8)$
& 576 & 2,033
& $9.55\mathrm{e}{-133}$
& $7.63\mathrm{e}{-128}$
& $9.72\mathrm{e}{-121}$\\
$(16,8)$
& 2,176 & 8,177
& $2.47\mathrm{e}{-128}$
& $7.18\mathrm{e}{-125}$
& $5.95\mathrm{e}{-118}$\\
$(32,8)$
& 8,448 & 32,753
& $1.58\mathrm{e}{-125}$
& $1.54\mathrm{e}{-121}$
& $1.22\mathrm{e}{-115}$\\
$(64,4)$
& 16,640 & 65,529
& $2.03\mathrm{e}{-153}$
& $3.45\mathrm{e}{-151}$
& $1.20\mathrm{e}{-145}$\\
$(64,8)$
& 33,280 & 131,057
& $8.29\mathrm{e}{-121}$
& $5.12\mathrm{e}{-118}$
& $7.70\mathrm{e}{-113}$\\
$(128,4)$
& 66,048 & 262,137
& $2.02\mathrm{e}{-152}$
& $3.24\mathrm{e}{-149}$
& $1.23\mathrm{e}{-144}$\\
$(128,8)$
& 132,096 & 524,273
& $8.23\mathrm{e}{-117}$
& $3.38\mathrm{e}{-114}$
& $9.97\mathrm{e}{-110}$\\
\bottomrule
\end{tabular}
}
\caption{Representative 180-digit recovery results with $t=1$.
Each setting contains 100 targets, and the algorithm recovers all $H$
heads for every target.
\textnormal{Params.} denotes $H(d^2+d)$.}
\label{tab:exact-recovery}
\end{table}

\paragraph{Approximate oracle outputs.}
We first test the local stability prediction by adding errors bounded by $\tau$ to high-precision oracle responses while using sufficiently accurate offline arithmetic to make rounding errors negligible. For each of $(d,H)=(3,2),(3,3),(4,4)$, we sample 10 models and run 10 independent perturbations per model at each value of $\tau$. The median slope of $\log E_{\rm param}$ against $\log\tau$ is $1.00$ in every setting, while the empirical amplification $E_{\rm param}/\tau$ varies by more than eight orders of magnitude, supporting local linear scaling with a highly variable proportionality factor.

As a practical instance, we round only the scalar oracle responses to IEEE 754 binary64~\citep{ieee754} while keeping the target models, queried inputs, and offline arithmetic fixed. The algorithm returns all $H$ heads with $E_{\rm param}<10^{-2}$ in 99/100 runs for $(d,H)=(3,2)$, 1/100 for $(3,4)$, and 0/100 for $(3,8)$ and $(8,4)$. All paired high-precision runs satisfy the same criterion. Thus, the failures arise from oracle rounding and persist despite high-precision offline computation. Full results are given in \suppsec{supp:approx-experiments}.

%% file: sections/limitations_and_conclusion.tex
\section{Conclusion} We present a randomized algorithm for recovering the canonical representation of multi-head softmax attention from scalar final-token outputs without orthogonal-subspace assumptions or known subspace bases. Under exact oracle responses and the stated computational assumptions, the algorithm recovers this representation up to permutation with probability one. Repeated-token rational interpolation separates the head contributions. Queries formed by adding selected directions establish a consistent head labeling across local decoders, after which linear systems recover every canonical pair $(W_h,v_h)$. When $H$ is known, the algorithm uses exactly $4Hd^2-2H+1$ value queries of maximum length $2H+1$. If only an upper bound $H_0$ is known, the bounds become $4H_0d^2-2H_0+1$ and $2H_0+1$.

For approximate outputs, we give conditions under which queries distinguish the heads and reconstruction is well conditioned. Sufficiently small output errors then yield proportional parameter errors with a model- and query-dependent constant. A two-head construction rules out a uniform bound over all identifiable models. Experiments confirm end-to-end recovery with high-precision outputs and show sensitivity to IEEE~754 binary64 outputs. Finally, the odd-component reduction recovers the effective attention heads through a bias-free ReLU feed-forward network. Under additional assumptions, we construct a functionally equivalent one-layer ReLU Transformer without a separate feed-forward-network learner or the restriction $m\le d$.

%% file: supplementary/supplementary_material.tex
\setcounter{secnumdepth}{2}
\input{supplementary/sections/01_canonical_identifiability}
\input{supplementary/sections/02_genericity_and_exact_recovery}
\input{supplementary/sections/03_further_exact_recovery_variants}
\input{supplementary/sections/03_01_known_upper_bound_h}
\input{supplementary/sections/03_02_low_rank_recovery}
\input{supplementary/sections/04_stability_under_approximate_outputs}
\input{supplementary/sections/04_01_proof_explicit_constants}
\input{supplementary/sections/04_02_two_head_lower_bound}
\input{supplementary/sections/05_binary_membership_queries}
\input{supplementary/sections/06_one_layer_relu_transformer}
\input{supplementary/sections/06_01_breakpoint_recovery}
\input{supplementary/sections/06_02_sign_pattern_spanning_scale_search}
\input{supplementary/sections/07_numerical_implementation_experiments}
\input{supplementary/sections/07_01_numerical_realization_choice_t}
\input{supplementary/sections/07_02_experimental_protocol}
\input{supplementary/sections/07_03_high_precision_recovery}
\input{supplementary/sections/07_04_approximate_binary64_outputs}

%% file: supplementary/sections/01_canonical_identifiability.tex
\section{Canonical Identifiability}
\label{supp:canonical}

This section proves Proposition~\ref{prop:canonical-identifiability}.  We
first record that canonicalization preserves the scalar input--output
function.  For every \(A\in\R^{d\times d}\), the attention weights of all
raw heads with \(W_h=A\) are identical. Since a head output
is linear in its value vector,
\[
 \sum_{\substack{h\in[H_0]\\ W_h=A}} f_{A,v_h}(X)
 =f_{A,C_M(A)}(X).
\]
Terms for which \(C_M(A)=0\) vanish.  Consequently, for every input
sequence \(X\),
\begin{equation}
 F_M(X)
 =
 \sum_{(A,c)\in\operatorname{Can}(M)} f_{A,c}(X).
 \label{eq:supp-canonical-reduction}
\end{equation}

It remains to show that heads with distinct matrices \(W_h\) are linearly
independent even when only the scalar output at the final token is
observed.  For the three-token input family used below, the part of a head
output that depends on the perturbation parameter \(t\) is proportional to
\[
 p_s(t)\coloneqq\frac{\exp(st)}{2+\exp(st)},
 \qquad s,t\in\R,
\]
where \(s=u^\top A q\) for the corresponding matrix \(A\).
The following lemma shows that these functions cannot cancel when their
scalar parameters are pairwise distinct.

\begin{lemma}[Linear independence of three-token curves]
\label{lem:supp-three-token-curves}
If \(s_1,\ldots,s_k\) are pairwise distinct real numbers, then the
functions \(p_{s_1},\ldots,p_{s_k}\) are linearly independent over
\(\R\).
\end{lemma}

\begin{proof}
For \(s\ne0\), define the complex-valued function
\[
 P_s(z)\coloneqq\frac{\exp(sz)}{2+\exp(sz)},
 \qquad z\in\mathbb C,
\]
and set \(P_0(z)\coloneqq 1/3\).  The function \(P_s\) agrees with \(p_s\)
on the real line.  Since its numerator and denominator are entire and the
denominator has only isolated zeros, \(P_s\) is meromorphic.

For \(s\ne0\), the poles of \(P_s\) are
\[
 z=\frac{\log 2+\mathrm{i}(2n+1)\pi}{s},
 \qquad n\in\mathbb Z.
\]
Every such pole is simple with residue \(1/s\).  Moreover, the pole sets
corresponding to two distinct nonzero real numbers \(s\) and \(s'\) are
disjoint.  Indeed, equality of one pole from each set implies, by comparing
real parts,
\[
 \frac{\log 2}{s}=\frac{\log 2}{s'},
\]
and hence \(s=s'\).

Now suppose that
\[
 \sum_{\ell=1}^k a_\ell p_{s_\ell}(t)=0
\]
for every real \(t\).  Let \(\mathcal P\) be the union of the poles of
\(P_{s_1},\ldots,P_{s_k}\), and define
\[
 D\coloneqq\mathbb C\setminus\mathcal P.
\]
The set \(\mathcal P\) is discrete, so \(D\) is connected.  The function
\[
 F(z)\coloneqq\sum_{\ell=1}^k a_\ell P_{s_\ell}(z)
\]
is holomorphic on \(D\) and vanishes on the real line.  The identity
theorem therefore gives \(F=0\) throughout \(D\).

Consider a pole belonging to \(P_{s_\ell}\) for some \(s_\ell\ne0\).
Since the pole sets are disjoint, the residue of \(F\) at this pole is
\(a_\ell/s_\ell\).  Because \(F\) vanishes on the corresponding punctured
neighborhood, this residue must be zero.  Hence \(a_\ell=0\) for every
\(s_\ell\ne0\).  If one of the \(s_\ell\) is zero, the only remaining term
is the constant \(a_\ell P_0=a_\ell/3\), so its coefficient is also zero.
\end{proof}

Lemma~\ref{lem:supp-three-token-curves} applies once the scalars
\(s_\ell=u^\top A_\ell q\) are pairwise distinct.  We next show that,
for pairwise distinct matrices \(A_1,\ldots,A_k\), one can choose \(u\)
and \(q\) so that this condition holds simultaneously for every
\(\ell\).

\begin{lemma}[Linear independence of distinct scalar-output heads]
\label{lem:supp-distinct-heads}
Let \(A_1,\ldots,A_k\in\R^{d\times d}\) be pairwise distinct.  If
\[
 \sum_{\ell=1}^k f_{A_\ell,c_\ell}(X)=0
 \qquad\text{for every }X\in\R^{3\times d},
\]
then \(c_\ell=0\) for every \(\ell\in[k]\).
\end{lemma}

\begin{proof}
For \(u,q\in\R^d\) and \(t\in\R\), consider
\[
 X(t)=[(q+tu)^\top;q^\top;q^\top].
\]
The final token is \(q\). For head \(\ell\), the quantities inside the
exponential for the first token and either copy of \(q\) differ by
\(t u^\top A_\ell q\).
Writing
\[
 s_\ell=u^\top A_\ell q,
 \qquad
 b_\ell=u^\top c_\ell,
\]
gives
\begin{equation}
 f_{A_\ell,c_\ell}(X(t))
 =q^\top c_\ell+t b_\ell p_{s_\ell}(t).
 \label{eq:supp-three-token-head}
\end{equation}

Define
\[
 \mathcal U
 \coloneqq
 \R^d\setminus
 \bigcup_{1\le\ell<j\le k}
 \ker\!\bigl((A_\ell-A_j)^\top\bigr).
\]
Every kernel removed in this definition is a proper linear subspace
because \(A_\ell\ne A_j\).  Hence \(\mathcal U\) is a nonempty open set.
Fix \(u\in\mathcal U\).  For every \(\ell\ne j\), the linear functional
\[
 q\longmapsto u^\top(A_\ell-A_j)q
\]
is nonzero.  We may therefore choose \(q\) outside the finite union of its
zero hyperplanes.  For this choice, \(s_1,\ldots,s_k\) are pairwise
distinct.

Apply the assumed head identity to \(X(t)\), subtract its value at \(t=0\),
and use \eqref{eq:supp-three-token-head}.  For every \(t\ne0\),
\[
 \sum_{\ell=1}^k b_\ell p_{s_\ell}(t)=0.
\]
Continuity gives the same identity at \(t=0\).  By
Lemma~\ref{lem:supp-three-token-curves}, \(b_\ell=u^\top c_\ell=0\) for
every \(\ell\).  This conclusion holds for every \(u\) in the nonempty
open set \(\mathcal U\).  A linear functional that vanishes on a nonempty
open set vanishes everywhere, so \(c_\ell=0\) for all \(\ell\).
\end{proof}

\begin{proof}[Proof of Proposition~\ref{prop:canonical-identifiability}\unskip]
If \(\operatorname{Can}(M)=\operatorname{Can}(M')\), then
\eqref{eq:supp-canonical-reduction} immediately gives
\(F_M(X)=F_{M'}(X)\) for every sequence \(X\), and in particular for every
\(X\in\R^{3\times d}\).

Conversely, suppose that \(F_M(X)=F_{M'}(X)\) for every
\(X\in\R^{3\times d}\).  Enumerate the distinct matrices appearing as head parameters \(W_h\)
in either raw model as \(A_1,\ldots,A_k\), and define
\[
 c_\ell\coloneqq C_M(A_\ell)-C_{M'}(A_\ell).
\]
Applying \eqref{eq:supp-canonical-reduction} to both models yields
\[
 \sum_{\ell=1}^k f_{A_\ell,c_\ell}(X)=0
 \qquad\text{for every }X\in\R^{3\times d}.
\]
Lemma~\ref{lem:supp-distinct-heads} gives \(c_\ell=0\) for all \(\ell\).
Thus \(C_M(A)=C_{M'}(A)\) for every \(A\in\R^{d\times d}\).  Removing the
zero aggregates from these identical collections proves
\(\operatorname{Can}(M)=\operatorname{Can}(M')\).
\end{proof}

%% file: supplementary/sections/02_genericity_and_exact_recovery.tex
\section{Probability-One Local Recovery and Consistent Head Labeling}
\label{supp:exact}

This section proves that the randomized construction in
\eqref{eq:conditioned-bases} satisfies all local-recovery and consistent
head-labeling conditions simultaneously with probability one. Throughout,
the model is fixed, \(t\ne0\) is fixed, and the probability is only
over the learner's draw of \(\mathbf U\) and \(\mathbf Q\).

Write
\[
 \Lambda_{\mathbf U}=\operatorname{diag}(\lambda_1,\ldots,\lambda_d),
 \qquad
 \Lambda_{\mathbf Q}=\operatorname{diag}(\mu_1,\ldots,\mu_d),
\]
where the \(\lambda_i\) and \(\mu_j\) are independent and uniform on
\([1,2]\).  Let \(a_i^\top\) be row \(i\) of \(O_{\mathbf U}\), and let
\(b_j\) be column \(j\) of \(O_{\mathbf Q}\).  Then
\[
 u_i=\lambda_i a_i,
 \qquad
 q_j=\mu_j b_j.
\]
The two orthonormal systems \((a_i)_{i=1}^d\) and
\((b_j)_{j=1}^d\) are independent.  Let \(g,h\in\R^d\) be independent
random vectors with i.i.d.\ standard Gaussian entries.  For every fixed
\(i\) and \(j\), the pair \((a_i,b_j)\) has the same distribution as
\[
 \left(\frac{g}{\|g\|_2},\frac{h}{\|h\|_2}\right).
\]
The two orthonormal systems are also independent of all
\(\lambda_i\) and \(\mu_j\).

\begin{lemma}[Probability-zero identities for the sampled vectors]
\label{lem:supp-conditioned-null}
The following statements hold.
\begin{enumerate}
 \item For every fixed nonzero \(x\in\R^d\) and every \(i\),
 \(\Pr(u_i^\top x=0)=0\).
 \item For every fixed nonzero \(A\in\R^{d\times d}\) and every \(i\),
 \(\Pr(A^\top u_i=0)=0\).
 \item For \(i\ne j\) and fixed matrices \(A,B\in\R^{d\times d}\) that
 are not both zero,
 \[
  \Pr(A^\top u_i+B^\top u_j=0)=0.
 \]
 \item The analogous statements hold for \(q_i,q_j\).
\end{enumerate}
\end{lemma}

\begin{proof}
For the first statement, \(u_i^\top x=0\) if and only if
\(g^\top x=0\), because \(u_i=\lambda_i a_i\), \(\lambda_i\ne0\), and
\(a_i\) has the same distribution as \(g/\|g\|_2\).  For every fixed
\(x\ne0\), the scalar \(g^\top x\) is Gaussian with variance
\(\|x\|_2^2>0\).  Hence
\[
 \Pr(u_i^\top x=0)=0.
\]

For the second statement, choose \(y\in\R^d\) such that \(Ay\ne0\).
The event \(A^\top u_i=0\) implies \(u_i^\top Ay=0\), so the first
statement gives
\[
 \Pr(A^\top u_i=0)=0.
\]

For the third statement, the second statement implies that, with
probability one over the orthogonal matrix, at least one of
\(A^\top a_i\) and \(B^\top a_j\) is nonzero.  On this event, condition
on the orthogonal matrix and on \(\lambda_j\).
If \(A^\top a_i\ne0\), the equation
\[
 \lambda_i A^\top a_i+\lambda_j B^\top a_j=0
\]
determines at most one value of \(\lambda_i\).  Since \(\lambda_i\) is
uniform on \([1,2]\), this occurs with probability zero.  If
\(A^\top a_i=0\), then \(B^\top a_j\ne0\), so the equality is
impossible.  The proof for \(q_i,q_j\) is identical.
\end{proof}

We will also use the following consequence.  Let \(Z\in\R^d\) be determined by \(\mathbf U\), and suppose that
\(\Pr(Z\ne0)=1\). Conditional on \(Z\),
the independence of \(\mathbf Q\) and \(\mathbf U\) and the Gaussian
representation above give
\[
 \Pr(q_j^\top Z=0\mid Z)=0.
\]
Therefore
\begin{equation}
 \Pr(q_j^\top Z=0)=0.
 \label{eq:supp-independent-inner-product}
\end{equation}
The same statement holds with the roles of \(\mathbf U\) and
\(\mathbf Q\) interchanged.

\begin{proposition}[Local recovery and consistent head labeling succeed
with probability one]
\label{prop:supp-local-decoding-matching}
For every fixed canonical model satisfying \eqref{eq:id}, the following
events occur simultaneously with probability one.
\begin{enumerate}
 \item For every pair \((u,q)\) in
 \eqref{eq:grid}--\eqref{eq:qbridge}, the values
 \(s_1(u,q),\ldots,s_H(u,q)\) are pairwise distinct and every
 \(c_h(u,q)\) is nonzero.
 \item For all indices in Lemma~\ref{lem:labeling}, each identity in
 \eqref{eq:add-u} and \eqref{eq:add-q} can hold only
 when the entries belong to the same head.
\end{enumerate}
\end{proposition}

\begin{proof}
Fix two heads \(h\ne g\) and write
\(\Delta_{hg}=W_h-W_g\ne0\).  For \((u,q)=(u_i,q_j)\), the equality
\(s_h(u,q)=s_g(u,q)\) is equivalent to
\[
 u_i^\top\Delta_{hg}q_j=0.
\]
By Lemma~\ref{lem:supp-conditioned-null},
\(\Delta_{hg}^\top u_i\ne0\) with probability one.  Equation
\eqref{eq:supp-independent-inner-product} shows that this equality has
probability zero.

For \((u,q)=(u_1+u_i,q_1)\), the corresponding equality is
\[
 (u_1+u_i)^\top\Delta_{hg}q_1=0.
\]
Lemma~\ref{lem:supp-conditioned-null}, applied with
\(A=B=\Delta_{hg}\), gives
\(\Delta_{hg}^\top(u_1+u_i)\ne0\) with probability one, after which
\eqref{eq:supp-independent-inner-product} applies.

For \((u,q)=(u_i,q_1+q_j)\), the same argument with the roles of
\(\mathbf U\) and \(\mathbf Q\) interchanged excludes
\[
 u_i^\top\Delta_{hg}(q_1+q_j)=0.
\]

The possible choices of \(u\) are \(u_i\) and
\(u_1+u_i\). For every \(v_h\ne0\),
the first part of Lemma~\ref{lem:supp-conditioned-null} gives
\(u_i^\top v_h\ne0\) with probability one.  For the sum, let \(e_1\) be the
first standard basis vector and apply the third part of the lemma with
\(A=B=v_he_1^\top\).  Since
\[
 A^\top(u_1+u_i)=e_1v_h^\top(u_1+u_i),
\]
we also obtain
\[
 (u_1+u_i)^\top v_h\ne0
\]
with probability one.  Since \(t\ne0\), all corresponding
\(c_h(u,q)=t u^\top v_h\) are nonzero.  The exponential function is
injective on \(\R\), so pairwise distinct \(s_h\)-values also give
pairwise distinct \(r_h=\exp(s_h)\)-values.

It remains to exclude incorrect head-labeling equalities. For heads
\(h_1,h_2,h_3\), the equality in \eqref{eq:add-u} is
\[
 s_{h_1}(u_1+u_i,q_1)
 =
 s_{h_2}(u_1,q_1)+s_{h_3}(u_i,q_1).
\]
After division by \(t\) and rearrangement, this is
\begin{equation}
 q_1^\top\!\left[
  (W_{h_1}-W_{h_2})^\top u_1
  +(W_{h_1}-W_{h_3})^\top u_i
 \right]=0.
 \label{eq:supp-false-left-match}
\end{equation}
Unless \(h_1=h_2=h_3\), pairwise distinctness of the \(W_h\) implies that
the two matrix differences appearing in the brackets are not both zero.
Lemma~\ref{lem:supp-conditioned-null} makes the bracketed vector nonzero
with probability one, and independence of \(q_1\) then makes
\eqref{eq:supp-false-left-match} a probability-zero event.

Similarly, the equality in \eqref{eq:add-q} is equivalent to
\begin{equation}
 u_i^\top\!\left[
  (W_{h_1}-W_{h_2})q_1
  +(W_{h_1}-W_{h_3})q_j
 \right]=0.
 \label{eq:supp-false-right-match}
\end{equation}
If the three indices are not all equal, the bracketed vector is nonzero with probability one by the version of
Lemma~\ref{lem:supp-conditioned-null} for \(\mathbf Q\).  Independence of
\(u_i\) excludes \eqref{eq:supp-false-right-match}.  When all three
indices agree, both additive identities hold deterministically by
bilinearity.

There are finitely many choices of \((u,q)\) and finitely many head and
index combinations. The union of their probability-zero failure events still
has probability zero, proving both assertions simultaneously.  When \(d=1\), no consistent head labeling is required, and the arguments for
\((u_1,q_1)\) prove the proposition.
\end{proof}

\paragraph{Completion of exact recovery.}
On the event of Proposition~\ref{prop:supp-local-decoding-matching}, every
local response has the reduced rational representation
\[
 \mathcal R(z)
 =
 \sum_{h=1}^H\frac{c_h r_h}{z+r_h}.
\]
Indeed, the \(r_h\) are pairwise distinct and the residue \(c_hr_h\) at
the pole \(-r_h\) is nonzero.  Thus no pole cancels, the reduced
denominator has degree \(H\), and the rational interpolation lemma of
Section~\ref{sec:local} recovers precisely
\[
 D_t(u,q)=\{(t u^\top W_hq,t u^\top v_h):h\in[H]\}.
\]
Proposition~\ref{prop:supp-local-decoding-matching} also verifies exactly
the hypotheses of Lemma~\ref{lem:labeling}.  That lemma assigns one common head labeling to all sets
\(D_t(u_i,q_j)\) and then applies \eqref{eq:reconstruct}.  Since every singular value of
\(\mathbf U\) and \(\mathbf Q\) lies in \([1,2]\), both matrices are
invertible for every draw.  This proves the probability-one recovery claim
in Theorem~\ref{thm:main}.

For completeness, the algorithm uses
\[
 d^2+(d-1)+d(d-1)=2d^2-1
\]
local decoders.  Each uses the \(2H\) repeated-token responses with
\(m=1,\ldots,2H\). All local decoders using $q_1$ use the same queried one-token output,
$F_M([q_1^\top])$. After they recover the \(v_h\), every other
one-token output is computed exactly from
\[
 F_M([q^\top])=q^\top\sum_{h=1}^H v_h.
\]
The resulting query count is
\[
 2H(2d^2-1)+1=4Hd^2-2H+1,
\]
and the largest value of \(m\) gives maximum sequence length \(2H+1\).
All query directions and repetition counts are chosen before any oracle
response is observed, so the algorithm is nonadaptive.

%% file: supplementary/sections/03_further_exact_recovery_variants.tex
\section{Further Exact-Recovery Variants}
\label{supp:variants}

This section proves two exact-recovery variants.
Section~\ref{supp:unknown-h} considers the case where the learner knows
\(H_0\), but not \(H\).  Section~\ref{supp:low-rank} considers the case
where each \(W_h\) satisfies the known rank bound
\(\operatorname{rank}(W_h)\le r\).

%% file: supplementary/sections/03_01_known_upper_bound_h.tex
\subsection{Recovery with a Known Upper Bound on $H$}
\label{supp:unknown-h}

Let \(H_0\) be the raw head count.  Canonicalization gives
\(0\le H\le H_0\), but \(H\) need not be known to the learner.  We extend the rational interpolation argument in Section~\ref{sec:local} to the case where only \(H_0\) is known.

\begin{lemma}[Least-degree interpolation]
\label{lem:supp-least-degree}
Let \(\mathcal R=\mathcal P/\mathcal Q\) be a reduced proper rational
function whose monic denominator has degree \(H\le H_0\), and assume that
\(\mathcal Q(m)\ne0\) for \(m=1,\ldots,2H_0\).  Then the values
\(\mathcal R(1),\ldots,\mathcal R(2H_0)\) uniquely determine \(\mathcal R\) among all proper rational functions
with denominator degree at most \(H_0\) whose denominators are nonzero at
\(1,\ldots,2H_0\). Moreover, \(H\) is the smallest integer
\(k\in\{0,\ldots,H_0\}\) for which the samples admit such a rational
representation of denominator degree \(k\).
\end{lemma}

\begin{proof}
Let \(\mathcal P_1/\mathcal Q_1\) and
\(\mathcal P_2/\mathcal Q_2\) satisfy the stated conditions and agree at
all \(2H_0\) sample points.  The polynomial
\[
 \mathcal P_1\mathcal Q_2-\mathcal P_2\mathcal Q_1
\]
has degree at most \(2H_0-1\) and vanishes at \(2H_0\) distinct points.
It is therefore zero, so the two rational functions are identical.

The true reduced fraction is feasible for \(k=H\).  If some \(k<H\)
were feasible, uniqueness would give a representation of
\(\mathcal R\) with denominator degree less than \(H\), contradicting
the definition of \(H\).  For \(k=0\), properness permits only the zero
function, so \(k=0\) is feasible exactly when all samples are zero.
\end{proof}

The learner determines \(H\) by considering
\(k=0,\ldots,H_0\) in increasing order. The case \(k=0\) is feasible
exactly when all samples are zero. For \(k\ge1\), the numerator and
denominator take the form
\[
 \mathcal Q_k(z)=z^k+\sum_{\ell=0}^{k-1}q_\ell z^\ell,
 \qquad
 \mathcal P_k(z)=\sum_{\ell=0}^{k-1}p_\ell z^\ell,
\]
and their coefficients must satisfy
\begin{equation}
 \mathcal P_k(m)=\mathcal R(m)\mathcal Q_k(m),
 \qquad m=1,\ldots,2H_0.
 \label{eq:supp-upper-bound-system}
\end{equation}
The value \(k\) is discarded if
\eqref{eq:supp-upper-bound-system} has no solution. It is also discarded
if, for some \(m\in\{1,\ldots,2H_0\}\), every solution satisfies
\(\mathcal Q_k(m)=0\); otherwise it is retained.
Whether \(\mathcal Q_k(m)=0\) holds for every solution can be determined
by exact linear algebra.  Over \(\R\), finitely many such equalities
cannot cover the nonempty solution set unless one of them holds for every
solution. Thus a retained \(k\) has a solution satisfying
\(\mathcal Q_k(m)\ne0\) for every \(m\).
By Lemma~\ref{lem:supp-least-degree}, the smallest retained \(k\) is \(H\).

\begin{theorem}[Exact recovery from a known head-count upper bound]
\label{thm:supp-upper-bound}
Let \(M\) be a raw \(H_0\)-head model with canonical representation
\(\operatorname{Can}(M)=\{(W_h,v_h):h\in[H]\}\), where
\(0\le H\le H_0\).  Suppose the learner knows \(d,H_0\), but not \(H\).
For any fixed \(t\ne0\), a nonadaptive randomized algorithm recovers
\(\operatorname{Can}(M)\) exactly with probability one.  It uses exactly
\[
 2H_0(2d^2-1)+1
 =
 4H_0d^2-2H_0+1
\]
value queries, each of length at most \(2H_0+1\).
\end{theorem}

\begin{proof}
The algorithm samples \(\mathbf U,\mathbf Q\) according to
\eqref{eq:conditioned-bases}. It first queries
\(F_M([q_1^\top])\).
It then makes the corresponding repeated-token query for every pair
listed in \eqref{eq:grid}--\eqref{eq:qbridge} and every
\(m=1,\ldots,2H_0\). If $H=0$, then \eqref{eq:supp-canonical-reduction} gives
$F_M\equiv0$. Hence every queried output difference is zero, and the
algorithm returns the empty canonical set.

For \(H\ge1\), Proposition~\ref{prop:supp-local-decoding-matching}
shows that, with probability one, for every pair listed in
\eqref{eq:grid}--\eqref{eq:qbridge}, the values \(s_1,\ldots,s_H\) are pairwise
distinct and every \(c_h\) is nonzero.  Hence the values
\(r_h=\exp(s_h)\) are pairwise distinct and every residue \(c_hr_h\) is
nonzero.  Therefore, \(\mathcal R\) is reduced and
\(\deg\mathcal Q=H\).  Since its poles are \(-r_h<0\),
\(\mathcal Q(m)\ne0\) for \(m=1,\ldots,2H_0\).

Applying Lemma~\ref{lem:supp-least-degree} to the response for
\((u_1,q_1)\) determines \(H\).  Applying the same interpolation and the
local-decoding procedure in Section~\ref{sec:local} to every pair
listed in \eqref{eq:grid}--\eqref{eq:qbridge} recovers every corresponding set
\(D_t(u,q)\).  Consistent head labeling and reconstruction then follow
from Lemma~\ref{lem:labeling}.

There are $2d^2-1$ local decoders, each using $2H_0$
repeated-token responses, and the only one-token output queried directly
is $F_M([q_1^\top])$. Thus the stated query count follows.  The largest repetition count is
\(2H_0\), giving maximum sequence length \(2H_0+1\).
\end{proof}

%% file: supplementary/sections/03_02_low_rank_recovery.tex
\subsection{Low-Rank Recovery}
\label{supp:low-rank}

We extend the single-head low-rank recovery construction of
\citet[Theorem~5.1]{bhattamishra2026attention} to canonical multi-head
models. Suppose that \(H\) is known and that
\(\operatorname{rank}(W_h)\le r\) for every \(h\in[H]\).
Let \(u_0,\ldots,u_{L-1},q_0,\ldots,q_{L-1}\in\R^d\) be independent
random vectors with i.i.d.\ standard Gaussian entries.  Define
\[
 \mathcal A_L(Z)
 \coloneqq
 \bigl(u_k^\top Zq_k\bigr)_{k=0}^{L-1}.
\]
For \(Z\in\R^{d\times d}\), let \(\|Z\|_*\) denote the sum of its
singular values.

By Corollary~2.1 of \citet{caizhang2015rop}, there are universal constants
\(C_{\rm rop},c_{\rm rop}>0\) such that, if
\begin{equation}
 L\ge 2C_{\rm rop}rd,
 \label{eq:supp-rop-count}
\end{equation}
then, with probability at least \(1-\exp(-c_{\rm rop}L)\), every
\(Z\in\R^{d\times d}\) satisfying \(\operatorname{rank}(Z)\le r\) is
the unique solution of
\begin{equation}
 \min_{X\in\R^{d\times d}}\|X\|_*
 \quad\text{subject to}\quad
 \mathcal A_L(X)=\mathcal A_L(Z).
 \label{eq:supp-nuclear-program}
\end{equation}
This statement holds simultaneously for every such \(Z\), so the same
operator \(\mathcal A_L\) recovers all \(W_h\).  We assume that
\eqref{eq:supp-nuclear-program} can be solved exactly.

\paragraph{Query directions.}
The construction uses \(D_t(u_0,q_0)\) and, for every
\(k=1,\ldots,L-1\), the four sets
\begin{equation}
 D_t(u_k,q_0),\quad D_t(u_0+u_k,q_0),\quad
 D_t(u_k,q_k),\quad D_t(u_k,q_0+q_k).
 \label{eq:supp-low-rank-schedule}
\end{equation}
Together with \(D_t(u_0,q_0)\), the sets in
\eqref{eq:supp-low-rank-schedule} give
\[
 1+4(L-1)=4L-3
\]
local decoders.

\begin{lemma}[Probability-one local recovery and consistent head labeling]
\label{lem:supp-low-rank-decoding-matching}
Fix pairwise distinct \(W_1,\ldots,W_H\) and nonzero
\(v_1,\ldots,v_H\).  With probability one, the following statements hold
simultaneously.
\begin{enumerate}
 \item For \(D_t(u_0,q_0)\) and every set in
 \eqref{eq:supp-low-rank-schedule}, the values
 \(s_1(u,q),\ldots,s_H(u,q)\) are pairwise distinct and every
 \(c_h(u,q)\) is nonzero.

 \item For every \(k=1,\ldots,L-1\) and
 \(h_1,h_2,h_3\in[H]\),
 \begin{equation}
 \begin{aligned}
 &s_{h_1}(u_0,q_0)+s_{h_2}(u_k,q_0)
   =s_{h_3}(u_0+u_k,q_0)\\
 &\hspace{25mm}\Longrightarrow h_1=h_2=h_3,
 \end{aligned}
 \label{eq:supp-low-rank-left-match}
 \end{equation}
 and
 \begin{equation}
 \begin{aligned}
 &s_{h_1}(u_k,q_0)+s_{h_2}(u_k,q_k)
   =s_{h_3}(u_k,q_0+q_k)\\
 &\hspace{25mm}\Longrightarrow h_1=h_2=h_3.
 \end{aligned}
 \label{eq:supp-low-rank-right-match}
 \end{equation}
\end{enumerate}
Consequently, every local decoder used above recovers all \(H\) elements
of \(D_t(u,q)\), and \eqref{eq:supp-low-rank-left-match} and
\eqref{eq:supp-low-rank-right-match} determine consistent head labels.
\end{lemma}

\begin{proof}
A nonzero polynomial in the independent standard Gaussian coordinates
defined above vanishes with probability zero. For \(h\ne g\), the equality
\(s_h(u,q)=s_g(u,q)\) is equivalent to
\[
 u^\top(W_h-W_g)q=0.
\]
This is a nonzero polynomial because \(W_h\ne W_g\).  Similarly,
\(c_h(u,q)=0\) is equivalent to \(u^\top v_h=0\), which is nonzero
because \(v_h\ne0\). For the polynomial expressing the equality at
\((u_0+u_k,q_0)\), setting \(u_k=0\) yields the nonzero polynomial
associated with \((u_0,q_0)\). For the polynomial expressing the equality at
\((u_k,q_0+q_k)\), setting \(q_k=0\) yields the nonzero polynomial
associated with \((u_k,q_0)\). For \(c_h(u,q)=0\), the only additional case is
\((u_0+u_k)^\top v_h=0\), whose polynomial is nonzero after setting
\(u_k=0\).  Hence item 1 holds with probability one.

For indices \(h_1,h_2,h_3\) that are not all equal, the equality in
\eqref{eq:supp-low-rank-left-match} is equivalent to
\[
 u_0^\top(W_{h_1}-W_{h_3})q_0
 +
 u_k^\top(W_{h_2}-W_{h_3})q_0
 =
 0.
\]
This polynomial is identically zero only if
\(W_{h_1}=W_{h_3}\) and \(W_{h_2}=W_{h_3}\), which would imply
\(h_1=h_2=h_3\).  Therefore, it is nonzero for every prohibited match.

Likewise, the equality in \eqref{eq:supp-low-rank-right-match} is
equivalent to
\[
 u_k^\top(W_{h_1}-W_{h_3})q_0
 +
 u_k^\top(W_{h_2}-W_{h_3})q_k
 =
 0.
\]
This polynomial is identically zero only when
\(h_1=h_2=h_3\). There are finitely many pairs and head
triples, so the probability that any prohibited equality holds is zero.
\end{proof}

\paragraph{Recovery of \(v_h\) and measurements of \(W_h\).}
The elements of \(D_t(u_0,q_0)\) are given arbitrary initial labels.
For every \(k=1,\ldots,L-1\),
\eqref{eq:supp-low-rank-left-match} assigns
consistent labels to \(D_t(u_k,q_0)\) using
\(D_t(u_0+u_k,q_0)\).
\[
 s_h(u_0+u_k,q_0)
 =
 s_h(u_0,q_0)+s_h(u_k,q_0).
\]
The local decoder keeps each \(c_h\)-value paired with its
\(s_h\)-value, so this step labels
\[
 z_{kh}\coloneqq c_h(u_k,q_0)=t u_k^\top v_h,
 \qquad k=0,\ldots,L-1.
\]
Let
\begin{align*}
 U_L&=[u_0^\top;\ldots;u_{L-1}^\top]\in\R^{L\times d},\quad
 z_h=(z_{0h},\ldots,z_{L-1,h})^\top.
\end{align*}
When \(L\ge d\), the Gaussian matrix \(U_L\) has full column rank with
probability one, and hence
\begin{equation}
 v_h
 =
 t^{-1}(U_L^\top U_L)^{-1}U_L^\top z_h.
 \label{eq:supp-low-rank-values}
\end{equation}
This argument does not require the \(v_h\) to be distinct or linearly
independent.

After recovering the \(v_h\), every one-token output follows from
\[
 F_M([q^\top])=q^\top\sum_{h=1}^H v_h.
\]
For every \(k=1,\ldots,L-1\), the algorithm next recovers
\(D_t(u_k,q_k)\) and \(D_t(u_k,q_0+q_k)\).
Since \(D_t(u_k,q_0)\) is already labeled, the implication in
\eqref{eq:supp-low-rank-right-match} uniquely determines the labels of
\(D_t(u_k,q_k)\).
Together with the measurement from \(D_t(u_0,q_0)\), this gives
\begin{equation}
 y_{kh}
 \coloneqq
 t^{-1}s_h(u_k,q_k)
 =
 u_k^\top W_hq_k,
 \qquad k=0,\ldots,L-1.
 \label{eq:supp-labeled-rop}
\end{equation}
Thus \((y_{kh})_{k=0}^{L-1}=\mathcal A_L(W_h)\), and
the simultaneous guarantee stated above implies that
\eqref{eq:supp-nuclear-program} recovers all \(W_h\) with probability at
least \(1-\exp(-c_{\rm rop}L)\).

\begin{theorem}[Exact low-rank multi-head recovery]
\label{thm:supp-low-rank}
Let \(1\le r\le d\). Suppose that the learner knows \(H\) and \(r\),
and that the canonical heads satisfy \eqref{eq:id} and
\(\operatorname{rank}(W_h)\le r\) for every \(h\).
Let \(t\ne0\), and let \(L\) be any integer satisfying
\[
 L\ge
 \max\{d,\lceil2C_{\rm rop}rd\rceil\}.
\]
Under the exact computation assumptions, augmented by exact solution of
\eqref{eq:supp-nuclear-program}, the construction in
\eqref{eq:supp-low-rank-schedule} recovers every \((W_h,v_h)\) up to one
common permutation with probability at least
\(1-\exp(-c_{\rm rop}L)\).  It uses exactly
\[
 2H(4L-3)+1=8HL-6H+1
\]
value queries, each of length at most \(2H+1\).  Choosing the smallest
admissible \(L\) gives \(O(Hrd)\) value queries.
\end{theorem}

\begin{proof}
The probability-one event in
Lemma~\ref{lem:supp-low-rank-decoding-matching} makes every local decoder and
head-labeling step valid.  Since \(L\ge d\), it also holds with probability one
that \(U_L\) has full column rank, so
\eqref{eq:supp-low-rank-values} recovers all value vectors. Equation~\eqref{eq:supp-labeled-rop} gives
\(\mathcal A_L(W_h)\) with consistent head labels for every \(h\).
With probability at least \(1-\exp(-c_{\rm rop}L)\), every rank-\(r\)
matrix is the unique solution of \eqref{eq:supp-nuclear-program}; on
this event, the program recovers every \(W_h\) simultaneously.
Intersecting this event with the probability-one event in
Lemma~\ref{lem:supp-low-rank-decoding-matching} and the probability-one
full-column-rank event for \(U_L\) preserves the lower bound
\(1-\exp(-c_{\rm rop}L)\).

There are \(4L-3\) local decoders, each using \(2H\) repeated-token value
queries, and only the one-token output at \(q_0\) is queried.  This gives
\(2H(4L-3)+1\) queries.  The largest repeated-token query has \(2H\)
copies of its base token and one perturbed token, so its length is
\(2H+1\).
\end{proof}

%% file: supplementary/sections/04_stability_under_approximate_outputs.tex
\section{Stability under Approximate Outputs}
\label{supp:stability}

We prove Theorem~\ref{thm:robust-recovery} under the quantitative
separation and conditioning assumptions defined below. Throughout this
section, the model, the matrices $\mathbf U$ and $\mathbf Q$, and the
nonzero scale $t$ are fixed.

\paragraph{Approximate outputs.}
The approximate oracle satisfies
\begin{equation}
 \left|\widetilde F(X)-F_M(X)\right|\leq\tau
 \qquad\text{for every queried }X.
 \label{eq:supp-approx-oracle}
\end{equation}
We assume that all subsequent computations on the returned values are exact.
Let
\begin{equation}
 \begin{aligned}
  \mathcal A
  ={}&
  \{(u_i,q_j):i,j\in[d]\}{}\cup
  \{(u_1+u_i,q_1):i=2,\ldots,d\}{}\cup
  \{(u_i,q_1+q_j):i\in[d],\ j=2,\ldots,d\}.
 \end{aligned}
 \label{eq:supp-stab-schedule}
\end{equation}
be the $2d^2-1$ pairs in \eqref{eq:grid}--\eqref{eq:qbridge}. In addition to
the repeated-token queries, the approximate-output algorithm queries
\begin{equation}
 \begin{aligned}
  &\{[q_1^\top]\}\cup
  \{[q_j^\top],[(q_1+q_j)^\top]:j=2,\ldots,d\}.
 \end{aligned}
 \label{eq:supp-stab-baselines}
\end{equation}
Consequently, the required one-token output is observed directly for every
$(u,q)\in\mathcal A$. For $(u,q)\in\mathcal A$, $m\in[2H]$, and
$h\in[H]$, define, consistently with Section~\ref{sec:local},
\begin{equation}
 \begin{aligned}
  s_h(u,q)&=t u^\top W_hq,\qquad
  c_h(u,q)=t u^\top v_h,\qquad
  r_h(u,q)=\exp(s_h(u,q)),\\
  y_m(u,q)
  &=F_M(X_m^{(t)}(u,q))-F_M([q^\top])
    =\sum_{h=1}^H
     \frac{c_h(u,q)r_h(u,q)}{m+r_h(u,q)}.
 \end{aligned}
 \label{eq:supp-stab-local-quantities}
\end{equation}
The algorithm observes the corresponding output difference
\begin{equation}
 \widetilde y_m(u,q)
 =
 \widetilde F(X_m^{(t)}(u,q))-\widetilde F([q^\top]).
 \label{eq:supp-stab-noisy-sample}
\end{equation}
Since each of the two queried outputs has error at most $\tau$,
\begin{equation}
 \left|\widetilde y_m(u,q)-y_m(u,q)\right|\leq 2\tau.
 \label{eq:supp-stab-sample-error}
\end{equation}
Directly querying the required one-token outputs prevents errors in the
recovered $v_h$ from propagating to later local recoveries.
\paragraph{Quantitative separation and conditioning.}
We first record the ranges and separations of the quantities recovered at
each pair. All extrema below range over $(u,q)\in\mathcal A$ and,
where applicable, $h,g\in[H]$.  Define
\begin{align}
 r_{\min}
 &=
 \min r_h(u,q),
 \qquad
 r_{\max}
 =
 \max r_h(u,q),
 \label{eq:supp-stab-r-range}\\
 \delta_r
 &=
 \min_{(u,q)\in\mathcal A}\min_{h\neq g}
 |r_h(u,q)-r_g(u,q)|,
 \qquad
 c_{\min}
 =
 \min |c_h(u,q)|,
 \label{eq:supp-stab-primitive-margins}\\
 B_c
 &=
 \max_{(u,q)\in\mathcal A}
 \left\|(c_1(u,q),\ldots,c_H(u,q))\right\|_2.
 \label{eq:supp-stab-c-bound}
\end{align}
For the random construction used in the exact-recovery theorem,
$r_{\min}>0$, $\delta_r>0$, and $c_{\min}>0$ with probability one.  Once
$\mathbf U$ and $\mathbf Q$ have been drawn, however, these are fixed
instance-dependent quantities and may be arbitrarily small.

We next quantify the conditioning of the two linear systems used in local
recovery.  For $m\in[2H]$ and $a=0,\ldots,H-1$, let
\begin{equation}
 V_{m,a}=m^a,
 \qquad
 w_m=m^H.
 \label{eq:supp-stab-vandermonde}
\end{equation}
For $y=(y_1,\ldots,y_{2H})^\top$, define
\begin{equation}
 A(y)=
 \left[
  V,\ -\operatorname{diag}(y)V
 \right],
 \qquad
 b(y)=\operatorname{diag}(y)w.
 \label{eq:supp-stab-interpolation-system}
\end{equation}
If
\[
 \mathcal P(z)=\sum_{a=0}^{H-1}p_az^a,
 \qquad
 \mathcal Q(z)=z^H+\sum_{a=0}^{H-1}q_az^a,
\]
then the coefficient vector
$\theta=(p_0,\ldots,p_{H-1},q_0,\ldots,q_{H-1})^\top$
satisfies $A(y)\theta=b(y)$.

For $r=(r_1,\ldots,r_H)$, define the $2H\times H$ coefficient-recovery
matrix
\begin{equation}
 G(r)_{m,h}=\frac{r_h}{m+r_h},
 \qquad m\in[2H],\ h\in[H].
 \label{eq:supp-stab-G}
\end{equation}
To obtain bounds that hold for every pair in $\mathcal A$, define
\begin{equation}
 \gamma_{\rm int}
 =
 \min_{(u,q)\in\mathcal A}
 \sigma_{\min}(A(y(u,q))),
 \qquad
 \gamma_c
 =
 \min_{(u,q)\in\mathcal A}
 \sigma_{\min}(G(r(u,q))).
 \label{eq:supp-stab-conditioning}
\end{equation}
Appendix~\ref{supp:stability-proof} shows that distinct poles and nonzero
coefficients make $A(y)$ nonsingular. Distinct positive poles also make
$G(r)$ have full column rank. Since $\mathcal A$ is finite,
$\gamma_{\rm int}$ and $\gamma_c$ are positive whenever these conditions
hold for every $(u,q)\in\mathcal A$.

\paragraph{Consistent head labeling.}
Recovery at each pair determines the heads only up to a permutation. To
ensure that selecting the candidate with the smallest absolute residual
preserves a common labeling, define, for $d\geq2$,
\begin{align}
 \Delta_u
 &=
 \min_{\substack{i=2,\ldots,d\\
                  (h_1,h_2,h_3)\in[H]^3\\
                  h_1,h_2,h_3\ {\rm not\ all\ equal}}}
 \left|
  s_{h_1}(u_1+u_i,q_1)
  -s_{h_2}(u_1,q_1)
  -s_{h_3}(u_i,q_1)
 \right|,
 \label{eq:supp-stab-u-margin}\\
 \Delta_q
 &=
 \min_{\substack{i\in[d],\ j=2,\ldots,d\\
                  (h_1,h_2,h_3)\in[H]^3\\
                  h_1,h_2,h_3\ {\rm not\ all\ equal}}}
 \left|
  s_{h_1}(u_i,q_1+q_j)
  -s_{h_2}(u_i,q_1)
  -s_{h_3}(u_i,q_j)
 \right|,
 \label{eq:supp-stab-q-margin}\\
 \Delta_{\rm match}
 &=
 \min\{\Delta_u,\Delta_q\}.
 \label{eq:supp-stab-match-margin}
\end{align}
The identities hold only for triples from the same head precisely when
$\Delta_{\rm match}>0$. When $d=1$, no head labeling is required, and
we set $\Delta_{\rm match}=+\infty$.

The final linear reconstruction also depends on the conditioning of the
two basis matrices.  Write
\begin{equation}
 \begin{aligned}
  \mathbf U&=[u_1,\ldots,u_d]^\top,
  &\mathbf Q&=[q_1,\ldots,q_d],\\
  \gamma_U&=\sigma_{\min}(\mathbf U),
  &\gamma_Q&=\sigma_{\min}(\mathbf Q).
 \end{aligned}
 \label{eq:supp-stab-basis-margins}
\end{equation}
The construction in \eqref{eq:conditioned-bases} gives
$\gamma_U,\gamma_Q\geq1$.  The proof uses only their positivity and
therefore applies to any fixed invertible $\mathbf U$ and $\mathbf Q$
satisfying the preceding assumptions.

\paragraph{Recovery from approximate outputs.}
For each $(u,q)\in\mathcal A$, reconstruction begins by solving
\[
 A(\widetilde y)\widehat\theta=b(\widetilde y).
\]
The values $\widehat r_1,\ldots,\widehat r_H$ are then obtained as the
negatives of the roots of $\widehat{\mathcal Q}$, and
$\widehat s_h=\log\widehat r_h$. For this pair, write
\[
 c=(c_1(u,q),\ldots,c_H(u,q))^\top.
\]
The coefficient estimate is the least-squares solution
\begin{equation}
 \widehat c
 \coloneqq
 \arg\min_{z\in\R^H}
 \|G(\widehat r)z-\widetilde y\|_2.
 \label{eq:supp-stab-least-squares}
\end{equation}
This solution is unique whenever $G(\widehat r)$ has full column rank.
To obtain a common head labeling across different choices of $(u,q)$,
the algorithm selects the candidate with the smallest absolute residual
for each identity in \eqref{eq:add-u} and
\eqref{eq:add-q}.

%% file: supplementary/sections/04_01_proof_explicit_constants.tex
\subsection{Proof and Explicit Constants}
\label{supp:stability-proof}

We now derive explicit values of $\tau_0$ and $C_{\rm stab}$.  The bounds
are conservative, but display how the final error depends on the
conditioning quantities defined above.

\paragraph{Interpolation coefficients.}
Let
\begin{equation}
 V_H=\|V\|_2,
 \qquad
 S_H=\|w\|_2,
 \label{eq:supp-stab-VS}
\end{equation}
where $V$ and $w$ are defined in
\eqref{eq:supp-stab-vandermonde}, and set
\begin{align}
 B_\theta
 &=
 (1+r_{\max})^H
 +\sqrt H\,B_c r_{\max}(1+r_{\max})^{H-1},
 \label{eq:supp-stab-Btheta}\\
 C_\theta
 &=
 \frac{4(S_H+V_HB_\theta)}{\gamma_{\rm int}}.
 \label{eq:supp-stab-Ctheta}
\end{align}

\begin{lemma}[Interpolation perturbation]
\label{lem:supp-stab-interpolation}
If
\begin{equation}
 \tau\leq\frac{\gamma_{\rm int}}{4V_H},
 \label{eq:supp-stab-interpolation-threshold}
\end{equation}
then every noisy interpolation system is nonsingular and
\begin{equation}
 \|\widehat\theta-\theta\|_2\leq C_\theta\tau.
 \label{eq:supp-stab-theta-error}
\end{equation}
\end{lemma}

\begin{proof}
We first show that the exact interpolation matrix is nonsingular.
Suppose
$A(y)(P_0,Q_0)=0$, where both $P_0$ and $Q_0$ have degree less than $H$.
For the reduced rational function $\mathcal P/\mathcal Q$, the polynomial $P_0\mathcal Q-\mathcal P Q_0$ has degree at most $2H-1$ and vanishes at $1,\ldots,2H$.  It is therefore
the zero polynomial.  Distinct $r_h$ and nonzero $c_h$ imply that
$\mathcal P$ and $\mathcal Q$ are coprime.  Hence $\mathcal Q$ divides
$Q_0$.  Since $\deg Q_0<H=\deg\mathcal Q$, this gives $Q_0=0$, and then
$P_0=0$.  Thus $A(y)$ is nonsingular.

The coefficient $\ell_1$ norm of
$\mathcal Q(z)=\prod_h(z+r_h)$ is at most $(1+r_{\max})^H$.  Moreover,
\begin{equation}
 \mathcal P(z)
 =
 \sum_{h=1}^H
 c_hr_h\prod_{g\neq h}(z+r_g).
 \label{eq:supp-stab-P-expansion}
\end{equation}
Cauchy--Schwarz therefore gives
\[
 \|\mathcal P\|_{\mathrm{coeff},2}
 \leq
 \sqrt H\,B_cr_{\max}(1+r_{\max})^{H-1}.
\]
Consequently, $\|\theta\|_2\leq B_\theta$.

Let $e=\widetilde y-y$.  By
\eqref{eq:supp-stab-sample-error},
\begin{equation}
 \|A(\widetilde y)-A(y)\|_2\leq2\tau V_H,
 \qquad
 \|b(\widetilde y)-b(y)\|_2\leq2\tau S_H.
 \label{eq:supp-stab-Ab-perturbation}
\end{equation}
Weyl's inequality and
\eqref{eq:supp-stab-interpolation-threshold} give
\[
 \sigma_{\min}(A(\widetilde y))
 \geq
 \gamma_{\rm int}-2\tau V_H
 \geq\frac{\gamma_{\rm int}}2.
\]
Subtracting the exact and perturbed systems yields
\[
 A(\widetilde y)(\widehat\theta-\theta)
 =
 [b(\widetilde y)-b(y)]
 -
 [A(\widetilde y)-A(y)]\theta.
\]
Taking norms and using
$\|A(\widetilde y)^{-1}\|_2\leq2/\gamma_{\rm int}$
proves \eqref{eq:supp-stab-theta-error}.
\end{proof}

\paragraph{Perturbation of $r_h$ and $s_h$.}
Define
\begin{align}
 K_Q
 &=
 \left(
  \sum_{a=0}^{H-1}
  (r_{\max}+\delta_r/2)^{2a}
 \right)^{1/2},
 \label{eq:supp-stab-KQ}\\
 C_r
 &=
 \frac{2K_QC_\theta}{(\delta_r/2)^{H-1}},
 &
 C_s
 &=
 \frac{2C_r}{r_{\min}}.
 \label{eq:supp-stab-Cr-Cs}
\end{align}

\begin{lemma}[Perturbation bounds for $r_h$ and $s_h$]
\label{lem:supp-stab-roots}
Suppose Lemma~\ref{lem:supp-stab-interpolation} applies and
\begin{equation}
 C_r\tau<
 \min\left\{\frac{\delta_r}{2},\frac{r_{\min}}2\right\}.
 \label{eq:supp-stab-root-threshold}
\end{equation}
Then $\widehat{\mathcal Q}$ has $H$ distinct negative real roots.  The recovered values can be relabeled separately for each $(u,q)$ so that
\begin{equation}
 \max_h|\widehat r_h-r_h|\leq C_r\tau,
 \qquad
 \max_h|\widehat s_h-s_h|\leq C_s\tau.
 \label{eq:supp-stab-rs-error}
\end{equation}
\end{lemma}

\begin{proof}
The claim is immediate when $\tau=0$.  Assume $\tau>0$, put
$\epsilon_\theta=C_\theta\tau$, and define
\[
 \rho
 =
 \frac{2K_Q\epsilon_\theta}{(\delta_r/2)^{H-1}}
 =
 C_r\tau.
\]
On the circle $|z+r_h|=\rho$, the separation condition gives
\[
 |\mathcal Q(z)|
 =
 \prod_{g=1}^H|z+r_g|
 \geq
 \rho(\delta_r/2)^{H-1}
 =
 2K_Q\epsilon_\theta.
\]
Because $|z|\leq r_{\max}+\delta_r/2$ on these circles,
\eqref{eq:supp-stab-theta-error} gives
\[
 |\widehat{\mathcal Q}(z)-\mathcal Q(z)|
 \leq K_Q\epsilon_\theta.
\]
Rouch\'e's theorem places exactly one root of
$\widehat{\mathcal Q}$ in each disk centered at $-r_h$ with radius
$\rho$.  The disks are disjoint and account for all $H$ roots.  The
coefficients of $\widehat{\mathcal Q}$ are real, and each disk is invariant
under complex conjugation.  Uniqueness of the root in each disk therefore
forces that root to be real.  The bound $\rho<r_{\min}/2$ makes every root
negative.

It remains to take logarithms.  Both $r_h$ and $\widehat r_h$ are at least
$r_{\min}/2$, so the mean-value theorem for $\log r$ gives
\[
 |\log\widehat r_h-\log r_h|
 \leq\frac{2}{r_{\min}}|\widehat r_h-r_h|.
\]
This proves \eqref{eq:supp-stab-rs-error}.
\end{proof}

\paragraph{Perturbation of the recovered coefficients.}
Set
\begin{align}
 L_G
 &=
 \frac{1}{1+r_{\min}/2},
 &
 C_G
 &=
 \sqrt{2}\,H L_GC_r,
 \label{eq:supp-stab-CG}\\
 C_c
 &=
 \frac{4\sqrt{2H}+2C_GB_c}{\gamma_c}.
 \label{eq:supp-stab-Cc}
\end{align}

\begin{lemma}[Perturbation bound for $c$]
\label{lem:supp-stab-values}
Suppose Lemma~\ref{lem:supp-stab-roots} applies and
\begin{equation}
C_G\cdot\tau\leq\frac{\gamma_c}{2}.
 \label{eq:supp-stab-c-threshold}
\end{equation}
After relabeling the entries of $\widehat c$ in the same way as the
recovered $\widehat r_h$ in Lemma~\ref{lem:supp-stab-roots}, the
least-squares estimate in \eqref{eq:supp-stab-least-squares} satisfies
\begin{equation}
 \|\widehat c-c\|_2\leq C_c\cdot\tau.
 \label{eq:supp-stab-c-error}
\end{equation}
If additionally $C_c\cdot\tau<c_{\min}/2$, every recovered coefficient remains
nonzero and has the same sign as its corresponding exact coefficient.
\end{lemma}

\begin{proof}
For $g_m(r)=r/(m+r)$ and $r\geq r_{\min}/2$,
\[
 |g_m'(r)|
 =
 \frac{m}{(m+r)^2}
 \leq
 \frac{1}{1+r_{\min}/2}
 =
 L_G.
\]
There are $2H^2$ entries in $G$.  Hence
\begin{equation}
 \begin{aligned}
  \|G(\widehat r)-G(r)\|_2
  &\leq
  \|G(\widehat r)-G(r)\|_F\leq
  \sqrt2\,H L_GC_r\tau
  =
  C_G\cdot\tau.
 \end{aligned}
 \label{eq:supp-stab-G-perturbation}
\end{equation}
Weyl's inequality gives
$\sigma_{\min}(G(\widehat r))\geq\gamma_c/2$.
Therefore, $G(\widehat r)$ has full column rank, so the least-squares
solution is unique. Since $y=G(r)c$, the vector
$G(\widehat r)(\widehat c-c)$ is the orthogonal projection of
\[
 (\widetilde y-y)
 -(G(\widehat r)-G(r))c
\]
onto the column space of $G(\widehat r)$. Its norm is therefore no larger
than the norm of this vector. Hence
\[
 \frac{\gamma_c}{2}\|\widehat c-c\|_2
 \leq
 \|\widetilde y-y\|_2
 +
 \|G(\widehat r)-G(r)\|_2\|c\|_2.
\]
Using $\|\widetilde y-y\|_2\leq2\sqrt{2H}\tau$,
\eqref{eq:supp-stab-G-perturbation}, and $\|c\|_2\leq B_c$ proves
\eqref{eq:supp-stab-c-error}. The final statement follows from
$|\widehat c_h-c_h|\leq\|\widehat c-c\|_2$.
\end{proof}

Pole separation also gives an explicit lower bound on $\gamma_c$.
Taking the first $H$ rows of $G(r)$ and applying the Cauchy determinant
formula yields
\begin{equation}
 \gamma_c
 \geq
 \frac{
  r_{\min}^H\delta_r^{H(H-1)/2}
 }{
  (H+r_{\max})^{H^2}H^{H-1}
 }.
 \label{eq:supp-stab-gamma-c-lower}
\end{equation}
To obtain this bound, consider the $H\times H$ matrix formed by the
first $H$ rows of $G(r)$. Its Cauchy determinant contains the factor
\[
 \prod_{1\le i<j\le H}(j-i),
\]
which is at least one, and the factor
\[
 \prod_{h<g}|r_h-r_g|,
\]
which is at least $\delta_r^{H(H-1)/2}$. Since every entry of this
submatrix is at most one, its spectral norm is at most its Frobenius
norm and hence at most $H$. The determinant identity relating the
singular values then gives \eqref{eq:supp-stab-gamma-c-lower}. Adding the
remaining rows cannot decrease the smallest singular value.

\paragraph{Stability of consistent head labeling.}

\begin{lemma}[Preservation of head labels]
\label{lem:supp-stab-matching}
Suppose that, after relabeling as in
Lemma~\ref{lem:supp-stab-roots}, the recovered values satisfy
\[
 |\widehat s_h(u,q)-s_h(u,q)|
 \leq \epsilon_s
\]
for every $(u,q)\in\mathcal A$ and $h\in[H]$.  If
\begin{equation}
 6\epsilon_s<\Delta_{\rm match},
 \label{eq:supp-stab-match-threshold}
\end{equation}
then selecting the candidate with the smallest absolute residual in every
instance of \eqref{eq:add-u} and \eqref{eq:add-q}
recovers the exact common head labeling.
\end{lemma}

\begin{proof}
For a triple from the same head, the exact residual is zero, so the
recovered residual is at most $3\epsilon_s$. For every incorrect triple,
the exact residual has magnitude at least $\Delta_{\rm match}$, so the
recovered residual is at least
$\Delta_{\rm match}-3\epsilon_s>3\epsilon_s$.  Thus the true triple is the
unique closest candidate. Applying this argument to every anchor label
gives a consistent labeling of all recovered sets.
\end{proof}

\paragraph{End-to-end constants.}
Define
\begin{equation}
 \tau_0
 =
 \min\left\{
  \frac{\gamma_{\rm int}}{4V_H},
  \frac{\delta_r}{2C_r},
  \frac{r_{\min}}{2C_r},
  \frac{\gamma_c}{2C_G},
  \frac{c_{\min}}{2C_c},
  \frac{\Delta_{\rm match}}{6C_s}
 \right\},
 \label{eq:supp-stab-tau0}
\end{equation}
where the last term is omitted when $d=1$, and define
\begin{equation}
 C_{\rm stab}
 =
 \frac{1}{|t|}
 \left(
  \frac{dC_s}{\gamma_U\gamma_Q}
  +
  \frac{\sqrt d\,C_c}{\gamma_U}
 \right).
 \label{eq:supp-stab-Cstab}
\end{equation}

\begin{theorem}[Explicit conditional stability]
\label{thm:supp-stability-explicit}
Fix a model with $H\geq2$ canonical heads, fixed invertible bases
$\mathbf U,\mathbf Q$, and a fixed $t\neq0$.  Suppose all quantities in
\eqref{eq:supp-stab-r-range}--\eqref{eq:supp-stab-basis-margins} have the
stated positive margins. If $0\leq\tau<\tau_0$, the approximate-output
reconstruction
returns estimates and a permutation $\pi$ satisfying
\begin{equation}
 \max_{h\in[H]}
 \left(
  \|\widehat W_h-W_{\pi(h)}\|_F
  +
  \|\widehat v_h-v_{\pi(h)}\|_2
 \right)
 \leq
 C_{\rm stab}\tau.
 \label{eq:supp-stab-final-error}
\end{equation}
It uses $4Hd^2-2H+2d-1$ value queries, each of length at most $2H+1$.
\end{theorem}

\begin{proof}
Lemmas~\ref{lem:supp-stab-interpolation}--\ref{lem:supp-stab-values}
give, within every local decoder,
\[
 \max_h|\widehat s_h-s_h|\leq C_s\tau,
 \qquad
 \|\widehat c-c\|_2\leq C_c\tau.
\]
The last condition in \eqref{eq:supp-stab-tau0} and
Lemma~\ref{lem:supp-stab-matching} place every local output under one
common permutation.

For a matched head, form
\[
 (S_h)_{ij}=s_h(u_i,q_j),
 \qquad
 (\mathbf c_h)_i=c_h(u_i,q_1),
\]
and define $\widehat S_h,\widehat{\mathbf c}_h$ from the decoded values.
Then
\begin{equation}
 \|\widehat S_h-S_h\|_F\leq dC_s\tau,
 \qquad
 \|\widehat{\mathbf c}_h-\mathbf c_h\|_2
 \leq\sqrt d\,C_c\tau.
 \label{eq:supp-stab-global-sample-error}
\end{equation}
The scale $t$ appears in both exact reconstruction identities:
\begin{equation}
 S_h=t\mathbf U W_h\mathbf Q,
 \qquad
 \mathbf c_h=t\mathbf Uv_h.
 \label{eq:supp-stab-scaled-reconstruction}
\end{equation}
Consequently,
\[
 \widehat W_h
 =
 t^{-1}\mathbf U^{-1}\widehat S_h\mathbf Q^{-1},
 \qquad
 \widehat v_h
 =
 t^{-1}\mathbf U^{-1}\widehat{\mathbf c}_h.
\]
Using
$\|\mathbf U^{-1}\|_2=1/\gamma_U$ and
$\|\mathbf Q^{-1}\|_2=1/\gamma_Q$ in
\eqref{eq:supp-stab-global-sample-error} proves
\eqref{eq:supp-stab-final-error}.

There are $2d^2-1$ local decoders, each requiring $2H$ repeated-token queries, together with the $2d-1$ one-token queries in
\eqref{eq:supp-stab-baselines}. Hence the query count is
\[
 2H(2d^2-1)+(2d-1)
 =
 4Hd^2-2H+2d-1.
\]
The largest multiplicity is $2H$, so the maximum sequence length is
$2H+1$.
\end{proof}

The conditioned construction in \eqref{eq:conditioned-bases} gives
$\gamma_U,\gamma_Q\geq1$, but does not uniformly control
$\delta_r$, $c_{\min}$, $\gamma_{\rm int}$, $\gamma_c$, or
$\Delta_{\rm match}$.  These quantities are positive with probability
one for each fixed model, but can be arbitrarily small across models.
Theorem~\ref{thm:supp-stability-explicit} is therefore a local stability
result rather than a uniform finite-precision guarantee.

%% file: supplementary/sections/04_02_two_head_lower_bound.tex
\subsection{No Model-Independent Bound on $C_{\rm stab}$}
\label{supp:stability-lower-bound}

We show that $C_{\rm stab}$ cannot be bounded independently of the target
model under the qualitative identifiability assumptions alone, even for
two heads. The argument is information-theoretic and does not rely on
rational interpolation.

Let $\Pi_H$ denote the set of permutations of $[H]$. It suffices to
consider $d=1$ and $H=2$. In this case, each $W_h$ and $v_h$ is a
scalar, which we denote by $w_h$ and $v_h$, respectively. For a sequence
$X=(x_1,\ldots,x_N)^\top$, define
\begin{equation}
 g_w(X)
 =
 \sum_{i=1}^N
 \frac{\exp(w x_ix_N)}
      {\sum_{j=1}^N\exp(w x_jx_N)}
 x_i.
 \label{eq:supp-lb-g}
\end{equation}
Thus a scalar head $(w,v)$ contributes $v g_w(X)$. For $B>0$, let
\begin{equation}
 \mathcal X_B
 =
 \bigcup_{N\geq1}
 \left\{
  X\in\R^N:\max_{i\in[N]}|x_i|\leq B
 \right\}.
 \label{eq:supp-lb-domain}
\end{equation}
For a target model $M$ and a query domain $\mathcal D$, an
\emph{additive-$\tau$ oracle} is a deterministic response function
$\widetilde F:\mathcal D\to\R$ satisfying
\begin{equation}
 |\widetilde F(X)-F_M(X)|\leq\tau
 \qquad\text{for every }X\in\mathcal D.
 \label{eq:supp-lb-additive-oracle-definition}
\end{equation}
Apart from the error bound in
\eqref{eq:supp-lb-additive-oracle-definition}, $\widetilde F$ is
arbitrary. For two unordered two-head models $M$ and $M'$, define the bottleneck
parameter distance
\begin{equation}
 d_{\rm par}(M,M')
 =
 \min_{\pi\in\Pi_2}
 \max_{h\in[2]}
 \left(
  |w_h-w'_{\pi(h)}|
  +
  |v_h-v'_{\pi(h)}|
 \right).
 \label{eq:supp-lb-distance}
\end{equation}
This is exactly the error metric in
Theorem~\ref{thm:robust-recovery} specialized to $d=1$.

\begin{lemma}[Lipschitz continuity with respect to $w$]
\label{lem:supp-lb-lipschitz}
For every $X\in\mathcal X_B$ and every $w,w'\in\R$,
\begin{equation}
 |g_w(X)-g_{w'}(X)|
 \leq
 B^3|w-w'|.
 \label{eq:supp-lb-lipschitz}
\end{equation}
\end{lemma}

\begin{proof}
Let $a_i(w)$ be the softmax weight in
\eqref{eq:supp-lb-g}, and view the token value $x_i$ as a random variable
under the distribution $a(w)$. Differentiating the softmax expectation
gives
\begin{equation}
 \begin{aligned}
  \frac{d}{dw}g_w(X)
  &=
  x_N
  \left(
   \mathbb E_{a(w)}[x_i^2]
   -
   \mathbb E_{a(w)}[x_i]^2
  \right)\\
  &=
  x_N\operatorname{Var}_{a(w)}(x_i).
 \end{aligned}
 \label{eq:supp-lb-derivative}
\end{equation}
A random variable supported on $[-B,B]$ has variance at most $B^2$,
and $|x_N|\leq B$. Therefore
$|g_w'(X)|\leq B^3$. The mean-value theorem proves
\eqref{eq:supp-lb-lipschitz}.
\end{proof}

\begin{theorem}[Indistinguishable models under approximate outputs]
\label{thm:supp-lb-common-oracle}
Fix $B,\mu,\tau>0$ and distinct centers $a,b\in\R$. For any
$\delta>0$ satisfying
\begin{equation}
 \mu\delta B^3\leq\tau,
 \label{eq:supp-lb-delta-condition}
\end{equation}
define
\begin{align}
 M_a(\delta)
 &=
 \{(a,\mu),(a+\delta,-\mu)\},
 \notag\\
 M_b(\delta)
 &=
 \{(b,\mu),(b+\delta,-\mu)\}.
 \label{eq:supp-lb-models}
\end{align}
Each model satisfies $w_1\neq w_2$ and $v_1,v_2\neq0$. Nevertheless, the
deterministic oracle
\begin{equation}
 \widetilde F(X)=0,
 \qquad X\in\mathcal X_B,
 \label{eq:supp-lb-zero-oracle}
\end{equation}
is an additive-$\tau$ oracle for both models. Their parameter distance
satisfies
\begin{equation}
 d_{\rm par}(M_a(\delta),M_b(\delta))
 \geq
 D,
 \qquad
 D=\min\{|a-b|,2\mu\}>0.
 \label{eq:supp-lb-target-distance}
\end{equation}
Consequently, under the common oracle
\eqref{eq:supp-lb-zero-oracle}, every deterministic learner that queries
only inputs in $\mathcal X_B$ and returns an unordered two-head model
incurs parameter error at least $D/2$ on at least one of the two models.
\end{theorem}

\begin{proof}
The output of the first model is
\[
 F_{M_a(\delta)}(X)
 =
 \mu\bigl(g_a(X)-g_{a+\delta}(X)\bigr).
\]
Lemma~\ref{lem:supp-lb-lipschitz} and
\eqref{eq:supp-lb-delta-condition} imply
\[
 |F_{M_a(\delta)}(X)|
 \leq
 \mu\delta B^3
 \leq\tau
 \qquad
 (X\in\mathcal X_B).
\]
The same bound holds for $M_b(\delta)$, so the all-zero oracle is valid
for both models.

There are only two possible matchings between the heads. The matching
that pairs $\mu$ with $\mu$ and $-\mu$ with $-\mu$ incurs a difference
of $|a-b|$ in the $w$-parameters. The other matching pairs $\mu$ with
$-\mu$ and therefore incurs a difference of $2\mu$ in the
$v$-parameters. This proves \eqref{eq:supp-lb-target-distance}.

A deterministic learner receives the same responses under both models
and therefore returns the same estimate. By the triangle inequality for
$d_{\rm par}$, this estimate cannot be at distance strictly less than
$D/2$ from both models.
\end{proof}

\begin{corollary}[Lower bound for randomized learners]
\label{cor:supp-lb-randomized}
Let $\mathcal L$ be a randomized learner that queries only inputs in
$\mathcal X_B$ and returns an unordered two-head model. Run
$\mathcal L$ with the common oracle
\eqref{eq:supp-lb-zero-oracle}, and let $\widehat M_a$ and
$\widehat M_b$ denote its outputs when the target models are
$M_a(\delta)$ and $M_b(\delta)$, respectively. Then
\begin{equation}
 \max\left\{
  \mathbb E d_{\rm par}(\widehat M_a,M_a(\delta)),
  \mathbb E d_{\rm par}(\widehat M_b,M_b(\delta))
 \right\}
 \geq\frac D2.
 \label{eq:supp-lb-randomized-expectation}
\end{equation}
Moreover,
\begin{equation}
 \begin{aligned}
  \min\Bigl\{&
   \mathbb P\!\left(
    d_{\rm par}(\widehat M_a,M_a(\delta))<D/2
   \right),\\
  &
   \mathbb P\!\left(
    d_{\rm par}(\widehat M_b,M_b(\delta))<D/2
   \right)
  \Bigr\}
  \leq\frac12.
 \end{aligned}
 \label{eq:supp-lb-randomized-success}
\end{equation}
\end{corollary}

\begin{proof}
Use the same internal randomness in the two executions. Because the
oracle returns the same response under both models, the two executions
produce the same sequence of queries and responses and therefore the
same output for every realization of the learner's internal randomness.
Denote this output by
$\widehat M$. The triangle inequality and
\eqref{eq:supp-lb-target-distance} give
\[
 d_{\rm par}(\widehat M,M_a(\delta))
 +
 d_{\rm par}(\widehat M,M_b(\delta))
 \geq D
\]
for every such realization. Taking expectations proves
\eqref{eq:supp-lb-randomized-expectation}.

No estimate can be at distance strictly less than $D/2$ from both
models. Therefore, for every seed, at most one of the two success events
in \eqref{eq:supp-lb-randomized-success} can occur. Taking probabilities
proves the second claim.
\end{proof}

\begin{corollary}[No model-independent stability constant]
\label{cor:supp-lb-no-uniform-C}
Fix $B>0$. There do not exist constants $0<C<\infty$ and
$\tau_\star>0$ and a deterministic learner returning an unordered
two-head model such that, for every two-head model with distinct
$w$-parameters and nonzero $v$-parameters, every
$0<\tau<\tau_\star$, and every additive-$\tau$ oracle on
$\mathcal X_B$, the learner guarantees parameter error at most
$C\cdot\tau$. No randomized learner can provide the same guarantee with
success probability strictly greater than $1/2$ for every such model,
noise level, and oracle.
\end{corollary}

\begin{proof}
Fix distinct $a,b$ and $\mu>0$, and let
$D=\min\{|a-b|,2\mu\}$. If such $C$, $\tau_\star$, and a deterministic
learner existed, choose $\delta>0$ sufficiently small that
\[
 \tau=\mu\delta B^3
 <
 \min\left\{\tau_\star,\frac{D}{2C}\right\}.
\]
Theorem~\ref{thm:supp-lb-common-oracle} gives an additive-$\tau$ oracle
under which at least one of the two models has parameter error at least
\[
 D/2>C\cdot\tau,
\]
contradicting the assumed guarantee.

For a randomized learner, $C\cdot\tau<D/2$ implies that the event of
parameter error at most $C\cdot\tau$ is contained in the event of error
strictly less than $D/2$.
Corollary~\ref{cor:supp-lb-randomized} shows that the latter event has
probability at most $1/2$ for at least one of the two models.
\end{proof}

The preceding result assumes a uniform bound on token magnitudes.  For a
deterministic learner, the execution obtained by returning zero to every
query uses only finitely many token coordinates and therefore has a finite
bound.  This yields the finite-query statement below.

\begin{corollary}[Finite-query impossibility]
\label{cor:supp-lb-finite-query}
Fix $\tau>0$. Let $\mathcal L$ be a deterministic adaptive learner that
makes at most $Q<\infty$ value queries for every sequence of oracle
responses and returns an unordered two-head model. Then there exist two
canonical two-head models at parameter distance at least one and valid
additive-$\tau$ oracles under which $\mathcal L$ receives the same
response at every step. Consequently, $\mathcal L$ incurs parameter
error at least $1/2$ on at least one of the two models.
\end{corollary}

\begin{proof}
Run $\mathcal L$ while returning zero in response to every query. This
determines a finite set $\mathcal S$ of at most $Q$ queried sequences.
Define
\begin{equation}
 B_{\mathcal L}
 =
 \max\left\{
  1,\ 
  \max_{X\in\mathcal S}\max_i|x_i|
 \right\},
 \label{eq:supp-lb-learner-B}
\end{equation}
where the maximum over an empty set is zero. Take $a=0$, $b=1$,
$\mu=1$, and choose
\[
 0<\delta\leq\frac{\tau}{B_{\mathcal L}^3}.
\]
For each $M\in\{M_0(\delta),M_1(\delta)\}$, define
\begin{equation}
 \widetilde F_M(X)
 =
 \begin{cases}
  0,&X\in\mathcal S,\\
  F_M(X),&X\notin\mathcal S.
 \end{cases}
 \label{eq:supp-lb-adaptive-oracle}
\end{equation}
Theorem~\ref{thm:supp-lb-common-oracle} shows that
$|\widetilde F_M(X)-F_M(X)|\leq\tau$ for every $X\in\mathcal S$.
Outside $\mathcal S$, the error is zero. Hence both
$\widetilde F_{M_0(\delta)}$ and $\widetilde F_{M_1(\delta)}$ are valid
additive-$\tau$ oracles.

Each oracle returns zero on every sequence queried in the initial
execution. Therefore, by induction over the adaptive queries,
$\mathcal L$ makes the same queries and receives the same responses
under both models. It consequently returns the same estimate, whereas
\[
 d_{\rm par}(M_0(\delta),M_1(\delta))\geq1.
\]
The triangle inequality implies that this estimate has parameter error
at least $1/2$ on at least one of the two models.
\end{proof}

This construction does not show that every quantitative condition in
Appendix~\ref{supp:stability} is individually necessary. It only shows
that $W_1\neq W_2$ and $v_1,v_2\neq0$ do not by themselves imply a
model-independent bound on $C_{\rm stab}$.

%% file: supplementary/sections/05_binary_membership_queries.tex
\section{Conditional Recovery from Binary Membership Queries}
\label{supp:membership}

We consider the binary membership oracle
\begin{equation}
 \operatorname{MQ}_M(X)
 =
 \mathbf 1\{F_M(X)>0\}.
 \label{eq:supp-mq-oracle}
\end{equation}
The learner chooses an arbitrary finite real sequence $X$ but observes only
the bit in \eqref{eq:supp-mq-oracle}. We assume that the learner knows $d$
and the number $H\geq2$ of canonical heads. As in the main paper, the
canonical heads have pairwise distinct $W_h$ and nonzero $v_h$.

The positive result assumes that the learner knows constants
$\tau_{\rm safe}>0$, $C_{\rm ub}<\infty$, and
$\eta_{\rm dir}>0$. They satisfy
\[
 \tau_{\rm safe}\le\tau_0,
 \qquad
 C_{\rm ub}\ge C_{\rm stab},
\]
where $\tau_0$ and $C_{\rm stab}$ are the threshold and stability
constant from Theorem~\ref{thm:supp-stability-explicit}. The quantity
$\eta_{\rm dir}$ specifies the required accuracy of the direction
estimate defined below. We do not claim that these quantities can be
obtained from binary membership queries alone, and the cost of obtaining
or verifying them is not included in the query count of
Theorem~\ref{thm:supp-mq-conditional}.

\subsection{Normalization and Scale Ambiguity}

Let
\begin{equation}
 S=\sum_{h=1}^H v_h.
 \label{eq:supp-mq-S}
\end{equation}

\begin{lemma}[Scaling is unidentifiable]
\label{lem:supp-mq-scale-ambiguity}
For every $\alpha>0$, replacing all value vectors $v_h$ by
$\alpha v_h$ leaves the membership oracle unchanged.
\end{lemma}

\begin{proof}
The replacement multiplies $F_M(X)$ by $\alpha$ for every $X$.
Multiplication by a positive number preserves whether the output is
strictly positive.
\end{proof}

Suppose first that $S\neq0$.  Define
\begin{equation}
 a_S=\|S\|_2,
 \qquad
 \omega=\frac{S}{a_S},
 \qquad
 \bar v_h=\frac{v_h}{a_S},
 \qquad
 \bar F_M=\frac{F_M}{a_S}.
 \label{eq:supp-mq-normalization}
\end{equation}
Then $\|\omega\|_2=1$, $\sum_h\bar v_h=\omega$, and $\bar F_M$ induces
the same membership oracle as $F_M$. Therefore, this section recovers
\begin{equation}
 \begin{aligned}
  \left\{(W_h,\bar v_h):h\in[H]\right\}
  =
  \Biggl\{&
   \left(W_h,\frac{v_h}{\|\sum_gv_g\|_2}\right): h\in[H]
  \Biggr\},
 \end{aligned}
 \label{eq:supp-mq-normalized-model}
\end{equation}
up to a permutation in $\Pi_H$.

For a one-token input, every head assigns attention weight one to the
unique token, and hence
\begin{equation}
 \operatorname{MQ}_M([x^\top])
 =
 \mathbf 1\{x^\top\omega>0\}.
 \label{eq:supp-mq-one-token}
\end{equation}
Thus length-one membership queries determine the homogeneous halfspace
with unit normal $\omega$.

\subsection{Recovering Rational-Function Values from Membership Queries}

Fix a known scale $t>0$. The main paper uses $t=1$; retaining $t$ here
makes the connection to the repeated-token local decoder explicit. For
$(u,q)\in\R^d\times\R^d$, define
\begin{equation}
 \begin{aligned}
  s_h(u,q)&=t u^\top W_hq,\qquad
  c_h(u,q)=t u^\top\bar v_h,\qquad
  r_h(u,q)=\exp(s_h(u,q)),\\
  R_{u,q}^{(t)}(m)
  &=\bar F_M(X_m^{(t)}(u,q))-\bar F_M([q^\top])
    =\sum_{h=1}^H
     \frac{c_h(u,q)r_h(u,q)}{m+r_h(u,q)}.
 \end{aligned}
 \label{eq:supp-mq-local-sample}
\end{equation}
The fixed scale $t$ appears in both $s_h$ and $c_h$.

Use the same collection of pairs as the value-query algorithm:
\begin{equation}
 \begin{split}
  \mathcal A_{\rm mq}
  ={}&
  \{(u_i,q_j):i,j\in[d]\}\\
  &{}\cup
  \{(u_1+u_i,q_1):i=2,\ldots,d\}\\
  &{}\cup
  \{(u_i,q_1+q_j):i\in[d],\ j=2,\ldots,d\}.
 \end{split}
 \label{eq:supp-mq-pairs}
\end{equation}
Its cardinality is
\begin{equation}
 K=|\mathcal A_{\rm mq}|=2d^2-1.
 \label{eq:supp-mq-K}
\end{equation}

For a positive scaling variable $a$, define
\begin{equation}
 X_{m,a}^{(t)}(u,q)
 =
 \left[
  (q/a+a t u)^\top;
  (q/a)^\top;
  \ldots;
  (q/a)^\top
 \right],
 \label{eq:supp-mq-scaled-input}
\end{equation}
where $q/a$ occurs $m$ times, including as the final token. The variable
$a$ is distinct from the fixed repeated-token scale $t$.

\begin{lemma}[Scale-preserving identity]
\label{lem:supp-mq-scale-identity}
Let $B(q)=q^\top\omega$. For every $a>0$,
\begin{equation}
 \bar F_M(X_{m,a}^{(t)}(u,q))
 =
 \frac{B(q)}{a}
 +
 aR_{u,q}^{(t)}(m).
 \label{eq:supp-mq-scale-identity}
\end{equation}
Consequently, the membership query at
$X_{m,a}^{(t)}(u,q)$ determines the sign of
\begin{equation}
 \Phi_{u,q,m}(a)
 =
 B(q)+a^2R_{u,q}^{(t)}(m).
 \label{eq:supp-mq-Phi}
\end{equation}
\end{lemma}

\begin{proof}
The difference between the first token and each repeated token is
$a t u$, while the final query token is $q/a$. The difference between
their attention logits for head $h$ is
\[
 (a t u)^\top W_h(q/a)
 =
 t u^\top W_hq
 =
 s_h(u,q),
\]
which is independent of $a$. The attention weight of the first token is
$r_h/(m+r_h)$. Its value difference from a repeated token is
$a t u^\top\bar v_h=a c_h$, while the repeated-token value is
$q^\top\bar v_h/a$.  Summing the head outputs proves
\eqref{eq:supp-mq-scale-identity}.  Since $a_S>0$ and $a>0$, the sign of
the original output equals the sign of
\eqref{eq:supp-mq-Phi}.
\end{proof}

\paragraph{Bounds on $B(q)$ and $R_{u,q}^{(t)}(m)$.}
Assume that the learner knows finite constants
\begin{equation}
 0<B_-\leq B_+,
 \qquad
 0<D_-\leq D_+,
 \label{eq:supp-mq-bracket-constants}
\end{equation}
such that, for every $(u,q)\in\mathcal A_{\rm mq}$ and
$m\in[2H]$,
\begin{equation}
 B_-\leq B(q)\leq B_+,
 \qquad
 -D_+\leq R_{u,q}^{(t)}(m)\leq-D_-.
 \label{eq:supp-mq-sign-bracket}
\end{equation}
Define
\begin{equation}
 a_-=\sqrt{\frac{B_-}{2D_+}},
 \qquad
 a_+=\sqrt{\frac{2B_+}{D_-}}.
 \label{eq:supp-mq-a-bracket}
\end{equation}

The bounds above imply that
$\Phi_{u,q,m}(a_-)>0$ and $\Phi_{u,q,m}(a_+)<0$.
Starting from the interval $[a_-,a_+]$, the learner queries its midpoint.
If the returned bit is one, the lower endpoint is replaced by the
midpoint; otherwise, the upper endpoint is replaced by the midpoint.
Repeating this step halves an interval containing the unique positive
root of $\Phi_{u,q,m}$.

\begin{lemma}[Recovery of the root from membership queries]
\label{lem:supp-mq-bisection}
Under \eqref{eq:supp-mq-sign-bracket}, the function
$\Phi_{u,q,m}$ is strictly decreasing on $(0,\infty)$ and has the unique
positive root
\begin{equation}
 a_\star
 =
 \sqrt{\frac{B(q)}{-R_{u,q}^{(t)}(m)}}.
 \label{eq:supp-mq-root}
\end{equation}
The interval-halving procedure above returns $\widehat a$ satisfying
\begin{equation}
 |\widehat a-a_\star|\leq\eta_a
 \label{eq:supp-mq-a-error}
\end{equation}
using at most
\begin{equation}
 2+
 \left\lceil
  \log_2^+\frac{a_+-a_-}{\eta_a}
 \right\rceil
 \label{eq:supp-mq-bisection-count}
\end{equation}
membership queries, where
$\log_2^+(x)=\max\{0,\log_2x\}$.
\end{lemma}

\begin{proof}
By \eqref{eq:supp-mq-sign-bracket},
\[
 \Phi_{u,q,m}'(a)
 =
 2aR_{u,q}^{(t)}(m)
 <
 0
 \qquad(a>0),
\]
so $\Phi_{u,q,m}$ is strictly decreasing. Moreover,
\[
 \Phi_{u,q,m}(a_-)
 \geq
 B_--a_-^2D_+
 =
 \frac{B_-}{2}
 >
 0
\]
and
\[
 \Phi_{u,q,m}(a_+)
 \leq
 B_+-a_+^2D_-
 =
 -B_+
 <
 0.
\]
Therefore, $\Phi_{u,q,m}$ has exactly one positive root in
$[a_-,a_+]$, and solving
$B(q)+a^2R_{u,q}^{(t)}(m)=0$ gives
\eqref{eq:supp-mq-root}.

Each midpoint query preserves an interval containing $a_\star$ and
reduces its length by a factor of two. Let
\[
 N_{\rm half}
 =
 \left\lceil
  \log_2^+\frac{a_+-a_-}{\eta_a}
 \right\rceil
\]
be the number of interval-halving steps. After these
$N_{\rm half}$ steps, the remaining interval has length at most
$\eta_a$. Returning its midpoint therefore gives
\eqref{eq:supp-mq-a-error}. The two additional queries evaluate the
initial endpoints. If a queried midpoint is exactly the root, the
membership oracle returns zero. Treating that midpoint as the
nonpositive endpoint still preserves an interval containing
$a_\star$.
\end{proof}

Suppose a unit-vector estimate $\widehat\omega$ satisfies
\begin{equation}
 \|\widehat\omega-\omega\|_2\leq\eta_\omega.
 \label{eq:supp-mq-omega-error}
\end{equation}
Let $Q_{\max}$ be a known bound such that
\begin{equation}
 \max_{(u,q)\in\mathcal A_{\rm mq}}\|q\|_2\leq Q_{\max},
 \label{eq:supp-mq-Qmax}
\end{equation}
and define
\begin{equation}
 \widehat B(q)=q^\top\widehat\omega,
 \qquad
 \widehat R_{u,q}^{(t)}(m)
 =
 -\frac{\widehat B(q)}{\widehat a^2}.
 \label{eq:supp-mq-R-estimator}
\end{equation}

\begin{lemma}[Error in the recovered rational-function values]
\label{lem:supp-mq-value-error}
If $\eta_a\leq a_-/2$, then
\begin{equation}
 \left|
  \widehat R_{u,q}^{(t)}(m)-R_{u,q}^{(t)}(m)
 \right|
 \leq
 \frac{4Q_{\max}\eta_\omega}{a_-^2}
 +
 \frac{16B_+\eta_a}{a_-^3}.
 \label{eq:supp-mq-value-error}
\end{equation}
In particular, the error is at most $\eta_R$ whenever
\begin{equation}
 \eta_\omega
 \leq
 \frac{\eta_Ra_-^2}{8Q_{\max}},
 \qquad
 \eta_a
 \leq
 \min\left\{
  \frac{a_-}{2},
  \frac{\eta_Ra_-^3}{32B_+}
 \right\}.
 \label{eq:supp-mq-value-tolerances}
\end{equation}
\end{lemma}

\begin{proof}
Equation~\eqref{eq:supp-mq-omega-error} gives
$|\widehat B(q)-B(q)|\leq Q_{\max}\eta_\omega$. The true root is at least
$a_-$, and
$|\widehat a-a_\star|\leq a_-/2$, so both $\widehat a$ and
$a_\star$ are at least $a_-/2$.  Therefore
\[
 \frac{|\widehat B(q)-B(q)|}{\widehat a^2}
 \leq
 \frac{4Q_{\max}\eta_\omega}{a_-^2}.
\]
The derivative magnitude of $a\mapsto a^{-2}$ on
$[a_-/2,\infty)$ is at most $16/a_-^3$.  Since
$R=-B/a_\star^2$ and $B\leq B_+$, the second source of error is at most
$16B_+\eta_a/a_-^3$.  This proves
\eqref{eq:supp-mq-value-error}; the two bounds in
\eqref{eq:supp-mq-value-tolerances} assign at most $\eta_R/2$ to
each term.
\end{proof}

\subsection{Estimating the Direction of $S$}

Let $x$ be uniform on the unit sphere in $\R^d$.  Rotational invariance
gives
\begin{equation}
 \mathbb E[\operatorname{sign}(x^\top\omega)x]
 =\rho_d\omega,
 \qquad
 \rho_d
 =\frac{\Gamma(d/2)}{\sqrt\pi\,\Gamma((d+1)/2)}
 =\Theta(d^{-1/2}).
 \label{eq:supp-mq-sphere-identity}
\end{equation}
Indeed, the expectation must be parallel to $\omega$, and its inner
product with $\omega$ equals
$\mathbb E|x^\top\omega|=\rho_d$.

\begin{lemma}[Direction-estimation error]
\label{lem:supp-mq-direction}
For $\eta_\omega\in(0,1/2)$ and $\delta\in(0,1)$, length-one membership
queries return a unit vector $\widehat\omega$ satisfying
\eqref{eq:supp-mq-omega-error} with probability at least $1-\delta$ using
\begin{equation}
 Q_{\rm dir}
 =
 O\left(
  \frac{d^2}{\eta_\omega^2}
  \log\frac{2d}{\delta}
 \right)
 \label{eq:supp-mq-direction-count}
\end{equation}
queries.
\end{lemma}

\begin{proof}
Draw independent sphere points $x_1,\ldots,x_n$ and set
\[
 Z_k=
 \left(2\operatorname{MQ}_M([x_k^\top])-1\right)x_k.
\]
The probability that $x_k^\top\omega=0$ is zero, so
$Z_k=\operatorname{sign}(x_k^\top\omega)x_k$ with probability one. Estimate
$\omega$ by normalizing
\[
 \widetilde\omega
 =
 \frac{1}{n\rho_d}\sum_{k=1}^n Z_k.
\]
Each coordinate of $Z_k/\rho_d$ has magnitude at most
$1/\rho_d=O(\sqrt d)$.
Coordinatewise Hoeffding bounds with coordinate accuracy
$\eta_\omega/(4\sqrt d)$, followed by a union bound over the $d$
coordinates,
give
\[
 \|\widetilde\omega-\omega\|_2\leq\eta_\omega/4
\]
with probability at least $1-\delta$ for the sample size in
\eqref{eq:supp-mq-direction-count}.  If a vector lies within distance
less than $1/2$ of a unit vector, normalization increases its distance by
at most a universal constant factor.  Thus the normalized estimate
satisfies \eqref{eq:supp-mq-omega-error}.
\end{proof}

\subsection{Conditional Recovery Guarantee}

Theorem~\ref{thm:supp-stability-explicit} applies to the normalized model
$\{(W_h,\bar v_h):h\in[H]\}$, the pairs in
\eqref{eq:supp-mq-pairs}, and the fixed scale $t$. It requires
$\mathbf U$ and $\mathbf Q$ to be invertible and the stated quantitative
conditions to hold over all these pairs, but it does not otherwise require
the distribution used in \eqref{eq:conditioned-bases}. For every
fixed choice satisfying these conditions, the theorem provides a threshold
$\tau_0>0$ and a stability constant $C_{\rm stab}<\infty$.

In the binary membership-query procedure, the basis directions are
constructed from an estimate of $\omega$. Their stability must therefore
be controlled not only at the exact direction $\omega$, but also at every
sufficiently accurate estimate. We assume that the learner knows constants
\[
 \tau_{\rm safe}>0,
 \qquad
 C_{\rm ub}<\infty,
 \qquad
 \eta_{\rm dir}\in(0,1/2)
\]
with the following property. For every unit vector $\omega'$ satisfying
\[
 \|\omega'-\omega\|_2\leq\eta_{\rm dir},
\]
the construction specified below produces invertible matrices
$\mathbf U$ and $\mathbf Q$, satisfies
\eqref{eq:supp-mq-sign-bracket} with the same constants
$B_-,B_+,D_-,D_+$, and yields constants $\tau_0$ and
$C_{\rm stab}$ satisfying
\begin{equation}
 \tau_{\rm safe}\leq\tau_0,
 \qquad
 C_{\rm stab}\leq C_{\rm ub}.
 \label{eq:supp-mq-stability-bounds}
\end{equation}
Thus, the assumption requires these bounds throughout a neighborhood of \(\omega\).

\begin{theorem}[Conditional binary membership recovery]
\label{thm:supp-mq-conditional}
Let $H\geq2$. Suppose that $S\neq0$, that
$W_h\neq W_g$ whenever $h\neq g$, and that $v_h\neq0$ for every $h$.
Suppose also that the known bounds in
\eqref{eq:supp-mq-bracket-constants}--\eqref{eq:supp-mq-sign-bracket}
and the neighborhood condition above hold.

For $\epsilon,\delta\in(0,1)$, define
\begin{equation}
 \eta_R
 =\min\left\{\tau_{\rm safe},\frac{2\epsilon}{C_{\rm ub}}\right\},
 \qquad
 \eta_\omega^\star
 =\min\left\{\eta_{\rm dir},
 \frac{\eta_Ra_-^2}{8Q_{\max}}\right\},
 \qquad
 \eta_a^\star
 =\min\left\{\frac{a_-}{2},
 \frac{\eta_Ra_-^3}{32B_+}\right\}.
 \label{eq:supp-mq-recovery-tolerances}
\end{equation}
Then a randomized membership-query algorithm returns estimates and a
permutation $\pi\in\Pi_H$ satisfying
\begin{equation}
 \max_{h\in[H]}
 \left(
  \|\widehat W_h-W_{\pi(h)}\|_F
  +
  \|\widehat{\bar v}_h-\bar v_{\pi(h)}\|_2
 \right)
 \leq\epsilon
 \label{eq:supp-mq-final-error}
\end{equation}
with probability at least $1-\delta$.

The number of membership queries is at most
\begin{equation}
 O\left(
  \frac{d^2}{(\eta_\omega^\star)^2}
  \log\frac{2d}{\delta}
 \right)
 +
 2H(2d^2-1)
 \left(
  2+
  \left\lceil
   \log_2^+\frac{a_+-a_-}{\eta_a^\star}
  \right\rceil
 \right).
 \label{eq:supp-mq-total-count}
\end{equation}
Every queried sequence has length at most $2H+1$. This count includes
the length-one queries used to estimate $\omega$ and all queries used to
halve the root-containing intervals. It does not include the cost of
obtaining or verifying the quantitative bounds assumed above.
\end{theorem}

\begin{proof}
The algorithm first invokes Lemma~\ref{lem:supp-mq-direction} with
accuracy $\eta_\omega^\star$ and failure probability $\delta$. Consider
the event
\[
 \|\widehat\omega-\omega\|_2
 \leq
 \eta_\omega^\star.
\]
On this event, the basis directions constructed from
$\widehat\omega$ satisfy the neighborhood condition above. The algorithm
forms the summed directions in \eqref{eq:supp-mq-pairs} and, for each of
the $2H(2d^2-1)$ required rational-function values, applies
Lemma~\ref{lem:supp-mq-bisection} with accuracy $\eta_a^\star$. The
resulting estimate is given by \eqref{eq:supp-mq-R-estimator}.

Lemma~\ref{lem:supp-mq-value-error} then gives
\[
 \left|
  \widehat R_{u,q}^{(t)}(m)-R_{u,q}^{(t)}(m)
 \right|
 \leq
 \eta_R
\]
for every recovered value. All interval-halving steps are deterministic
conditional on the binary membership oracle, so they introduce no
additional failure probability.

The approximate-output reconstruction is applied with
\[
 \tau=\frac{\eta_R}{2}.
\]
This is valid because its proof requires
$|\widetilde y_m-y_m|\leq2\tau$. Moreover,
\[
 \tau
 \leq
 \frac{\tau_{\rm safe}}{2}
 <
 \tau_0,
 \qquad
 C_{\rm stab}\tau
 \leq
 \frac{C_{\rm ub}\eta_R}{2}
 \leq
 \epsilon.
\]
Theorem~\ref{thm:supp-stability-explicit} therefore gives
\eqref{eq:supp-mq-final-error}.

Lemma~\ref{lem:supp-mq-direction} gives the first term in
\eqref{eq:supp-mq-total-count}. Applying
Lemma~\ref{lem:supp-mq-bisection} to all
$2H(2d^2-1)$ recovered values gives the second term. A query with
multiplicity $m$ contains $m+1\leq2H+1$ tokens, which proves the maximum
sequence length.
\end{proof}

\subsection{Sufficient Conditions for the Sign Bounds}

We next give a concrete construction that ensures
\eqref{eq:supp-mq-sign-bracket} without imposing a sign condition on
each $\omega^\top\bar v_h$. Assume known bounds
\begin{equation}
 \|\bar v_h\|_2\leq V_{\max},
 \qquad
 \|W_h\|_F\leq W_{\max}
 \qquad(h\in[H])
 \label{eq:supp-mq-norm-bounds}
\end{equation}
hold. Fix known $\eta_{\rm dir}\in(0,1/2)$, $\kappa>0$, and $\zeta>0$.
Given a unit estimate $\widehat\omega$, extend
$b_1=\widehat\omega$ to an orthonormal basis
$b_1,\ldots,b_d$, and define
\begin{equation}
 q_1^0=\widehat\omega,\qquad
 q_j^0=\widehat\omega+\kappa b_j\ (j=2,\ldots,d),\qquad
 z_1^0=-\widehat\omega,\qquad
 z_i^0=-\widehat\omega+\kappa b_i\ (i=2,\ldots,d).
 \label{eq:supp-mq-centers}
\end{equation}
Let $\mathbf Q^0=[q_1^0,\ldots,q_d^0]$, and let $\mathbf Z^0$ have rows
$(z_1^0)^\top,\ldots,(z_d^0)^\top$.  Both matrices are invertible.  Put
\begin{equation}
 \gamma_0
 =
 \min\{
  \sigma_{\min}(\mathbf Q^0),
  \sigma_{\min}(\mathbf Z^0)
 \}
 >
 0
 \label{eq:supp-mq-gamma0}
\end{equation}
and assume
\begin{equation}
 c_0
 =
 1-\frac{\eta_{\rm dir}^2}{2}
 -\kappa\eta_{\rm dir}
 -\zeta
 >
 0,
 \qquad
 \sqrt d\,\zeta\leq\frac{\gamma_0}{2}.
 \label{eq:supp-mq-center-conditions}
\end{equation}

Draw independent perturbations $\xi_j^q,\xi_i^z$ from distributions that
have densities on the Euclidean ball of radius $\zeta$, and set
\begin{equation}
 q_j=q_j^0+\xi_j^q,
 \qquad
 z_i=z_i^0+\xi_i^z,
 \qquad
 u_i=\varepsilon z_i
 \label{eq:supp-mq-basis-directions}
\end{equation}
for a known $\varepsilon>0$. Only the basis vectors are perturbed. Every
summed direction is then formed exactly as
$u_1+u_i=\varepsilon(z_1+z_i)$ or $q_1+q_j$.

On the event
$\|\widehat\omega-\omega\|_2\leq\eta_{\rm dir}$, every unscaled left
vector $z$ and right vector $q$ appearing in \eqref{eq:supp-mq-pairs}
satisfies
\begin{equation}
 z^\top\omega\le -c_0,\qquad
 q^\top\omega\ge c_0,\qquad
 \|z\|_2\le Z_{\max},\qquad
 \|q\|_2\le Q_{\max}.
 \label{eq:supp-mq-zq-bounds}
\end{equation}
Set
\begin{equation}
 Z_{\max}=Q_{\max}=2(1+\kappa+\zeta).
 \label{eq:supp-mq-zq-max}
\end{equation}
These bounds cover all individual and summed directions. To verify the sign bounds,
note that
\[
 \widehat\omega^\top\omega
 =1-\frac{\|\widehat\omega-\omega\|_2^2}{2},
 \qquad
 |b_j^\top\omega|
 =|b_j^\top(\omega-\widehat\omega)|
 \le \eta_{\rm dir}
 \quad(j\ge 2).
\]
The perturbations change an inner product with the unit vector $\omega$ by at
most $\zeta$, and exact sums only strengthen the corresponding one-sided
bound.

Define
\begin{equation}
 L=Z_{\max}W_{\max}Q_{\max},
 \qquad
 \alpha=t\varepsilon.
 \label{eq:supp-mq-L-alpha}
\end{equation}
For $m\geq1$, let
\[
 \vartheta_m(x)=\frac{e^x}{m+e^x}.
\]
Since
$\vartheta_m'(x)=me^x/(m+e^x)^2\leq1/4$,
\begin{equation}
 \left|
  \vartheta_m(x)-\frac{1}{m+1}
 \right|
 \leq\frac{|x|}{4}.
 \label{eq:supp-mq-theta-bound}
\end{equation}

\begin{lemma}[Sign bounds from small left directions]
\label{lem:supp-mq-small-left-bounds}
Suppose \eqref{eq:supp-mq-norm-bounds}--
\eqref{eq:supp-mq-L-alpha} hold.  If $L>0$, choose $\varepsilon>0$ so
that
\begin{equation}
 \alpha=t\varepsilon
 \leq
 \frac{2c_0}
 {(2H+1)H Z_{\max}V_{\max}L}.
 \label{eq:supp-mq-epsilon-condition}
\end{equation}
If $L=0$, any $\varepsilon>0$ is allowed. On the event
$\|\widehat\omega-\omega\|_2\leq\eta_{\rm dir}$, the resulting pairs
satisfy \eqref{eq:supp-mq-sign-bracket} with
\begin{equation}
 B_-=c_0,\qquad
 B_+=Q_{\max},\qquad
 D_-=\frac{\alpha c_0}{2(2H+1)},\qquad
 D_+=\alpha HZ_{\max}V_{\max}.
 \label{eq:supp-mq-explicit-bracket}
\end{equation}
The two direction matrices $\mathbf U$ and $\mathbf Q$ are
invertible.
\end{lemma}

\begin{proof}
Equation~\eqref{eq:supp-mq-zq-bounds} and $\|\omega\|_2=1$ give
\[
 c_0\leq B(q)=q^\top\omega\leq Q_{\max}.
\]
For $(u,q)=(\varepsilon z,q)\in\mathcal A_{\rm mq}$, put
\[
 \ell_h
 =
 s_h(\varepsilon z,q)
 =
 \alpha z^\top W_hq.
\]
Using $\sum_h\bar v_h=\omega$, the rational-function value decomposes as
\begin{equation}
 R_{\varepsilon z,q}^{(t)}(m)
 =
 \frac{\alpha z^\top\omega}{m+1}
 +
 E_m,
 \label{eq:supp-mq-R-decomposition}
\end{equation}
where
\begin{equation}
 |E_m|
 \le
 \sum_{h=1}^H
 \alpha|z^\top\bar v_h|
 \left|
  \vartheta_m(\ell_h)-\frac{1}{m+1}
 \right|
 \le
 \frac{\alpha^2HZ_{\max}V_{\max}L}{4}.
 \label{eq:supp-mq-E-bound}
\end{equation}
Here $|\ell_h|\leq\alpha L$, and
\eqref{eq:supp-mq-theta-bound} was used in the second line.

Because $m\leq2H$, condition
\eqref{eq:supp-mq-epsilon-condition} makes
\eqref{eq:supp-mq-E-bound} at most half the smallest magnitude of the
negative leading term in
\eqref{eq:supp-mq-R-decomposition}.  Therefore
\[
 R_{\varepsilon z,q}^{(t)}(m)
 \leq
 -\frac{\alpha c_0}{2H+1}
 +
 \frac{\alpha c_0}{2(2H+1)}
 =
 -\frac{\alpha c_0}{2(2H+1)}.
\]
Also, $0<\vartheta_m<1$ gives
\[
 |R_{\varepsilon z,q}^{(t)}(m)|
 \leq
 \alpha H Z_{\max}V_{\max}.
\]
These inequalities prove \eqref{eq:supp-mq-explicit-bracket}.

Let $\mathbf Z$ have rows $z_i^\top$, and let $\mathbf Q$ have columns
$q_j$.  The perturbation matrices have Frobenius norm at most
$\sqrt d\,\zeta$, so Weyl's inequality and
\eqref{eq:supp-mq-center-conditions} give
\[
 \sigma_{\min}(\mathbf Q)\geq\frac{\gamma_0}{2},
 \qquad
 \sigma_{\min}(\mathbf U)
 =
 \varepsilon\sigma_{\min}(\mathbf Z)
 \geq
 \frac{\varepsilon\gamma_0}{2}.
\]
Thus both $\mathbf U$ and $\mathbf Q$ are invertible.
\end{proof}

The perturbations in \eqref{eq:supp-mq-basis-directions} have densities
on open sets. For pairwise distinct $W_h$ and nonzero $v_h$, equalities
among distinct $s_h(u,q)$, zero $c_h(u,q)$, and incorrect additive
equalities are zero sets of nonzero polynomials in the basis directions.
They therefore occur with probability zero. Forming the summed directions
exactly preserves the true additive identities.

The norm bounds and Lemma~\ref{lem:supp-mq-small-left-bounds} provide the
required sign bounds, but they do not provide useful quantitative lower bounds on
$c_{\min}$, $\delta_r$, $\gamma_{\rm int}$, $\gamma_c$, or
$\Delta_{\rm match}$. The known bounds in
\eqref{eq:supp-mq-stability-bounds} remain separate
instance-dependent assumptions. In particular, making $\varepsilon$
smaller can worsen these constants because it also shrinks the
$s_h(u,q)$ and $c_h(u,q)$ values.

\begin{corollary}[Recovery under known norm bounds]
\label{cor:supp-mq-norm-bounded}
Suppose $S\neq0$, the $W_h$ are pairwise distinct, every $v_h$ is nonzero,
and the known norm and direction conditions
\eqref{eq:supp-mq-norm-bounds}--
\eqref{eq:supp-mq-epsilon-condition} are satisfied.  Assume additionally
that the directions constructed whenever
$\|\widehat\omega-\omega\|_2\leq\eta_{\rm dir}$ have constants satisfying
\eqref{eq:supp-mq-stability-bounds}. Then
Theorem~\ref{thm:supp-mq-conditional} recovers all $W_h$ and all normalized
values
\[
 \bar v_h
 =
 \frac{v_h}{\|\sum_gv_g\|_2}
\]
up to a permutation in $\Pi_H$, with the error, success probability, query
count, and maximum sequence length stated there.
\end{corollary}

\subsection{Impossibility when $S=0$}

When $S=0$, all one-token outputs are zero and $B(q)$ in
\eqref{eq:supp-mq-Phi} vanishes. The following result shows that this is
an information-theoretic obstruction, not only a limitation of the
bisection construction.

\begin{proposition}[Indistinguishable models when $S=0$]
\label{prop:supp-mq-S-zero}
Let $d=1$, $H=2$, and, for real $a<b$, define
\begin{equation}
 M_{a,b}
 =
 \{(a,1),(b,-1)\}.
 \label{eq:supp-mq-S-zero-model}
\end{equation}
Then $S=0$, the two $w$-parameters are distinct, and both $v$-parameters are
nonzero.  For any $a<b$ and $a'<b'$,
\begin{equation}
 \operatorname{MQ}_{M_{a,b}}(X)
 =
 \operatorname{MQ}_{M_{a',b'}}(X)
 \label{eq:supp-mq-identical-oracles}
\end{equation}
for every real sequence $X$ of every finite length.
\end{proposition}

\begin{proof}
For scalar tokens $x_1,\ldots,x_N$ and final query $q=x_N$, define
\begin{equation}
 y_X(w)
 =
 \frac{\sum_{i=1}^N x_i e^{wqx_i}}
      {\sum_{j=1}^N e^{wqx_j}}.
 \label{eq:supp-mq-yw}
\end{equation}
Differentiation gives
\begin{equation}
 \frac{d}{dw}y_X(w)
 =
 q\operatorname{Var}_w(x_i),
 \label{eq:supp-mq-yw-derivative}
\end{equation}
where the variance is computed under the softmax weights at $w$.  If the
tokens are not all equal, the variance is strictly positive.  Hence
$y_X(w)$ is strictly increasing when $q>0$, strictly decreasing when
$q<0$, and constant when $q=0$ or all tokens are equal.

The two-head output is
\[
 F_{M_{a,b}}(X)
 =
 y_X(a)-y_X(b).
\]
For every $a<b$, this output has sign
$-\operatorname{sign}(q)$ when $q\neq0$ and the tokens are not all equal,
and it is zero otherwise.  This description does not depend on the
values of $a$ and $b$, proving
\eqref{eq:supp-mq-identical-oracles}.
\end{proof}

For two unordered scalar two-head models
$M=\{(w_h,v_h):h\in[2]\}$ and
$M'=\{(w'_h,v'_h):h\in[2]\}$, define the distance between their
$w$-parameters by
\begin{equation}
 d_W(M,M')
 =
 \min_{\pi\in\Pi_2}
 \max_{h\in[2]}|w_h-w'_{\pi(h)}|.
 \label{eq:supp-mq-W-distance}
\end{equation}
In particular,
\begin{equation}
 d_W(M_{0,1},M_{0,2})=1.
 \label{eq:supp-mq-W-distance-example}
\end{equation}

\begin{corollary}[Impossibility when $S=0$]
\label{cor:supp-mq-S-zero-impossibility}
Consider any binary membership-query learner that returns an unordered
two-head model. If the learner is deterministic, its error measured by
$d_W$ is at least $1/2$ on at least one of $M_{0,1}$ and $M_{0,2}$.

If the learner is randomized, let $\widehat M_{01}$ and
$\widehat M_{02}$ denote its outputs for these two models. Then
\begin{equation}
 \max\left\{
  \mathbb E d_W(\widehat M_{01},M_{0,1}),
  \mathbb E d_W(\widehat M_{02},M_{0,2})
 \right\}
 \geq\frac12,
 \label{eq:supp-mq-S-zero-expected}
\end{equation}
and
\begin{equation}
 \begin{aligned}
  \min\Bigl\{&
   \mathbb P(d_W(\widehat M_{01},M_{0,1})<1/2),\\
  &
   \mathbb P(d_W(\widehat M_{02},M_{0,2})<1/2)
  \Bigr\}
  \leq\frac12.
 \end{aligned}
 \label{eq:supp-mq-S-zero-success}
\end{equation}
These conclusions hold regardless of the number or adaptivity of the
membership queries.
\end{corollary}

\begin{proof}
Proposition~\ref{prop:supp-mq-S-zero} shows that every query receives the
same answer for the two models. A deterministic learner therefore returns
the same output in both cases. The triangle inequality
and \eqref{eq:supp-mq-W-distance-example} force at least one of its
two errors to be at least $1/2$.

For a randomized learner, couple the two runs with the same internal
randomness. The two runs then produce the same sequence of queries and
answers, and hence the same output $\widehat M$. For every realization of
this randomness,
\[
 d_W(\widehat M,M_{0,1})+d_W(\widehat M,M_{0,2})\geq1.
\]
Taking expectations proves
\eqref{eq:supp-mq-S-zero-expected}. No estimate can have $d_W$-distance
strictly less than $1/2$ from both models, so the two success indicators
sum to at most one for every realization. Taking expectations proves
\eqref{eq:supp-mq-S-zero-success}.
\end{proof}

Therefore the qualitative conditions $W_h\neq W_g$ and $v_h\neq0$ do
not suffice for binary membership recovery.  The positive result above
necessarily uses $S\neq0$, fixes the unavoidable common positive value
scale through \eqref{eq:supp-mq-normalization}, and remains conditional on
known quantitative bounds that ensure the required sign changes and
stability of the reconstruction.

%% file: supplementary/sections/06_one_layer_relu_transformer.tex
\section{One-Layer ReLU Transformer}
\label{supp:transformer}

This section proves Theorem~\ref{thm:full-transformer} using the notation of
Section~\ref{sec:relu-extension}. In particular,
\[
 \begin{aligned}
  \mathrm{TF}(X)
  &=
  w_o^\top\operatorname{ReLU}\!\left(
    \sum_{h=1}^H A_h^\top y_h(X)\right),\\
  y_h(X)
  &=
  X^\top\softmax(XW_hq),
 \end{aligned}
\]
where the final token of \(X\) is \(q\),
\(A_h=[b_{h1},\ldots,b_{hm}]\in\mathbb R^{d\times m}\), and
\(w_o=(w_1,\ldots,w_m)^\top\). For $j\in[m]$ and $h\in[H]$, define
\[
 z_j(X)\coloneqq\sum_{g=1}^H b_{gj}^\top y_g(X),\qquad
 \bar b_j\coloneqq\sum_{g=1}^H b_{gj},\qquad
 v_h\coloneqq A_hw_o=\sum_{k=1}^m w_kb_{hk}.
\]
Thus \(\mathrm{TF}(X)=\sum_{j=1}^m
w_j\operatorname{ReLU}(z_j(X))\).

\paragraph{Recovering the effective attention heads.}
Negating every token leaves every attention weight unchanged. Indeed, if
the final token of \(X\) is \(q\), then the final token of \(-X\) is
\(-q\), and
\[
 (-X)W_h(-q)=XW_hq.
\]
Consequently, \(y_h(-X)=-y_h(X)\) and \(z_j(-X)=-z_j(X)\). The identities
\(\operatorname{ReLU}(a)-\operatorname{ReLU}(-a)=a\) and
\(\operatorname{ReLU}(a)+\operatorname{ReLU}(-a)=|a|\) therefore give
\begin{align}
 \mathrm{TF}(X)-\mathrm{TF}(-X)
 &=\sum_{h=1}^H v_h^\top y_h(X),                         \label{eq:supp-tf-odd}\\
 \mathrm{TF}(X)+\mathrm{TF}(-X)
 &=\sum_{j=1}^m w_j|z_j(X)|.                             \label{eq:supp-tf-even}
\end{align}
The right-hand side of \eqref{eq:supp-tf-odd} is precisely the
scalar-output multi-head attention model with effective parameters
\(\{(W_h,v_h):h\in[H]\}\). Hence two Transformer queries, at \(X\) and
\(-X\), provide one value query to that model. Under the assumptions of
Theorem~\ref{thm:full-transformer}, applying Theorem~\ref{thm:main} with
the fixed choice \(t=1\) recovers all pairs \((W_h,v_h)\) up to permutation
using
\[
 2(4Hd^2-2H+1)
\]
Transformer queries. Every such query has length at most \(2H+1\).

\paragraph{Information from length-one inputs.}
For a length-one input \(X=[x^\top]\), we have
\(y_h([x^\top])=x\) for every head. Define the restriction of the
Transformer to length-one inputs by
\[
 f_1(x)\coloneqq\mathrm{TF}([x^\top])
 =\sum_{j=1}^m w_j\operatorname{ReLU}(\bar b_j^\top x)
\]
and
\[
 g_1(x)
 \coloneqq f_1(x)+f_1(-x)
 =\sum_{j=1}^m w_j|\bar b_j^\top x|.
\]
The theorem assumes that every \(w_j\) and \(\bar b_j\) is nonzero and
that no two vectors \(\bar b_1,\ldots,\bar b_m\) are scalar multiples of
one another. Let \(\xi_j\in\{-1,+1\}\) be the unique sign for which
\[
 n_j=\frac{\xi_j\bar b_j}{\|\bar b_j\|_2}
\]
has first nonzero coordinate positive, and let
\[
 \gamma_j=w_j\|\bar b_j\|_2,
 \qquad
 v_{\rm sum}=\sum_{j=1}^m w_j\bar b_j.
\]
Then
\begin{align}
 g_1(x)
 &=\sum_{j=1}^m\gamma_j|n_j^\top x|,                    \label{eq:supp-one-even}\\
 f_1(x)-f_1(-x)
 &=v_{\rm sum}^\top x.                                 \label{eq:supp-one-odd}
\end{align}
The signs \(\xi_j\) are used only in the proof; the recovery algorithm
does not need to identify them.

Appendix~\ref{supp:critical-points} gives an algorithm that terminates
after finitely many exact value queries with probability one and recovers the unordered collection
\(\{(n_j,\gamma_j):j\in[m]\}\) and \(v_{\rm sum}\) using only
length-one inputs.
Appendix~\ref{supp:sign-patterns} then combines these quantities with the
recovered effective attention heads to construct a bias-free one-layer ReLU
Transformer that computes the same function as the original model.

%% file: supplementary/sections/06_01_breakpoint_recovery.tex
\subsection{Recovery from Critical Points}
\label{supp:critical-points}
Following the terminology of the
ReLU model-extraction literature
\citep{carlini2020cryptanalytic,jagielski2020high}, we call a point at
which the left and right slopes of a one-dimensional piecewise-linear
restriction differ a \emph{critical point}. More precisely, on an interval
where \(\phi(t)=ct+d\), the coefficient \(c\) is its \emph{slope}. If
\(c_-\) and \(c_+\) are the slopes immediately to the left and right of a
critical point, respectively, we call \(c_+-c_-\) its signed change in
slope.

The proof sketch in the main paper recovers the pairs
\((n_j,\gamma_j)\) from such changes along queried lines. To explain this
step, restrict \eqref{eq:supp-one-even} to a line \(x=a+tb\):
\[
 g_1(a+tb)
 =
 \sum_{j=1}^m
 \gamma_j\bigl|n_j^\top a+t n_j^\top b\bigr|.
\]
This is a continuous piecewise-linear function of \(t\). If
\(n_j^\top b\neq0\), its \(j\)-th term changes slope when
\(n_j^\top a+t n_j^\top b=0\). When \(a\) and \(b\) are sampled from
continuous distributions, these locations are distinct with probability
one. Each critical point is then caused by one term, and its signed change
in slope is \(2\gamma_j|n_j^\top b|\). Observing how the critical point moves
when the line is shifted determines \(n_j\). We first give a general
procedure for recovering these quantities from exact value queries.

\begin{lemma}[Recovery of a univariate piecewise-linear function]
\label{lem:supp-univariate-critical-points}
Let \(\phi:\mathbb R\to\mathbb R\) be continuous and piecewise linear.
Suppose that \(\phi\) has exactly \(k\geq1\) critical points and that \(k\)
is known. Exact value queries recover their locations, the changes in
slope, and every linear piece with probability one after finitely many
queries. Without known bounds on the critical-point locations and their
minimum separation, the required number of queries is model-dependent.
\end{lemma}

\begin{proof}
The procedure samples \(\lambda_{\mathrm{grid}}\) from an absolutely
continuous distribution on \((0,1)\). At round \(r\geq1\), it sets
\(h_r=2^{-r}\) and \(R_r=2^r\) and considers every shifted grid cell
\[
 I_{r,i}
 =
 [\lambda_{\mathrm{grid}}+ih_r,\,
  \lambda_{\mathrm{grid}}+(i+1)h_r]
 \subseteq[-R_r,R_r].
\]
For a cell with endpoints \(L,R\) and midpoint \(M=(L+R)/2\), the
procedure marks the cell if
\begin{equation}
 \phi(L)+\phi(R)\neq2\phi(M).                            \label{eq:supp-midpoint-test}
\end{equation}
The search stops as soon as \(k\) pairwise disjoint cells are marked.

On any cell containing no critical point, \(\phi\) is affine and satisfies
equality in \eqref{eq:supp-midpoint-test}, so every marked cell contains a
critical point.  The random shift avoids, with probability one, the
countable set of shifts for which a critical point is a grid endpoint or a
grid midpoint in some round.  If a cell contains exactly one critical point
\(t_\star\) in its interior, let \(a_-\) and \(a_+\) be the slopes on its
two sides.  If \(M<t_\star\), then the secant slope on \([L,M]\) is
\(a_-\), whereas the secant slope on \([M,R]\) is a strict convex
combination of \(a_-\) and \(a_+\).  Since \(a_-\neq a_+\),
\eqref{eq:supp-midpoint-test} holds.  The case \(M>t_\star\) is
symmetric.

For all sufficiently large \(r\), the search interval contains all
critical points and \(h_r\) is smaller than the minimum distance between
two distinct critical points.  Each critical point then lies in its own marked
cell.  Conversely, \(k\) disjoint marked cells contain at least \(k\)
distinct critical points.  Since there are exactly \(k\), the stopping rule
cannot omit a critical point, and each retained cell contains exactly one.

For each retained cell, the procedure repeatedly divides it into two
halves and applies \eqref{eq:supp-midpoint-test} to both. It retains the
marked half and records the line determined by the endpoints of
the unmarked half. Because \(t_\star\) is not one of the subdivision
points, this process eventually records a line from each side
of \(t_\star\).  Once both slopes have been obtained, the two recorded
lines have different slopes, so their exact intersection is \(t_\star\),
and their slope difference is the signed change in slope.
Repeating this construction for the \(k\) retained cells recovers the
complete piecewise-linear formula.
\end{proof}

\begin{proposition}[Recovery from length-one inputs]
\label{prop:supp-length-one-recovery}
Under the assumptions of Theorem~\ref{thm:full-transformer}, exact
Transformer queries on length-one inputs recover the unordered collection
\[
 \{(n_j,\gamma_j):j\in[m]\}
\]
and \(v_{\rm sum}=\sum_jw_j\bar b_j\) with probability one after finitely
many
queries. For every \(x\in\mathbb R^d\), these quantities satisfy
\begin{equation}
 f_1(x)
 =\frac12\sum_{j=1}^m\gamma_j|n_j^\top x|
  +\frac12v_{\rm sum}^\top x.                          \label{eq:supp-one-representation}
\end{equation}
\end{proposition}

\begin{proof}
For \(r\in[d]\), let \(e_r\) denote the \(r\)-th standard basis vector
of \(\mathbb R^d\). Equation~\eqref{eq:supp-one-odd} gives
\[
 (v_{\rm sum})_r=f_1(e_r)-f_1(-e_r),
\]
so \(2d\) length-one queries recover \(v_{\rm sum}\). A value query to
\(g_1\) is implemented by two length-one queries to \(f_1\).

The algorithm samples independent vectors \(a,b\in\mathbb R^d\) from
absolutely continuous distributions and restricts
\eqref{eq:supp-one-even} to
\[
 \phi_0(t)=g_1(a+tb).
\]
For each \(j\in[m]\), write
\[
 \alpha_j=n_j^\top a,\qquad
 \beta_j=n_j^\top b,\qquad
 t_j=-\frac{\alpha_j}{\beta_j}.
\]
The event \(\beta_j=0\) is a proper hyperplane in \(b\).  For
\(j\neq k\), equality \(t_j=t_k\) is equivalent to
\[
 (n_j^\top a)(n_k^\top b)
 -(n_k^\top a)(n_j^\top b)=0.
\]
Because \(n_j\) and \(n_k\) are not scalar multiples of one another, the left-hand side is
a nonzero polynomial in \((a,b)\).  Thus, with probability one,
\(\phi_0\) has exactly the \(m\) distinct critical points \(t_1,\ldots,t_m\).
The signed change in slope at \(t_j\) is
\begin{equation}
 \Delta_j=2\gamma_j|\beta_j|\neq0.                      \label{eq:supp-base-slope-jump}
\end{equation}
Lemma~\ref{lem:supp-univariate-critical-points} therefore recovers the
unordered pairs \((t_j,\Delta_j)\).

It remains to recover the vector \(n_j\) associated with each critical point.
For each \(r\in[d]\), the algorithm samples two independent nonzero shifts
\(\epsilon_{r,1},\epsilon_{r,2}\) from absolutely continuous
distributions and queries
\[
 \phi_{r,\nu}(t)
 =g_1(a+\epsilon_{r,\nu}e_r+tb),
 \qquad \nu\in\{1,2\}.
\]
The critical point corresponding to ReLU unit \(j\) on the line shifted by
\(\epsilon\) is
\begin{equation}
 t_j(\epsilon)
 =t_j+\kappa_{jr}\epsilon,
 \qquad
 \kappa_{jr}=-\frac{(n_j)_r}{n_j^\top b}.               \label{eq:supp-critical-point-trajectory}
\end{equation}
Two shifted critical-point locations can coincide only at an isolated value
of \(\epsilon\), because their values at \(\epsilon=0\) are distinct.
Hence both shifted restrictions have \(m\) distinct critical points with
probability one, and
Lemma~\ref{lem:supp-univariate-critical-points} recovers their two unordered
sets of critical points.

To match the shifted critical points to the base critical points, enumerate the
finitely many pairs of bijections \((\pi_1,\pi_2)\) and retain those
satisfying
\begin{equation}
 \frac{t_{\pi_1(j)}(\epsilon_{r,1})-t_j}{\epsilon_{r,1}}
 =
 \frac{t_{\pi_2(j)}(\epsilon_{r,2})-t_j}{\epsilon_{r,2}}
 \quad\text{for every }j\in[m].                         \label{eq:supp-trajectory-match}
\end{equation}
The true matching satisfies this identity. For a false assignment, let
\(j_1\) and \(j_2\) be the indices assigned to the base critical point
\(t_j\) in the two shifted restrictions. The required equality becomes
\[
 \frac{t_{j_1}-t_j}{\epsilon_{r,1}}+\kappa_{j_1r}
 =
 \frac{t_{j_2}-t_j}{\epsilon_{r,2}}+\kappa_{j_2r}.
\]
As an identity in the two independent continuous shifts, this can hold
only when \(j_1=j_2=j\); otherwise, distinctness of the base critical points
leaves a nonzero coefficient of \(1/\epsilon_{r,1}\) or
\(1/\epsilon_{r,2}\).  There are finitely many false assignments.
Therefore the true matching is unique with probability one, and
\eqref{eq:supp-trajectory-match} recovers every \(\kappa_{jr}\).

Let
\(\kappa_j=(\kappa_{j1},\ldots,\kappa_{jd})^\top\).  By
\eqref{eq:supp-critical-point-trajectory},
\[
 \kappa_j=-\frac{n_j}{n_j^\top b},
 \qquad
 \|\kappa_j\|_2=\frac1{|n_j^\top b|}.
\]
Normalize \(\kappa_j\) and apply the fixed orientation convention that
the first nonzero coordinate is positive.  This returns \(n_j\).
Equations~\eqref{eq:supp-base-slope-jump} and
\eqref{eq:supp-critical-point-trajectory} then give
\[
 \gamma_j=\frac{\Delta_j\|\kappa_j\|_2}{2}.
\]
Finally,
\(\operatorname{ReLU}(z)=(|z|+z)/2\) proves
\eqref{eq:supp-one-representation}.

The procedure applies Lemma~\ref{lem:supp-univariate-critical-points} to
\(1+2d\) univariate restrictions. Each application terminates after
finitely many queries with probability one, and all other
exceptional events form a finite union of zero sets of nonzero
polynomials. Thus the complete procedure terminates with probability one
after finitely many exact Transformer queries on length-one inputs.
\end{proof}

\begin{remark}[No uniform query bound]
\label{rem:supp-length-one-query-bound}
The expanding searches in
Lemma~\ref{lem:supp-univariate-critical-points} must accommodate
model-dependent critical-point locations and separations. The proposition
therefore guarantees termination after finitely many queries with
probability one, but it does not give a uniform query bound over all
admissible models. Known location and separation bounds would
turn the expanding search into an explicit finite grid, but such bounds
are not assumed in Theorem~\ref{thm:full-transformer}.
\end{remark}

%% file: supplementary/sections/06_02_sign_pattern_spanning_scale_search.tex
\subsection{Constructing a Functionally Equivalent Transformer}
\label{supp:sign-patterns}

We now use the recovered quantities to construct a Transformer that
computes the same function as the original model. The pairs
\((W_h,v_h)\) and all quantities obtained from the preceding
critical-point recovery are already known. It remains to recover how each
attention head contributes to each ReLU unit. We first choose input vectors
whose ReLU sign patterns are linearly independent, then recover these
head-specific contributions from scaled repeated-token queries, and finally
assemble them into a functionally equivalent Transformer.

\paragraph{Linearly independent sign patterns.}
For \(q\) outside the hyperplanes
\(\{x:n_j^\top x=0\}\), define
\[
 \sigma(q)
 =
 \bigl(\operatorname{sign}(n_1^\top q),\ldots,
       \operatorname{sign}(n_m^\top q)\bigr)
 \in\{-1,+1\}^m.
\]
A sign vector in the image of this map is called a \emph{realized sign
pattern}.  Let \(\mathcal S\) be the set of all realized sign patterns.

\begin{lemma}[Spanning by realized sign patterns]
\label{lem:supp-sign-pattern-span}
If \(n_1,\ldots,n_m\) are nonzero and pairwise nonproportional, then
\[
 \operatorname{span}(\mathcal S)=\mathbb R^m.
\]
In particular, there are \(m\) realized sign patterns that form the rows
of a nonsingular \(m\times m\) matrix, including when \(m>d\).
\end{lemma}

\begin{proof}
Fix \(j\in[m]\), and let
\(\mathcal H_j=\{x:n_j^\top x=0\}\).  Since the normals are pairwise
nonproportional, the central hyperplanes \(\mathcal H_1,\ldots,\mathcal
H_m\) are distinct.  The finite union of proper subspaces
\[
 \bigcup_{k\neq j}(\mathcal H_j\cap\mathcal H_k)
\]
cannot cover \(\mathcal H_j\).  Choose
\[
 x_j\in
 \mathcal H_j\setminus\bigcup_{k\neq j}\mathcal H_k
\]
and choose \(r_j\) such that \(n_j^\top r_j\neq0\).  For all sufficiently
small \(\epsilon>0\), the points \(x_j+\epsilon r_j\) and
\(x_j-\epsilon r_j\) have the same sign against every \(n_k\) with
\(k\neq j\), and opposite signs against \(n_j\).  Their realized sign
patterns therefore differ by \(2e_j\) or \(-2e_j\), where \(e_j\) is the
\(j\)-th standard basis vector of \(\mathbb R^m\).  Thus
\(e_j\in\operatorname{span}(\mathcal S)\) for every \(j\), proving the
claim.
\end{proof}

The construction can be implemented under the computational assumptions
of Section~\ref{sec:problem-formulation}.
When \(d=1\), pairwise nonproportionality implies \(m=1\), and any
nonzero \(q_1\) gives a nonsingular \(1\times1\) sign matrix.  When
\(d\ge2\), for each \(j\), sample a random point in \(H_j\) from a
continuous distribution on \(H_j\), perturb it to the two sides of that
hyperplane, collect the resulting sign patterns, and use exact Gaussian
elimination to select \(m\) independent patterns. For each selected pattern, choose a point \(q_\ell\) that realizes it,
and define the matrix \(\Sigma\in\{-1,+1\}^{m\times m}\) by
\begin{equation}
 \Sigma_{\ell j}
 =\operatorname{sign}(n_j^\top q_\ell),
 \qquad \ell,j\in[m].
 \label{eq:supp-sign-matrix}
\end{equation}
By construction, \(\Sigma\) is nonsingular.

The points \(q_\ell\) may also be chosen to satisfy the following two conditions.  First,
\begin{equation}
 (W_h-W_g)q_\ell\neq0
 \quad\text{for every }h\neq g.                         \label{eq:supp-q-head-separation}
\end{equation}
For fixed \(h\neq g\), the excluded set is a proper subspace because
\(W_h\neq W_g\).  Second, writing
\[
 Z(\sigma)=\sum_{j=1}^m\gamma_j\sigma_jn_j
\]
as in Section~\ref{sec:relu-extension}, choose \(q_\ell\) so that
\begin{equation}
 q_\ell^\top\{Z(\tau)-Z(\sigma(q_\ell))\}\neq0
 \quad
 \text{for every }
 \tau\in\{-1,+1\}^m\setminus\{\sigma(q_\ell)\}.
 \label{eq:supp-z-separation}
\end{equation}
The injectivity of \(Z\) makes every vector in braces nonzero. For each
\(\ell\), the set of points \(q\) satisfying
\(\operatorname{sign}(n_j^\top q)=\Sigma_{\ell j}\) for every \(j\) is
nonempty and open. The finitely many proper subspaces and hyperplanes
excluded by \eqref{eq:supp-q-head-separation} and
\eqref{eq:supp-z-separation} cannot cover this set. Sampling from a
continuous distribution supported on it therefore produces a valid
\(q_\ell\) with probability one.

\paragraph{Recovering head contributions to ReLU units.}
To describe the remaining unknown contributions, recall the orientation
signs \(\xi_j\) from
Appendix~\ref{supp:transformer}.  Define
\begin{equation}
 d_{hj}
 =\frac{\xi_jb_{hj}}{\|\bar b_j\|_2},
 \qquad
 p_{hj}
 =\gamma_jd_{hj}
 =\xi_jw_jb_{hj}.                                      \label{eq:supp-canonical-blocks}
\end{equation}
These definitions imply
\begin{equation}
 \sum_{h=1}^H d_{hj}=n_j,
 \qquad
 \sum_{h=1}^H p_{hj}=\gamma_jn_j.                      \label{eq:supp-block-sums}
\end{equation}
Although the individual signs \(\xi_j\) need not be recovered, the
vectors \(p_{hj}\) will be recovered directly from even Transformer
responses.

For each \(\ell\in[m]\), choose \(d\) linearly independent directions
\(u_{\ell1},\ldots,u_{\ell d}\) such that
\begin{equation}
 u_{\ell r}^\top W_hq_\ell
 \neq
 u_{\ell r}^\top W_gq_\ell
 \quad
 \text{for every }r\in[d]\text{ and }h\neq g.           \label{eq:supp-projected-head-separation}
\end{equation}
Such directions can be sampled without additional oracle queries because the
\(W_h\) and \(q_\ell\) have already been recovered.  For each fixed
triple \((\ell,h,g)\), the excluded directions form a proper hyperplane
by \eqref{eq:supp-q-head-separation}; linear dependence of the \(d\)
directions is also a measure-zero event.

For \(\nu\in[H]\) and a scale \(T>0\), query the sequence
\begin{equation}
 X_{\ell r\nu}(T)
 =
 \left[
  (Tq_\ell+u_{\ell r}/T)^\top;
  \underbrace{(Tq_\ell)^\top;\ldots;(Tq_\ell)^\top}_{\nu\ \mathrm{copies}}
 \right].                                               \label{eq:supp-outer-query}
\end{equation}
The final token is one of the \(\nu\) copies of \(Tq_\ell\). Let
\[
 \eta_{\ell rh}=u_{\ell r}^\top W_hq_\ell,
 \qquad
 r_{\ell rh}=\exp(\eta_{\ell rh}),
 \qquad
 \theta_{\ell r\nu h}
 =\frac{r_{\ell rh}}{\nu+r_{\ell rh}}.
\]
The factor \(T\) cancels from the difference between the attention logits
of the perturbed token and a repeated token. Direct substitution into the
attention formula
therefore yields
\begin{equation}
 y_h(X_{\ell r\nu}(T))
 =Tq_\ell+\frac{\theta_{\ell r\nu h}}{T}u_{\ell r}.    \label{eq:supp-outer-summary}
\end{equation}

Define
\[
 \alpha_{\ell j}=n_j^\top q_\ell,
 \qquad
 \beta_{\ell r\nu j}
 =\sum_{h=1}^H
   \theta_{\ell r\nu h}d_{hj}^\top u_{\ell r}.
\]
After orienting and normalizing the scalar input to the \(j\)-th ReLU unit,
we obtain
\begin{equation}
 \frac{\xi_j z_j(X_{\ell r\nu}(T))}
      {\|\bar b_j\|_2}
 =
 T\alpha_{\ell j}+\frac{\beta_{\ell r\nu j}}{T}.      \label{eq:supp-scaled-preactivation}
\end{equation}
Because \(\alpha_{\ell j}\neq0\), its sign eventually equals
\(\Sigma_{\ell j}\) as \(T\) increases.

\begin{lemma}[Recovery when the ReLU signs are stable]
\label{lem:supp-recovery-at-scale}
Suppose that, at a scale \(T\), every quantity in
\eqref{eq:supp-scaled-preactivation} has sign
\(\Sigma_{\ell j}\).  Then the even responses at that scale recover every
vector \(p_{hj}\).
\end{lemma}

\begin{proof}
For each outer query, define the even response
\[
 E_{\ell r\nu}(T)
 =
 \mathrm{TF}(X_{\ell r\nu}(T))
 +\mathrm{TF}(-X_{\ell r\nu}(T)).
\]
The function \(g_1\) is already known from
Appendix~\ref{supp:critical-points}. Under the stated sign condition,
\eqref{eq:supp-tf-even},
\eqref{eq:supp-scaled-preactivation}, and
\eqref{eq:supp-canonical-blocks} give
\begin{align}
 Y_{\ell r\nu}(T)
 &\coloneqq
 T\{E_{\ell r\nu}(T)-Tg_1(q_\ell)\}
 =\sum_{h=1}^H
 \theta_{\ell r\nu h}u_{\ell r}^\top p_h^{(\ell)},
 \label{eq:supp-outer-linear-system}\\
 p_h^{(\ell)}
 &\coloneqq
 \sum_{j=1}^m\Sigma_{\ell j}p_{hj}.
 \label{eq:supp-sign-combination}
\end{align}

For fixed \((\ell,r)\), let
\[
 \Theta_{\ell r}[\nu,h]
 =
 \frac{r_{\ell rh}}{\nu+r_{\ell rh}},
 \qquad \nu,h\in[H].
\]
The numbers \(r_{\ell r1},\ldots,r_{\ell rH}\) are positive and pairwise
distinct by \eqref{eq:supp-projected-head-separation}.  Factoring
\(r_{\ell rh}\) from column \(h\) leaves the Cauchy matrix
\([1/(\nu+r_{\ell rh})]_{\nu,h}\), which is nonsingular. Inverting
\(\Theta_{\ell r}\) in \eqref{eq:supp-outer-linear-system} recovers
\(u_{\ell r}^\top p_h^{(\ell)}\) for every \(h\). Since
\(u_{\ell1},\ldots,u_{\ell d}\) are linearly independent, these \(d\)
directional measurements recover each vector \(p_h^{(\ell)}\).

For a fixed head \(h\), form \(\widetilde P_h\in\mathbb R^{m\times d}\)
whose \(\ell\)-th row is \((p_h^{(\ell)})^\top\), and form
\(P_h\in\mathbb R^{m\times d}\) whose \(j\)-th row is \(p_{hj}^\top\).
Equation~\eqref{eq:supp-sign-combination} becomes
\[
 \widetilde P_h=\Sigma P_h.
\]
The nonsingularity of \(\Sigma\) gives
\(P_h=\Sigma^{-1}\widetilde P_h\), recovering all \(p_{hj}\).
\end{proof}

\paragraph{Finding a scale with stable ReLU signs.}
For every \((\ell,r,\nu,j)\), the quantity in
\eqref{eq:supp-scaled-preactivation} has the form
\(T\alpha_{\ell j}+\beta_{\ell r\nu j}/T\), where
\(\operatorname{sign}(\alpha_{\ell j})=\Sigma_{\ell j}\). Its sign is therefore
\(\Sigma_{\ell j}\) for all sufficiently large \(T\). The required value
of \(T\), however, depends on the unknown \(\beta_{\ell r\nu j}\). We therefore
increase \(T\) until two consecutive reconstructions agree exactly.

At every evaluated \(T>0\), the algorithm applies the same linear-algebra
steps as in Lemma~\ref{lem:supp-recovery-at-scale}, regardless of the
signs. It arranges the resulting vectors in a fixed order and denotes their
concatenation by \(\mathcal C(T)\). If all signs equal their corresponding
\(\Sigma_{\ell j}\), the components of \(\mathcal C(T)\) are exactly the
vectors \(p_{hj}\) in that order.

\begin{proposition}[Finite search for stable ReLU signs]
\label{prop:supp-scale-search}
Assume that \(Z:\{-1,+1\}^m\to\mathbb R^d\) is one-to-one. Let
\(\lambda\in(1,2)\) and \(\rho\in(3/2,2)\) be independent random variables
drawn from continuous distributions, and define
\[
 T_k=\lambda\rho^k,\qquad k=0,1,\ldots.
\]
The algorithm evaluates these scales in order and stops at the first
\(k\geq1\) for which
\(\mathcal C(T_{k-1})=\mathcal C(T_k)\). With probability one, it stops
after finitely many evaluations and only after every quantity in
\eqref{eq:supp-scaled-preactivation} has sign \(\Sigma_{\ell j}\).
\end{proposition}

\begin{proof}
Let \(\tau_{\ell r\nu}(T)\in\{-1,+1\}^m\) record the signs of the \(m\)
quantities in \eqref{eq:supp-scaled-preactivation}. Direct expansion of
the even response gives
\begin{align}
 Y_{\ell r\nu}(T)
 &=
 T^2 A_{\ell r\nu}(T)+B_{\ell r\nu}(T),                \label{eq:supp-scale-affine}\\
 A_{\ell r\nu}(T)
 &=
 q_\ell^\top
 \{Z(\tau_{\ell r\nu}(T))-Z(\sigma(q_\ell))\}.        \label{eq:supp-scale-leading}
\end{align}
As long as these signs do not change, both \(A_{\ell r\nu}(T)\) and
\(B_{\ell r\nu}(T)\) are constant. By
\eqref{eq:supp-z-separation}, \(A_{\ell r\nu}(T)=0\) exactly when
\(\tau_{\ell r\nu}(T)=\sigma(q_\ell)\).

The linear-algebra operations used to form \(\mathcal C(T)\) do not depend
on \(T\). Hence, while the complete collection of signs remains unchanged,
\begin{equation}
 \mathcal C(T)=T^2\mathcal U+\mathcal V,                \label{eq:supp-candidate-scale}
\end{equation}
for some constant vectors \(\mathcal U\) and \(\mathcal V\). Moreover,
\(\mathcal U\neq0\) whenever at least one sign differs from its
corresponding \(\Sigma_{\ell j}\). Once all signs are correct,
\(\mathcal U=0\), and \(\mathcal C(T)\) is the true collection of
\(p_{hj}\), independent of \(T\).

There are finitely many expressions in
\eqref{eq:supp-scaled-preactivation}.  Each has the form
\(Ta+b/T\) with \(a\neq0\), and hence has at most one positive zero.
Beyond the largest such zero, all signs equal their corresponding
\(\Sigma_{\ell j}\). For each expression and each \(k\), the probability
that \(T_k\) is exactly its zero is zero. Since there are finitely many
expressions and countably many evaluated scales, with probability one no
evaluated \(T_k\) is a zero. Eventually two consecutive evaluations
therefore return the same true collection of \(p_{hj}\), so the rule
terminates.

It remains to exclude premature termination. For a fixed \(k\) and fixed
collections of signs at \(T_{k-1}\) and \(T_k\), the two instances of
\eqref{eq:supp-candidate-scale} are
\[
 T_{k-1}^2\mathcal U+\mathcal V
 \quad\text{and}\quad
 T_k^2\mathcal U'+\mathcal V'.
\]
If at least one sign is wrong at either scale, then at least one of
\(\mathcal U,\mathcal U'\) is nonzero. Equality of the two reconstructions
would require
\begin{equation}
 \lambda^2\rho^{2k-2}(\mathcal U-\rho^2\mathcal U')
 +(\mathcal V-\mathcal V')=0.                          \label{eq:supp-premature-equality}
\end{equation}
There are only finitely many possible collections of signs. A continuously
sampled \(\rho\) avoids every value for which
\(\mathcal U-\rho^2\mathcal U'=0\) when at least one of the two vectors is
nonzero. Conditional on such a \(\rho\),
\eqref{eq:supp-premature-equality} has at most one positive solution for
\(\lambda\). Taking the union over the countably many \(k\) and the
finitely many possible sign collections still gives a probability-zero
event. Thus the rule does not stop prematurely with probability one.
\end{proof}

\paragraph{Constructing the functionally equivalent Transformer.}
After the scale search, define
\[
 t_j(X)=\sum_{h=1}^Hp_{hj}^\top y_h(X),
 \qquad
 L(X)=\sum_{h=1}^Hv_h^\top y_h(X).
\]
By \eqref{eq:supp-canonical-blocks},
\[
 t_j(X)=\xi_jw_jz_j(X).
\]
Equations~\eqref{eq:supp-tf-odd} and
\eqref{eq:supp-tf-even} therefore imply, for every input sequence \(X\),
\begin{equation}
 \mathrm{TF}(X)
 =
 \frac12\sum_{j=1}^m
 \operatorname{sign}(\gamma_j)|t_j(X)|
 +\frac12L(X).                                         \label{eq:supp-equivalent-transformer}
\end{equation}
This representation is independent of the unknown orientation signs
\(\xi_j\).

An explicit bias-free ReLU Transformer realizing
\eqref{eq:supp-equivalent-transformer} uses the recovered matrices
\(W_h\) and, for each head \(h\), the \(2m+2\) columns
\[
 p_{h1},-p_{h1},\ldots,p_{hm},-p_{hm},v_h,-v_h.
\]
The two columns associated with unit \(j\) both receive output weight
\(\operatorname{sign}(\gamma_j)/2\); the last two columns receive output
weights \(1/2\) and \(-1/2\), respectively.  Since
\(|z|=\operatorname{ReLU}(z)+\operatorname{ReLU}(-z)\) and
\(z=\operatorname{ReLU}(z)-\operatorname{ReLU}(-z)\), this Transformer
has width at most \(2m+2\) and agrees with \(\mathrm{TF}\) on every
input.

For completeness, let \(Q_1\) be the model-dependent
finite number of Transformer queries used in
Proposition~\ref{prop:supp-length-one-recovery}, and let
\(R_{\mathrm{scale}}\) be the number of evaluated scales.  The effective
attention stage uses \(2(4Hd^2-2H+1)\) Transformer queries. At each
evaluated scale there are \(m\) choices of \(q_\ell\), \(d\) directions,
and \(H\)
multiplicities, and each even response uses the pair \(X,-X\).
Therefore the complete query count is
\[
 Q_1
 +2(4Hd^2-2H+1)
 +2mHdR_{\mathrm{scale}}.
\]
Both \(Q_1\) and \(R_{\mathrm{scale}}\) are finite with
probability one but have no uniform upper bound under the stated
assumptions. The scaled queries have length at most \(H+1\), while the
effective attention stage has length at most \(2H+1\).  This proves all
claims of Theorem~\ref{thm:full-transformer}.

%% file: supplementary/sections/07_numerical_implementation_experiments.tex
\section{Numerical Implementation and Experiments}
\label{supp:experiments}

The guarantee in Theorem~\ref{thm:main} assumes exact scalar oracle
responses and exact arithmetic, as specified in
Section~\ref{sec:problem-formulation}. This section describes our numerical
implementation and reports three finite-precision experiments.

First, we run Algorithm~\ref{alg:main} using 180-significant-digit scalar
responses and 180-digit offline arithmetic. Second, we add controlled
perturbations to the scalar responses and measure the resulting parameter
error. Third, we compare 180-digit and IEEE~754 binary64 responses while
keeping the model, query directions, query inputs, \(t=1\), and offline
arithmetic fixed. We perform this comparison both for
Algorithm~\ref{alg:main} and for a variant that directly queries all
required one-token outputs.

These experiments provide empirical validation of
Algorithm~\ref{alg:main} on the sampled models and measure how errors in
the scalar responses affect the recovered parameters.  Success with
180-significant-digit scalar responses and 180-digit offline arithmetic
on these models does not imply that the same precision suffices for
every model.

%% file: supplementary/sections/07_01_numerical_realization_choice_t.tex
\subsection{Numerical Realization and Choice of $t$}
\label{supp:numerical}

\paragraph{Rational interpolation.}
For each local decoder \(D_t(u,q)\), the implementation first forms
\[
 \mathcal R(m)
 =
 F_M(X_m^{(t)}(u,q))-F_M([q^\top]),
 \qquad m=1,\ldots,2H.
\]
Write
\[
 \mathcal P(z)=\sum_{\ell=0}^{H-1}p_\ell z^\ell,
 \qquad
 \mathcal Q(z)=z^H+\sum_{\ell=0}^{H-1}q_\ell z^\ell.
\]
The interpolation equations
\(\mathcal P(m)=\mathcal R(m)\mathcal Q(m)\) become
\begin{equation}
 \begin{aligned}
  \sum_{\ell=0}^{H-1}p_\ell m^\ell
  -\mathcal R(m)\sum_{\ell=0}^{H-1}q_\ell m^\ell
  =
  \mathcal R(m)m^H, \quad m=1,\ldots,2H.
 \end{aligned}
 \label{eq:supp-numerical-system}
\end{equation}
Thus, for each local decoder, the implementation solves one \(2H\)-dimensional linear system for the coefficients of \(\mathcal P\) and \(\mathcal Q\).
The numerical experiments solve this system in multiprecision arithmetic.

\paragraph{Companion-matrix roots and residues.}
For the monic denominator
\(\mathcal Q(z)=z^H+q_{H-1}z^{H-1}+\cdots+q_0\), form the companion
matrix
\[
 C_{\mathcal Q}
 =
 \begin{bmatrix}
  0&0&\cdots&0&-q_0\\
  1&0&\cdots&0&-q_1\\
  0&1&\cdots&0&-q_2\\
  \vdots&\vdots&\ddots&\vdots&\vdots\\
  0&0&\cdots&1&-q_{H-1}
 \end{bmatrix}.
\]
The characteristic polynomial of \(C_{\mathcal Q}\) is
\(\mathcal Q\).  Hence, in exact arithmetic, its eigenvalues are
\(-r_1,\ldots,-r_H\), where
\[
 r_h=\exp(tu^\top W_hq)>0.
\]
For each recovered root \(-r_h\), the local quantities are computed as
\begin{equation}
 s_h=\log r_h,
 \qquad
 c_h=
 \frac{\mathcal P(-r_h)}
      {r_h\mathcal Q'(-r_h)}.                           \label{eq:supp-numerical-residue}
\end{equation}
Indeed, the residue of \(\mathcal P/\mathcal Q\) at \(-r_h\) is
\(\mathcal P(-r_h)/\mathcal Q'(-r_h)=c_hr_h\).
The implementation computes the companion eigenvalues and all subsequent
quantities at the prescribed multiprecision level.

The unordered local outputs are aligned using the additive identities
\[
 s_h(u_1+u_i,q_1)=s_h(u_1,q_1)+s_h(u_i,q_1)
\]
and
\[
 s_h(u_i,q_1+q_j)=s_h(u_i,q_1)+s_h(u_i,q_j).
\]
After consistent labeling, the implementation applies
\[
 \widehat W_h
 =t^{-1}\mathbf U^{-1}S_h\mathbf Q^{-1},
 \qquad
 \widehat v_h
 =t^{-1}\mathbf U^{-1}\mathbf c_h.
\]
No target parameter is used in interpolation, root recovery, residue
recovery, consistent head labeling, or these final linear solves.

\paragraph{Conditioning and the scale \(t\).}
The exact theorem holds for every fixed \(t\neq0\), but \(t\) changes the
conditioning of a finite-precision realization through
\[
 r_h=\exp(tu^\top W_hq),
 \qquad
 c_h=tu^\top v_h.
\]
As \(t\to0\),
\[
 r_h
 =
 1+tu^\top W_hq+O(t^2),
 \qquad
 c_hr_h
 =
 t u^\top v_h+O(t^2).
\]
Thus the values \(r_h\) cluster near \(1\), the poles \(-r_h\) cluster
near \(-1\), and the residues approach zero. When \(|t|\) is large, the
exponential can produce a wide dynamic range among the \(r_h\), which
can amplify numerical errors in the interpolation and companion-root
steps.  These observations explain why the exact condition
\(t\neq0\) is not, by itself, a finite-precision prescription.

All experiments use \(t=1\), fixed before target generation and not
adjusted for individual target models.

%% file: supplementary/sections/07_02_experimental_protocol.tex
\subsection{Experimental Setup}
\label{supp:experimental-protocol}

\paragraph{Targets.}
For each reported pair \((d,H)\), every entry of \(W_h\) and \(v_h\),
\(h\in[H]\), is sampled independently from a Gaussian distribution with
mean \(0\) and variance \(1/d\). For each reported \((d,H)\), the
high-precision and binary64 experiments use 100 independently sampled
target models, and all 100 are included in the reported results.
The controlled-perturbation experiment uses 10 independently sampled
target models and 10 noise patterns per model.

\paragraph{Direction bases and query pairs.}
Independently of \(W_h\) and \(v_h\), the implementation generates two
independent matrices whose entries are i.i.d.\ Gaussian random variables
with mean \(0\) and variance \(1\). It takes the orthogonal factors from
their QR factorizations as \(O_{\mathbf U}\) and \(O_{\mathbf Q}\).
Independently, it draws diagonal matrices
\(\Lambda_{\mathbf U}\) and \(\Lambda_{\mathbf Q}\) whose diagonal entries
are i.i.d.\ uniform random variables on \([1,2]\), and sets
\[
 \mathbf U=\Lambda_{\mathbf U}O_{\mathbf U},
 \qquad
 \mathbf Q=O_{\mathbf Q}\Lambda_{\mathbf Q}.
\]
All singular values of \(\mathbf U\) and \(\mathbf Q\) lie in \([1,2]\).
The algorithm uses the \(2d^2-1\) pairs specified in
Equations~(7)--(9) of the main paper. For each pair, the repeated-token
queries use \(m=1,\ldots,2H\). The high-precision exact-recovery experiment
directly queries only the one-token output \(F_M([q_1^\top])\) and computes
all other required one-token outputs after recovering the \(v_h\). Its query
count is therefore
\[
 1+2H(2d^2-1)=4Hd^2-2H+1,
\]
and its maximum sequence length is \(2H+1\). The value \(t=1\) is fixed
before the target models are generated. The matrices
\(\mathbf U,\mathbf Q\) and all query inputs are generated without using
\(W_h\) or \(v_h\).

\paragraph{Information separation.}
The learner is given \(d,H,t,\mathbf U,\mathbf Q\) and the query inputs,
and observes the corresponding scalar responses. It is not given \(W_h\)
or \(v_h\). The learner's output is saved before the evaluator reads
\(W_h\) and \(v_h\), and no quantity computed by the evaluator is returned
to the learner.

\paragraph{Precision terminology.}
A \(p\)-digit oracle response means that the scalar response is rounded
to \(p\) significant decimal digits. Offline arithmetic at \(p\) digits
means multiprecision arithmetic with \(p\) decimal digits of working
precision. Thus, ``180-digit'' refers to decimal, rather than binary,
precision.

IEEE~754 binary64~\citep{ieee754} has 53 bits of significand precision,
including the implicit leading bit. For each target model and query schedule, we run the learner once with
180-significant-digit scalar responses and once with responses rounded to
binary64.  The two runs use the same target model, \(\mathbf U,\mathbf Q\),
query inputs, \(t=1\), and 180-digit offline arithmetic.  In each run, the
learner sets its numerical acceptance tolerances according to the response
precision.

\paragraph{Evaluation metric.}
If the learner returns \(H\) heads, denote them by
\(\{(\widehat W_h,\widehat v_h):h\in[H]\}\), and let
\(\Pi_H\) be the set of permutations of \([H]\). Define
\[
 E_{\rm param}
 \coloneqq
 \min_{\pi\in\Pi_H}\max_{h\in[H]}
 \left(
  \|\widehat W_h-W_{\pi(h)}\|_F
  +\|\widehat v_h-v_{\pi(h)}\|_2
 \right).
\]

The same permutation is used for the \(W_h\) and \(v_h\) errors.
For the high-precision experiment, the minimum, median, and maximum of
\(E_{\rm param}\) are reported over the 100 target models. In both the
controlled-perturbation and binary64 experiments, a run is counted as
successful when it returns all \(H\) heads and
\(E_{\rm param}<10^{-2}\).

%% file: supplementary/sections/07_03_high_precision_recovery.tex
\subsection{High-Precision Recovery}
\label{supp:high-precision}

The high-precision experiment uses 180-significant-digit scalar oracle
responses and 180-digit offline arithmetic.  It runs
Algorithm~\ref{alg:main} with \(t=1\), using \(m=1,\ldots,2H\) for each
local decoder and directly querying only \(F_M([q_1^\top])\) among the
one-token outputs.  Table~\ref{tab:supp-high-precision} reports the
minimum, median, and maximum values of \(E_{\rm param}\) for the eight
tested pairs \((d,H)\).

\begin{table}[H]
\centering
{\small
\setlength{\tabcolsep}{1.5pt}
\begin{tabular}{@{}c r r c c c@{}}
\toprule
& & & \multicolumn{3}{c}{\(E_{\rm param}\)}\\
\cmidrule(lr){4-6}
\((d,H)\) & Params. & Queries & Min. & Median & Max.\\
\midrule
\((3,8)\)
& 96 & 273
& \(3.67\times10^{-142}\)
& \(1.46\times10^{-133}\)
& \(3.60\times10^{-124}\)\\
\((8,8)\)
& 576 & 2,033
& \(9.55\times10^{-133}\)
& \(7.63\times10^{-128}\)
& \(9.72\times10^{-121}\)\\
\((16,8)\)
& 2,176 & 8,177
& \(2.47\times10^{-128}\)
& \(7.18\times10^{-125}\)
& \(5.95\times10^{-118}\)\\
\((32,8)\)
& 8,448 & 32,753
& \(1.58\times10^{-125}\)
& \(1.54\times10^{-121}\)
& \(1.22\times10^{-115}\)\\
\((64,4)\)
& 16,640 & 65,529
& \(2.03\times10^{-153}\)
& \(3.45\times10^{-151}\)
& \(1.20\times10^{-145}\)\\
\((64,8)\)
& 33,280 & 131,057
& \(8.29\times10^{-121}\)
& \(5.12\times10^{-118}\)
& \(7.70\times10^{-113}\)\\
\((128,4)\)
& 66,048 & 262,137
& \(2.02\times10^{-152}\)
& \(3.24\times10^{-149}\)
& \(1.23\times10^{-144}\)\\
\((128,8)\)
& 132,096 & 524,273
& \(8.23\times10^{-117}\)
& \(3.38\times10^{-114}\)
& \(9.97\times10^{-110}\)\\
\bottomrule
\end{tabular}
}
\caption{High-precision results with \(t=1\).  Each row contains 100
independently sampled target models, and the algorithm returns \(H\)
heads for every model.  \textnormal{Params.} denotes \(H(d^2+d)\), and
\textnormal{Queries} denotes \(4Hd^2-2H+1\).}
\label{tab:supp-high-precision}
\end{table}

Across the 800 target models, every run returns \(H\) heads and satisfies
\(E_{\rm param}<10^{-100}\).  The query counts agree with
Theorem~\ref{thm:main} and range from \(273\) at
\((d,H)=(3,8)\) to \(524{,}273\) at \((d,H)=(128,8)\).  The maximum
sequence length is \(17\) for \(H=8\) and \(9\) for \(H=4\).

Together, these results provide empirical validation of the theoretical
recovery method in Theorem~\ref{thm:main}: the high-precision
implementation reproduces the theorem's conclusion that
Algorithm~\ref{alg:main} recovers all \(H\) heads when the scalar oracle
responses and the prescribed algebraic computations are exact.

%% file: supplementary/sections/07_04_approximate_binary64_outputs.tex
\subsection{Approximate and IEEE~754 binary64 Oracle Outputs}
\label{supp:approx-experiments}

\paragraph{Controlled perturbations.}
Let \(y_k\) denote the unperturbed response to the \(k\)-th value query.
For a noise pattern \(\eta=(\eta_k)_k\), the perturbed response is
\begin{equation}
 \widetilde y_k(\tau)
 =
 y_k+\tau\eta_k.
 \label{eq:supp-bounded-noise}
\end{equation}
Independently for each \(k\), the implementation samples \(J_k\)
uniformly from \(\{0,\ldots,2^{64}-1\}\) and sets
\[
 \eta_k
 =
 \frac{2J_k}{2^{64}-1}-1.
\]
For each noise pattern, the same \(\eta\) is used for every value of
\(\tau\).  The learner receives \(\widetilde y_k(\tau)\), but not
\(\eta_k\).

The experiment uses
\[
 (d,H)\in\{(3,2),(3,3),(4,4)\}.
\]
For each pair \((d,H)\), the experiment samples 10 independent models.
Each model is evaluated using the same collection of 10 independently
generated noise patterns. The value \(t=1\) is used throughout. The tested
values of \(\tau\) are
\[
 \begin{aligned}
  \tau\in\{&
   0,10^{-80},10^{-74},10^{-68},10^{-62},\\
  &
   10^{-56},10^{-50},10^{-44},10^{-38},10^{-32},\\
  &
   10^{-26},10^{-20},10^{-16},10^{-12},10^{-8}
  \}.
 \end{aligned}
\]
The unperturbed scalar responses are evaluated with 220 decimal digits,
and the learner uses 180-digit offline arithmetic.  The reported results
use only \(m=1,\ldots,2H\).  The stored records also contain responses
for \(m>2H\), but removing them does not change any learner output or
reported value of \(E_{\rm param}\).

The learner directly queries the \(2d-1\) one-token responses
\(F_M([q^\top])\) for
\[
 q\in
 \{q_1\}
 \cup
 \{q_j,q_1+q_j:j=2,\ldots,d\}.
\]
It therefore uses
\[
 (2d-1)+2H(2d^2-1)
\]
value queries: \(73\), \(107\), and \(255\) for
\((d,H)=(3,2),(3,3),(4,4)\), respectively.

For each model and noise pattern, let
\(E_{\rm param}(\tau)\) denote the error at perturbation level \(\tau\).
The analysis fits the least-squares slope of
\(\log E_{\rm param}(\tau)\) against \(\log\tau\) using every positive
tested value of \(\tau\) for which the learner returns \(H\) heads and
\[
E_{\rm param}(\tau)<10^{-2},
\qquad
E_{\rm param}(\tau)>10^6E_{\rm param}(0).
\]
For each model, the median is taken over the slopes obtained from
its 10 noise patterns.  For each pair \((d,H)\),
Table~\ref{tab:supp-bounded-noise} reports the median of the resulting
10 values.

At \(\tau_{\rm ref}=10^{-80}\), define
\[
 A_{\rm ref}
 =
 \frac{E_{\rm param}(\tau_{\rm ref})}{\tau_{\rm ref}}.
\]
For each model, the median of \(A_{\rm ref}\) is taken over its
10 noise patterns.  The table reports the median and interquartile range
of the resulting 10 values.

\begin{table}[H]
\centering
\small
\setlength{\tabcolsep}{4pt}
\begin{tabular}{@{}c c c c@{}}
\toprule
\((d,H)\) &
\shortstack{Median\\slope} &
\shortstack{Median\\\(A_{\rm ref}\)} &
\shortstack{\(A_{\rm ref}\)\\IQR}\\
\midrule
\((3,2)\) &
\(1.0000000\) &
\(6.98\times10^6\) &
\([5.25\times10^6,\,3.53\times10^8]\)\\
\((3,3)\) &
\(1.0000000\) &
\(5.37\times10^{11}\) &
\([1.20\times10^{10},\,2.95\times10^{12}]\)\\
\((4,4)\) &
\(1.0000000\) &
\(1.27\times10^{15}\) &
\([2.44\times10^{14},\,6.50\times10^{16}]\)\\
\bottomrule
\end{tabular}
\caption{Controlled-perturbation results.  The median slopes are rounded
to seven decimal places.  For each model, the median is first taken over
its 10 noise patterns.  The displayed median and interquartile range are
then computed over the 10 models.}
\label{tab:supp-bounded-noise}
\end{table}

In table order, the three median slopes before rounding are approximately
\[
 (0.9999999986,\;1.0000000000,\;0.9999999996).
\]
They therefore round to \(1.0000000\), as reported in
Table~\ref{tab:supp-bounded-noise}.  For the points used in the slope
fits, \(E_{\rm param}(\tau)\) is approximately proportional to \(\tau\).
The median \(A_{\rm ref}\) ranges from \(6.98\times10^6\) to
\(1.27\times10^{15}\), a factor greater than \(10^8\).

The experiment performs $30\times10\times15=4{,}500$ runs.  The learner returns \(H\) heads in \(3{,}987\) runs, and
\(3{,}891\) of these also satisfy
\(E_{\rm param}<10^{-2}\).

\paragraph{IEEE~754 binary64 outputs.}
For each target model and each of the two query schedules, the experiment
performs one run with 180-significant-digit scalar responses and one run
with responses rounded to IEEE~754 binary64.  The two runs use the same target model, \(\mathbf U,\mathbf Q\), query
inputs, \(t=1\), and 180-digit offline arithmetic.  Thus, within each
pair, the target model, query schedule, and offline arithmetic are fixed;
only the precision of the scalar responses changes.  The learner computes
its numerical acceptance tolerances from that precision. A run is successful if it returns
all \(H\) heads and satisfies \(E_{\rm param}<10^{-2}\).

\medskip
\noindent\textit{Queries from Algorithm~\ref{alg:main}.}
The first experiment uses the queries specified by
Algorithm~\ref{alg:main}. It directly queries only $F_M([q_1^\top])$ among
the one-token outputs and computes the remaining required one-token outputs
from the recovered $\widehat v_h$. The number of queries in each run is
\[
 1+2H(2d^2-1).
\]
Table~\ref{tab:supp-binary64} reports the results.

\noindent\begin{minipage}{\linewidth}
\centering
{\small
\setlength{\tabcolsep}{5pt}
\begin{tabular}{@{}c r r r@{}}
\toprule
\((d,H)\) & Queries & 180-digit & binary64\\
\midrule
\((3,2)\) & 69    & 100/100 & 99/100\\
\((3,4)\) & 137   & 100/100 & 1/100\\
\((3,8)\) & 273   & 100/100 & 0/100\\
\((8,4)\) & 1,017 & 100/100 & 0/100\\
\bottomrule
\end{tabular}
}
\captionof{table}{Results using the queries specified by
Algorithm~\ref{alg:main}. For each model, the model parameters, query
directions, query inputs, and offline arithmetic are the same in the
180-digit and binary64 runs.}
\label{tab:supp-binary64}
\end{minipage}

All \(400\) runs with 180-significant-digit scalar responses are
successful.  The binary64 runs are successful for \(99\), \(1\), \(0\),
and \(0\) of the \(100\) models, in table order.

\medskip
\noindent\textit{Direct one-token outputs.}
The second experiment directly queries all \(2d-1\) one-token
outputs, so no output is computed from \(\widehat v_h\).  It uses
\(2d-2\) more queries than Algorithm~\ref{alg:main}.\footnote{For
\((d,H)=(3,8)\), the high-precision experiment in
Table~\ref{tab:supp-high-precision} uses \(273\) queries, whereas the
direct-one-token experiment uses \(277\).  The difference is the
\(2d-2=4\) additional one-token outputs queried directly in the latter
experiment.}
Thus, each run uses
\[
 (2d-1)+2H(2d^2-1)
\]
queries. It uses the same \(100\) models in each row
as the first experiment.  Table~\ref{tab:supp-binary64-direct} reports
the results.

\noindent\begin{minipage}{\linewidth}
\centering
{\small
\setlength{\tabcolsep}{5pt}
\begin{tabular}{@{}c r r r@{}}
\toprule
\((d,H)\) & Queries & 180-digit & binary64\\
\midrule
\((3,2)\) & 73    & 100/100 & 100/100\\
\((3,4)\) & 141   & 100/100 & 31/100\\
\((3,8)\) & 277   & 100/100 & 0/100\\
\((8,4)\) & 1,031 & 100/100 & 0/100\\
\bottomrule
\end{tabular}
}
\captionof{table}{Results obtained by directly querying all one-token
outputs. For each model, the model parameters, query directions, query
inputs, and offline arithmetic are the same in the 180-digit and binary64
runs.}
\label{tab:supp-binary64-direct}
\end{minipage}

All \(400\) runs with 180-significant-digit scalar responses are
successful.  The binary64 runs are successful for \(100\), \(31\), \(0\),
and \(0\) of the \(100\) models, in table order.

In exact arithmetic, the outputs computed from \(v_h\) equal the
corresponding directly queried outputs.  In finite precision, however,
an error in \(\widehat v_h\) also affects every output computed from
\(\widehat v_h\).  The first experiment therefore evaluates
Algorithm~\ref{alg:main} with its stated query count.  The second uses
\(2d-2\) additional queries to remove this additional source of error
and more directly measure the effect of rounding the scalar responses
to binary64.

The controlled-perturbation results are consistent with the
linear-in-\(\tau\) upper bound in Theorem~\ref{thm:robust-recovery}.
For sufficiently small scalar-response perturbations, \(E_{\rm param}\)
is approximately proportional to the perturbation size, although the
proportionality factor varies substantially across models. The binary64
results show that IEEE~754 binary64 scalar responses can be
insufficient for satisfying the success criterion even when the
offline computations use 180-digit arithmetic.